%% file: main.tex
\documentclass[12pt,twoside,phd]{dms}

\usepackage[utf8]{inputenc}
\usepackage[T1]{fontenc}

\usepackage{lmodern}

\usepackage{graphicx}
\graphicspath{{.},{img/}}

\usepackage{thm-restate}
\usepackage{natbib}
\usepackage{multirow}

\usepackage{epigraph}

\usepackage{amsthm}

\renewcommand{\paragraph}[1]{\vspace{1.em} \noindent \textbf{#1}\hspace{0.25em}}

\anglais

\def\sloppy{%
  \tolerance 500
  \emergencystretch 3em%
  \hfuzz .5pt
  \vfuzz\hfuzz}
\usepackage{graphicx,amssymb,subfigure}

\usepackage[labelfont=bf, width=\linewidth]{caption}

\usepackage{hyperref}
\hypersetup{colorlinks=true,allcolors=black}
\usepackage{hypcap}
\usepackage{bookmark}

\hypersetup{
  pdftitle = {A deep learning theory for neural networks grounded in physics},
  pdfauthor = {Benjamin Scellier},
  pdfsubject = {Training physical systems with adjustable parameters by gradient descent},
  pdfkeywords = {deep learning, machine learning, physics, equilibrium propagation, energy-based model, variational principle, principle of least action, local learning rule, stochastic gradient descent, Hopfield networks, resistive networks, circuit theory, principle of minimum dissipated power, co-content, neuromorphic computing}
}

\newtheorem{thm}{Theorem}
\numberwithin{thm}{chapter}
\newtheorem{prop}[thm]{Proposition}
\newtheorem{lma}[thm]{Lemma}

\theoremstyle{definition}

\numberwithin{equation}{chapter}

\usepackage[onehalfspacing]{setspace}

\begin{document}

\version{1}

\title{
\centering
A deep learning theory\\
for neural networks grounded in physics
}

\author{Benjamin Scellier}

\copyrightyear{2020}

\department{Département d'informatique et de recherche opérationnelle}

\date{December 31, 2020}

\sujet{Informatique}

\president{Irina Rish}
\directeur{Yoshua Bengio}
\membrejury{Pierre-Luc Bacon}
\examinateur{Yann Ollivier}

\maketitle

\maketitle

\input{abstract-english.tex} 

\anglais

\cleardoublepage
\pdfbookmark[chapter]{\contentsname}{toc}

\tableofcontents
\cleardoublepage
\phantomsection
\listoftables
\cleardoublepage
\phantomsection
\listoffigures

\input{abbreviations.tex}

\cleardoublepage

\input{dedicace.tex}
\input{acknowledgements.tex}

\NoChapterPageNumber
\cleardoublepage

\input{intro.tex}
\input{eqprop.tex}
\input{hopfield-model.tex}
\input{resistive-networks.tex}
\input{discrete-time.tex}
\input{on-going-developments.tex}
\input{conclusion.tex}

\bibliographystyle{abbrvnat}
\bibliography{biblio}

\appendix
\input{appendix.tex}

\end{document}

%% file: abstract-english.tex
\anglais
\chapter*{Abstract}

In the last decade, deep learning has become a major component of artificial intelligence, leading to a series of breakthroughs across a wide variety of domains. The workhorse of deep learning is the optimization of loss functions by stochastic gradient descent (SGD). Traditionally in deep learning, neural networks are differentiable mathematical functions, and the loss gradients required for SGD are computed with the backpropagation algorithm. However, the computer architectures on which these neural networks are implemented and trained suffer from speed and energy inefficiency issues, due to the separation of memory and processing in these architectures. To solve these problems, the field of neuromorphic computing aims at implementing neural networks on hardware architectures that merge memory and processing, just like brains do. In this thesis, we argue that building large, fast and efficient neural networks on neuromorphic architectures also requires rethinking the algorithms to implement and train them. We present an alternative mathematical framework, also compatible with SGD, which offers the possibility to design neural networks in substrates that directly exploit the laws of physics. Our framework applies to a very broad class of models, namely those whose state or dynamics are described by variational equations. This includes physical systems whose equilibrium state minimizes an energy function, and physical systems whose trajectory minimizes an action functional (principle of least action). We present a simple procedure to compute the loss gradients in such systems, called equilibrium propagation (EqProp), which requires solely locally available information for each trainable parameter. Since many models in physics and engineering can be described by variational principles, our framework has the potential to be applied to a broad variety of physical systems, whose applications extend to various fields of engineering, beyond neuromorphic computing.

\paragraph{Keywords:} deep learning, machine learning, physical learning, equilibrium propagation, energy-based model, variational principle, principle of least action, local learning rule, stochastic gradient descent, Hopfield networks, resistive networks, circuit theory, principle of minimum dissipated power, co-content, neuromorphic computing

%% file: abbreviations.tex
\chapter*{Abbreviations List}

\begin{twocolumnlist}{.2\textwidth}{.7\textwidth}
  BPTT & Backpropagation Through Time\\
  CHL & Contrastive Hebbian Learning\\
  CIFAR-10 & A dataset of images of animals and objects (a standard benchmark in machine learning)\\
  ConvNet & Convolutional Network\\
  DHN & Deep Hopfield Network\\
  EBM & Energy-Based Model\\
  EqProp & Equilibrium Propagation\\
  GDD & Gradient Descending Dynamics, a property that relates the dynamics of EqProp to the loss gradients, in the discrete-time setting of Chapter \ref{chapter:discrete-time}\\
  GPU & Graphics Processing Unit\\
  LBM & Lagrangian-Based Model\\
  MNIST & A dataset of images of handwritten digits (a standard benchmark in machine learning)\\
  RBP & Recurrent Back-Propagation\\
  SGD & Stochastic Gradient Descent\\
\end{twocolumnlist}

%% file: dedicace.tex
\vspace*{2cm}
\epigraph{\hfill \itshape To the memory of my mother}

\vspace*{1cm}

\begin{figure*}[ht!]
\begin{center}
    \includegraphics[width=0.7\textwidth]{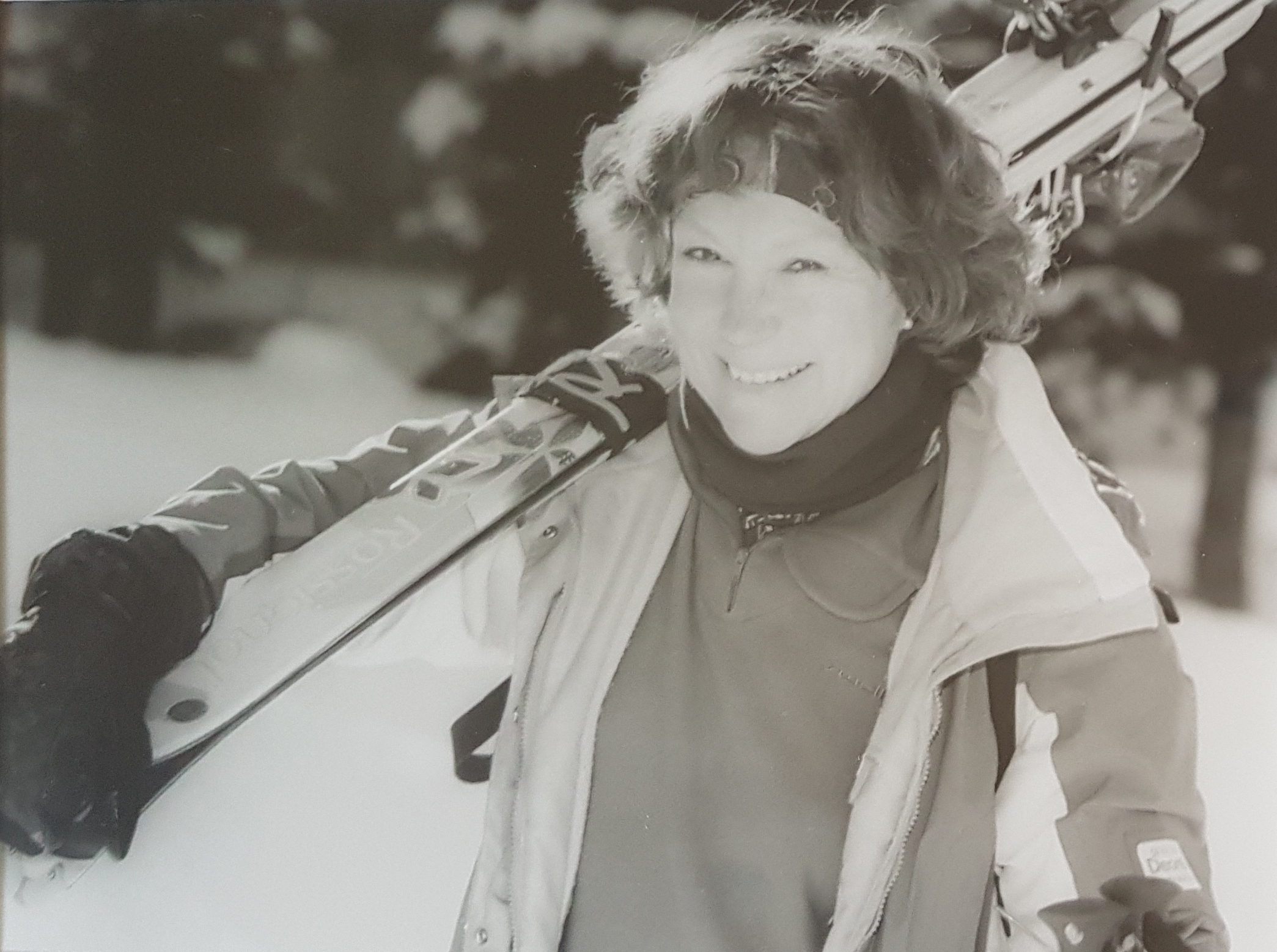}
\end{center}
\end{figure*}

%% file: acknowledgements.tex
\chapter*{Acknowledgements}

I had the privilege to do my doctoral studies at Mila to conduct research in the blooming field of neural networks, on a topic that I am extremely passionate about. The end of my program marks the end of a chapter, and it is an opportunity for me to reflect on what I have learned and how I have grown thanks to the people who have gone through this journey with me.

First of all, I want to thank my advisor Yoshua Bengio for being a great mentor and for communicating me his passion for neural network research and the `science of intelligence' in general (both `artificial' and `natural'). I am thankful for the freedom he granted me in my work, and for his guidance, his contributions and great insights, as well as his support throughout my PhD program. More than just a scientific advisor, Yoshua is a friend and a truly inspiring person who constantly takes actions for the common interests of society.

I would like to thank all my collaborators for their enthusiasm and contributions. Research is often exciting, but can be discouraging sometimes ; nevertheless, collaborating with them was always truly enjoyable. I would like to thank my friend Maxence Ernoult who I had the privilege to meet at a conference, after which we started an extremely fruitful collaboration, together with Axel Laborieux, Julie Grollier and Damien Querlioz. I would also like to thank Jack Kendall for collaborating together with his colleagues Ross Pantone and Kalpana Manickavasagam, as well as my friends and colleagues Jo\~ao Sacramento, Olexa Bilaniuk, Walter Senn, Anirudh Goyal, Jonathan Binas and Thomas Mesnard for our collaborations, and for all the great and inspiring scientific discussions.

I thank Wulfram Gerstner, Walter Senn and Alexandre Thierry for inviting me to their research groups, as well as the organizers of the Barbados workshop on `learning in the neocortex' -- Blake Richards, Timothy Lillicrap, Konrad Körding and Denis Therien. I also thank the organizers of the Brains, Minds and Machines summer school, and the organizers of the Cellular, Computational and Cognitive Neuroscience summer school. These visits, workshops and summer schools were precious opportunities for me to further expand my knowledge in other scientific disciplines, and I am grateful for the friends and colleagues that I met during these events, for the inspiring conversations that we had and the moments that we shared together. I also thank the juries of my thesis, Yann Ollivier, Pierre-Luc Bacon and Irina Rish, for their time and their insightful comments on my manuscript, as well as Bertrand Maubert, Jo\~ao Sacramento and Nicolas Zucchet for valuable feedback.

I would also like to thank all my friends in Montreal, Singapore, Switzerland and elsewhere, for all the wonderful moments that we have spent together, which brightened the journey of my doctoral studies. A special mention to the `friends of Normanton' for all the memories and moments of craziness. I also would like to thank all of those who were present and gave me their invaluable supports during the most difficult moments of my PhD program.

Finally, I thank my family members for their unconditional supports. I am thankful to my parents and brother who always supported me and gave me the freedom to pursue things that I felt like doing. Last but not least, I would like to thank my girlfriend Mannie for her support, her patience, and her love.

I would like to dedicate my thesis to my mother who passed away during the course of my PhD program. She has been a wonderful mother, always giving me support in all possible ways she could. She always encouraged me in pursuing my dreams, and she followed me around the world, even in the most difficult moments of her illness. She was really brave, and she attracted people's sympathy so naturally by her generosity and care. She inspired me so much and I have learned so much through her. I am extremely grateful that I had a mother like her.

%% file: intro.tex
\chapter{Introduction}

What is intelligence? Are there general principles from which every aspect of intelligence derives? If yes, can we discover these principles, formulate them in mathematical language, and use them to build machines that possess human-like intelligence? And if we managed to do so, would this teach us something about ourselves and the human nature? Here are some of the fascinating questions at the interface of several disciplines of science, engineering and philosophy that motivated me to pursue a PhD in artificial intelligence.

In this introductory chapter, we start by motivating the brain-inspired approach to artificial intelligence (AI). We then review the most important principles of current AI systems. Then, we point out one of their weaknesses: the energy inefficiency of their current implementation in hardware. Finally, we present a novel mathematical framework which allows us to preserve the core principles of current AI systems, while suggesting a path to implement them in substrates that directly exploit physics to do the computations for us. In the long run, this mathematical framework may help us develop AI systems that are much larger, much faster and much more efficient than those that we use today.

\section{On Artificial Intelligence}

Computers are since long able to surpass humans at tasks that we like to think of as intellectually advanced. For example, in 1997, Deep Blue, a computer program created by IBM, beat the world champion Garry Kasparov at the game of chess \citep{campbell2002deep}.
While this was an impressive achievement, the form of intelligence that we may want to assign to Deep Blue is very rudimentary, though. The strategy of Deep Blue consists in analyzing essentially every possible scenario to pick the move leading to the best possible outcome. The state of technology at the time made it possible, with enough computing power, to automate this brute-force strategy. Humans on the other hand are not able to analyze every possible scenario at the game of chess in a reasonable amount of time. Our brains haven't evolved to do that. Instead, we develop intuitions about what are the most promising moves. Developing the right intuitions is what makes the game of chess challenging for us.

Conversely, there are plenty of tasks that humans (and other animals) do so naturally and effortlessly that it can be hard to appreciate the difficulty to build machines to automate them. Consider for example the task of classifying images of cats and dogs. Although we now have computer programs that can classify images fairly reliably, it is only in the past ten years that we have seen impressive improvements in this area. No one is said to be `intelligent' for being able to tell apart cats from dogs, given that a two-year old can already do this. So why was it so difficult to design programs to classify images? Solving this task is indeed deceptively more complex than it seems. In fact, when we see something, myriads of calculations are performed continuously and automatically in our brains. These calculations happen `behind the scenes', unconsciously, until the concept of `cat' or 'dog' pops up to our consciousness. We take for granted the fact that our brains do all these calculations for us, every second of every day. Perhaps one way to appreciate the difficulty that it represents for a programmer to write a program to classify images, is to have in mind that when our eyes see a cat, the computer program sees a bunch of numbers corresponding to pixel intensities (Fig.~\ref{fig:cifar-samples}). The difficulty for the programmer thus resides in making sense and handling these numbers in such a way that they produce the answer `cat'. Similarly, the human brain is able to learn to hear and recognize sounds, to learn to process the sense of touch, etc. We can perceive the world around us, make sense of it and interact with it like no computer or machine can do today. Undoubtedly, humans (and other animals) are in many ways incredibly smarter than machines today.

\begin{figure}[!ht]
    \begin{center}
        \includegraphics[width= 0.8 \textwidth]{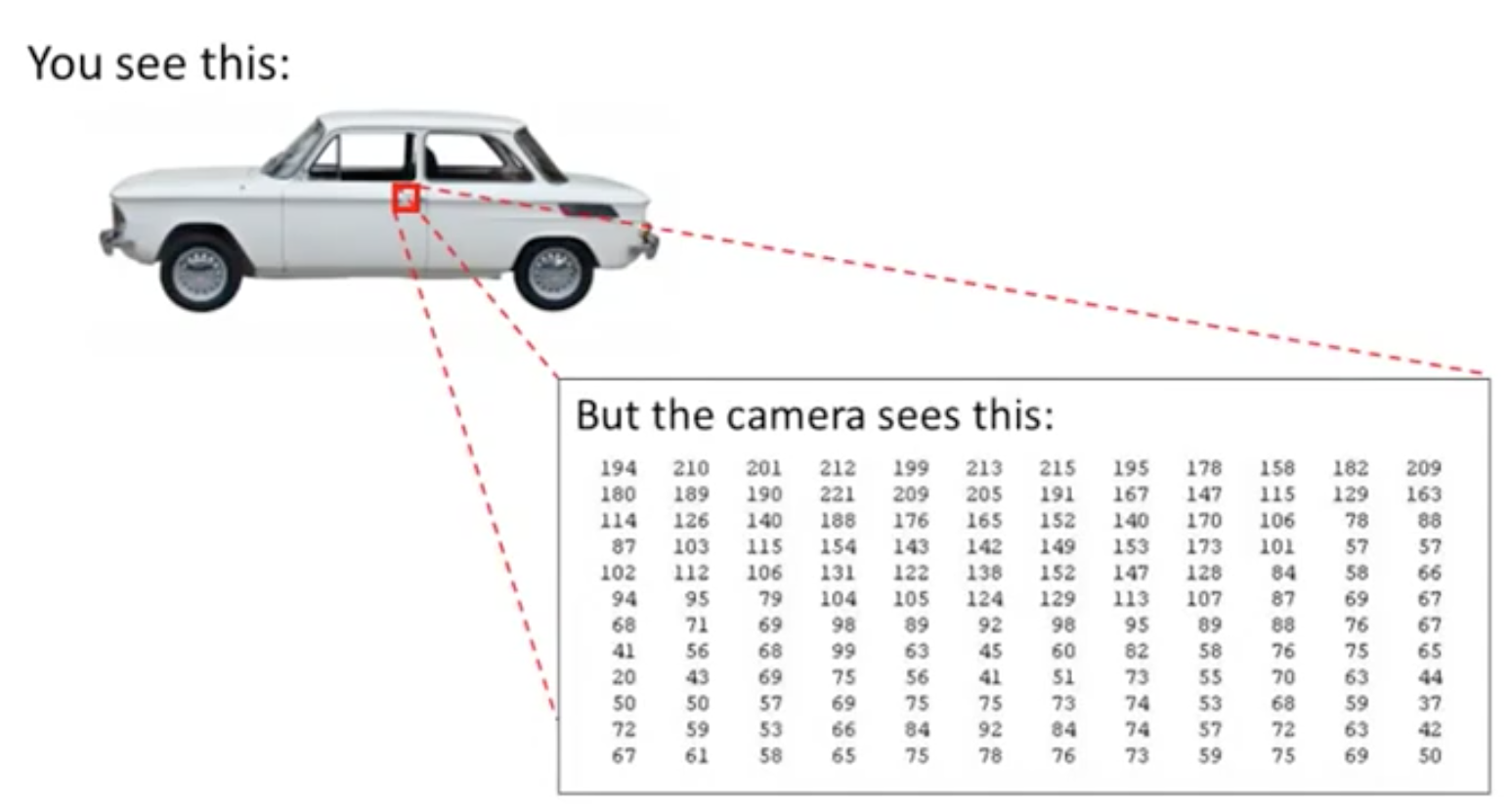}
        \caption[Picture of a car and the corresponding representation in computer language]{Picture of a car, and the corresponding representation in computer language. Each number represents the intensity of a pixel. The task of image classification consists in making sense of these numbers to produce the answer `car'. Figure credits: \citet{ng2014coursera}.}
        \label{fig:cifar-samples}
    \end{center}
\end{figure}

\subsection{Human Intelligence as a Benchmark}

Artificial intelligence (AI) emerged as a research discipline in the 1950s from the idea that every aspect of human intelligence can in principle be discovered, understood, and built into machines. The idea of AI started after Alan Turing formalized the notion of computation and began to study how computers can solve problems on their own \citep{turing2009computing}, but the term \textit{artificial intelligence} was first coined by John McCarthy in 1956. Today, AI is usually used in a broader sense, to refer more broadly to the field of study concerned with designing computational systems to solve practical tasks that we want to automate, whether or not such tasks require some form of human-like intelligence, and whether or not such computational systems are inspired by the brain. However, human intelligence is a natural benchmark for us to build `intelligent' machines. This is an arbitrary choice, and implies by no means that we humans should be regarded as the `perfect' or `ultimate' form of intelligence. But, until we are able to build machines that can do all the incredible things that we humans can do, as easily and as effortlessly as we do them, it seems rather natural to take human intelligence as a benchmark. In the quest of building intelligent machines, Turing proposed that the goal would be reached when our machines can exhibit intelligent behaviours indistinguishable from that of humans.

\subsection{Machine Learning Basics}

Consider the task of classifying images of cats and dogs mentioned earlier. Say that each image is made of $1000$ by $1000$ pixels, each pixel being described by three numbers (in the RGB color representation). Thus, each image can be represented by a vector of three million numbers. The goal is to come up with a program which, given such a vector $x$ as input, produces `dog' or `cat' as output, accordingly. Because there exist many very different vectors $x$ associated to the concept of `dog', there is no obvious, simple and reliable rule to recognize a dog. To solve this task, the program must combine a very large number of `weak rules'.

Early forms of AI consisted of explicit, manually-crafted rules, e.g. depending on formal logic. However, using this methodology to figure out all the weak rules that are necessary to correctly classify images is a really daunting task, given the complexity of real-world images. One of the key features of the brain, which these traditional programs did not have, is its ability to learn from experience and to adapt to the environment. Arthur Samuel introduced a new approach to AI, called \textit{machine learning} (ML) \citep{samuel1959some}, that takes inspiration from how we learn. In the ML approach, instead of operating with predetermined (i.e. immutable) instructions, the program is made of flexible rules that depend on adjustable parameters. As we modify the parameters, the program changes. The goal is then to tune these parameters so that the resulting program solves the task we want.

To solve the task of image classification, the ML approach requires to collect lots of examples that specify the correct output (the label) for a given input (the image). Such a collection of examples is called a \textit{dataset}. Then we use these examples to adjust the parameters of the ML program so that, for each input image, the program produces as output the label associated to that image. Such a procedure to adjust the parameters is called a \textit{learning algorithm}: the ML program \textit{learns} from examples to solve the problem. Once trained, the program obtained can then be used to predict outputs for new unseen inputs. The performance of the program is assessed on a separate set of examples called the \textit{test dataset}.

In the setting of image classification, the data is such that the labels are provided together with the corresponding images. This setting, where the expected result is known in advance for the available data, is called \textit{supervised learning}. This type of learning is currently the most widely used and successful approach to ML. Depending on the task that we want to solve, and the type and the amount of data that is available, there are two other main machine learning paradigms for training an ML program: unsupervised and reinforcement learning. Unsupervised learning refers to data for which no explicitly identified labels exist. Reinforcement learning refers to the case where no exact labels exist, but a scalar value is available (usually called `reward') that provides some knowledge on whether a proposed output is good or bad.

\section{Neural Networks}

Artificial neural networks (ANNs) are a family of ML models inspired by the basic working mechanisms of the brain. In recent years ANNs have had resounding success in AI, in areas as diverse as image recognition, speech recognition, image generation, speech synthesis, text generation and machine language translation. We start this section by presenting the basic concepts of neuroscience that have inspired the design of ANNs. Then we present the key principles at the heart of these ANNs. Finally we point out some weaknesses in the current implementation of these principles in hardware, which makes these neural networks orders of magnitude less energy efficient than brains.

Subsection \ref{sec:neuroscience} is inspired by \citet[Chapter 5]{dehaene2020we}.

\subsection{Neuroscience Basics}
\label{sec:neuroscience}

The foundations of modern neuroscience were laid by Santiago Ramon y Cajal, several decades before AI research started. Cajal was the first to observe the brain's micro-organisation with a microscope. He observed that the brain consists of disjoint nerve cells (the \textit{neurons}), not of a continuous network as the proponents of the \textit{reticular theory} thought before him. Neurons have a very particular shape. Each neuron is composed of three main parts (Figure \ref{fig:neuron}): a large `tree' composed of thousands of branches (the \textit{dendrites}\footnote{In Greek, the word \textit{dendron} means tree}), a cell body (also called the \textit{soma}), and a long fiber which extends out of the cell body towards other neurons (the \textit{axon}). A neuron collects information from other neurons through its dendritic tree. The messages collected in the dendrites converge to the cell body, where they are compiled. After compilation, the neuron sends a unique message, called \textit{action potential} (or \textit{spike}), which is carried along its \textit{axon} away from the cell body. In turn, this message is delivered to other neurons.

\begin{figure}[!ht]
	\begin{center}
		\includegraphics[width= 0.6 \textwidth]{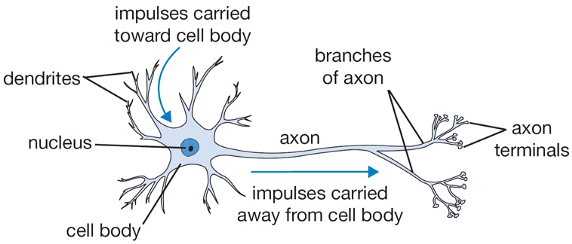}
		\caption[Schema of a neuron]{Schema of a neuron\footnotemark.}
	    \label{fig:neuron}
	\end{center}
\end{figure}

\footnotetext{https://towardsdatascience.com/a-gentle-introduction-to-neural-networks-series-part-1-2b90b87795bc}

While neurons are distinct cells, they come into contact at certain points called \textit{synapses} (Figure \ref{fig:synapse}). Synapses are junction zones through which neurons communicate. Specifically, each synapse is the point of contact of the axon of a neuron (called \textit{pre-synaptic} neuron) and the dendrite of another neuron (called \textit{post-synaptic} neuron). The message traveling through the axon of the pre-synaptic neuron is electrical, but the synapse turns it into a chemical message. The axon terminal of the pre-synaptic neuron contains some sorts of pockets (the \textit{vesicles}) filled with molecules (the \textit{neurotransmitters}). When the electrical signal reaches the axon terminal, the vesicles open and the neurotransmitters flow in the small synaptic gap between the two neurons. The neurotransmitters then bind with the membrane of the post-synaptic neuron at specific points (the \textit{receptors}). A neurotransmitter acts on a receptor as a key in a lock: they open `gates' (called \textit{channels}) in the post-synaptic membrane. As a result, ions flow from the extra-cellular fluid through these channels and generate a current in the post-synaptic neuron. To sum up, the message coming from the pre-synaptic neuron went from electrical to chemical, back to electrical, and in the process, the message was transmitted to the post-synaptic neuron.

Each synapse is a chemical factory in which numerous elements can be modified: the number of vesicles and their size, the number of receptors and their efficacy, as well as the size and the shape of the synapse itself. All these elements affect the strength with which the pre-synaptic electrical message is transmitted to the post-synaptic neuron. Synapses are constantly modified and these modifications reflect what we \textit{learn}.

\begin{figure}[!ht]
	\begin{center}
		\includegraphics[width= 0.5 \textwidth]{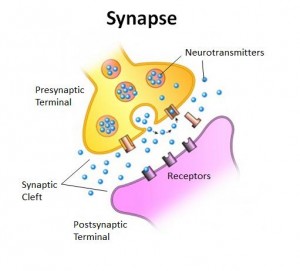}
		\caption[Schema of a synapse]{Schema of a synapse\footnotemark.}
	    \label{fig:synapse}
	\end{center}
\end{figure}

\footnotetext{https://thesalience.wordpress.com/neuroscience/the-chemical-synapse/chemical-synapses/}

The human brain is composed of around 100 billion ($10^{11}$) neurons, interconnected by a total of around a quadrillion ($10^{15}$) synapses. The brain is a huge parallel computer: in this incredibly complex machine, all the synapses work in parallel -- like independent nanoprocessors -- to process the messages sent between the neurons. Besides, synapses are modified in response to experience, and in turn these modifications alter our behaviours. Thus, synapses are both the computing units and the memory units of the brain. For every task we do, all our thoughts, memories and all our behaviours emerge from the neural activity generated by this machinery.

One of the fundamental questions of neuroscience is that of figuring out the \textit{learning algorithms} of the brain: what is the set of rules which translate the experiences we have into synaptic changes, and how do these synaptic changes modify our behaviour? Understanding the brain’s learning algorithms is not only key to understanding the biological basis of intelligence, but would also unlock the development of truly intelligent machines.

\subsection{Artificial Neural Networks}

Artificial neural networks are ML models that draw inspiration from real brains. Artificial neurons imitate the functionality of biological neurons. These models are highly simplified: they keep some essential ideas from real neurons and synapses but they discard many details of their working mechanisms. The first neuron model was introduced by \citet{mcculloch1943logical}, but the idea to use such artificial neurons in machine learning was proposed by \citet{rosenblatt1958perceptron} and \citet{widrow1960adaptive}. Artificial neurons used today in deep learning are essentially unchanged and rely on the same basic math algebra.

Each neuron $i$ is described by a single number $y_i$. This number can be thought of as the \textit{firing rate} of neuron $i$, that is the rate of spikes sent along its axon. Each synapse is also described by a single number, representing its strength. The strength of the synapse connecting pre-synaptic neuron $j$ to post-synaptic neuron $i$ is denoted $W_{ij}$. These artificial synapses can transmit signals with different efficacies depending on their strength. The neuron calculates a nonlinear function of the weighted sum of its inputs:
\begin{equation}
    \label{eq:artificial-neuron}
    y_i = \sigma \left( \sum_j W_{ij} y_j \right).
\end{equation}
The \textit{pre-activation} $x_i = \sum_j W_{ij} y_j$ is a weighted sum of the messages received from other neurons, weighted by the corresponding synaptic strengths. $x_i$ can be though of as the membrane voltage of neuron $i$. $\sigma$ is a function called an \textit{activation function}, which maps $x_i$ onto the firing-rate $y_i$.

Such artificial neurons can be combined to form an artificial neural network (ANN). Each neuron in the network receives messages from other neurons (the $y_j$'s), compiles them ($x_i$), and sends in turn a message to other neurons ($y_i$). Thus, a network of interconnected neurons exploits the composition of many elementary operations to form more complex computations. The synaptic strengths (the $W_{ij}$'s), also called \textit{weights}, play the role of adjustable parameters that parameterize this computation.

\textit{Deep learning} refers to ANNs composed of multiple layers of neurons \citep{lecun2015deep,goodfellow2016deep}. These \textit{deep neural networks} were inspired by the structure of the visual cortex in the brain, each layer corresponding to a different brain region. One of the core ideas of neural networks is that of \textit{distributed representations}, the idea that the vector of neuron's states can represent abstract concepts, by opposition to other approaches to AI that use discrete symbols to represent concepts. In a deep network, each layer of neurons applies specialized operations and transformations on its inputs, with the intuition that each layer builds up more abstract concepts than the previous \citep{bengio2009learning}.

\subsection{Energy-Based Models vs Differentiable Neural Networks}

Several families of neural networks emerged in the 1980s. One of these families is that of \textit{energy-based models}, which includes the Hopfield network \citep{hopfield1982neural} and the Boltzmann machine \citep{ackley1985learning}. In these models, under the assumption that the synaptic weights are symmetric (i.e. $W_{ij}=W_{ji}$ for every pair of neurons $i$ and $j$), the dynamics of the network converges to an equilibrium state, after iterating Eq.~\ref{eq:artificial-neuron} a large number of times for every neuron $i$. Because of the large number of iterations required, these models tend to be slow. This is one of the reasons why these neural networks have been mostly discontinued today. However, by reinterpreting the equilibrium equation of energy-based models as a variational principle of physics, I believe that these models could be the basis of a new generation of fast, efficient and scalable neural networks grounded in physics. We will come back to this point later in the discussion (Section \ref{sec:deep-learning-theory}).

The family of neural networks that is at the heart of the on-going deep learning revolution is that of \textit{differentiable neural networks}, which became popular thanks to the discovery of the \textit{backpropagation algorithm} to train them \citep{rumelhart1988learning}. In such neural networks, each operation in the process of computation is differentiable. The earliest models of this kind were feedforward neural networks (e.g. the multi-layer perceptron), wherein the connections between the neurons do not form loops. Recurrent neural networks can also be cast to this category of differentiable neural networks, by unfolding the graph of their computations in time. Since their inception, differentiable neural networks have come a long way. Many novel architectures have been introduced, in particular: convolutional neural networks \citep{fukushima1980neocognitron,lecun1989backpropagation}, Long Short-term Memory \citep{hochreiter1997long,graves2013generating}, and attention mechanisms \citep{bahdanau2014neural,vaswani2017attention}.

\subsection{Stochastic Gradient Descent}

The computations performed by a neural network are parameterized by its synaptic weights. The goal is to find weight values for which the computations of the neural network solve the task of interest. One essential idea of machine learning is to introduce a \textit{loss function}, which provides a numerical measure of how good or bad the computations of the model are, with respect to the task that we want to solve. The goal is then to minimize the loss function with respect to the model weights. For example, in image classification, the computations of the model produce an output which represents a `guess' for the class of that image, and the loss provides a graded measure of `wrongness' between that guess and the actual image label. A smaller value of the loss function means that the model produces an output closer to the desired target. The loss function is minimal when the output is equal to the desired target.

One of the most important ideas of deep learning today is \textit{stochastic gradient descent} (SGD). Provided that the loss function is differentiable with respect to the network weights, we can use the gradient of the loss function to indicate the direction of the minimum of this function. SGD consists in taking examples from the training set one at a time, and adjusting the network weights iteratively in proportion to the negative of the gradient of the loss function. At each iteration, the network performance (as measured per the loss value) slightly improves.

A key discovery that has greatly eased and accelerated deep learning research is the following. Given a computer program that computes a \textit{differentiable} scalar function $f: \mathbb{R}^n \to \mathbb{R}$, it is possible to automatically transform the program into another program that computes the gradient operator $\nabla f : \mathbb{R}^n \to \mathbb{R}^n$. The gradient $\nabla f(\theta)$ can then be evaluated at any given point $\theta \in \mathbb{R}^n$, with a computational overhead that scales linearly with the complexity of the original program. This technique, known as \textit{reverse-mode automatic differentiation} \citep{speelpenning1980compiling}, provides a general framework for `backpropagating loss gradients' \citep{rumelhart1988learning} in any \textit{differentiable computational graph}. In the last decade, dozens of deep learning frameworks and libraries have been developed, which exploit reverse-mode automatic differentiation to compute gradients in arbitrary differentiable computational graphs. This includes Theano \citep{bergstra2010theano} -- a framework that was developed at Université de Montréal -- and the more recent Tensorflow \citep{abadi2016tensorflow} and PyTorch \citep{paszke2017automatic} frameworks. The emergence of these deep learning frameworks has considerably accelerated deep learning research, by enabling researchers to quickly design neural network architectures and train them by SGD. Thanks to these frameworks, deep learning researchers can explore the space of differentiable computational graphs much more rapidly, as they seek novel and more effective neural architectures.

\subsection{Landscape of Loss Functions}

It was not trivial to discover that deep neural networks can be trained at all. Until the seminal work of \citet{hinton2006fast}, common belief was that neural networks with more than two layers were essentially impossible to train. In particular, because the landscape of the loss function associated to a deep network is typically highly non-convex, a common misconception was that gradient-descent methods would likely get stuck at bad local minima. In terms of generalization performance, the large over-parameterization of neural networks was also against general prescriptions from classical statistics and learning theory. One of the surprising discoveries of the deep learning revolution was that, in such highly non-convex and over-parameterized statistical models, provided that the loss landscape has appropriate shape, SGD can solve complex tasks by finding excellent parameter values that generalize well to unseen examples.

Several elements contributed to unlock the training of deep neural networks ; among others: the discovery \citep{glorot2011deep} that the ReLU (`Rectified Linear Unbounded') activation function usually outperforms the sigmoid activation function, the discovery of better weight initialization schemes \citep{glorot2010understanding,saxe2013exact,he2015delving}, and the batch-normalization technique to systematically normalize signal amplitudes at each layer of a network \citep{ioffe2015batch}. Besides, fundamental advances have come from the introduction of new network architectures such as the ones mentioned earlier, and from novel machine learning paradigms such as that of generative adversarial networks \citep{goodfellow2014generative}. All these techniques have as an effect to modify the landscape of the loss function, as well as the starting parameter (before training) in the parameter space. Understanding how the landscape of the loss function can be appropriately shaped to ease optimization by SGD is an active area of research \citep{poggio2017theory,arora2018toward}.

\subsection{Deep Learning Revolution}

In recent years, deep neural networks have proved capable of solving very complex problems across a wide variety of domains. Today, they achieve state-of-the-art performance in image recognition \citep{he2016deep}, speech recognition \citep{hinton2012deep,chan2016listen}, machine translation \citep{vaswani2017attention}, image-to-text \citep{sharma2018conceptual}, text-to-speech \citep{oord2016wavenet}, text generation \citep{brown2020language}, and synthesis of realistic-looking portraits \citep{karras2019style}, among many other applications. Neural networks have become better than humans at playing Atari games \citep{mnih2013playing}, playing the game of Go \citep{silver2018general} and playing Starcraft \citep{vinyals2019alphastar}.
Perhaps what is most exciting is that, although these neural networks are designed to solve very different tasks and deal with different types of data, they are all trained using the same handful of basic principles. As we scale these neural networks, more advanced aspects of intelligence seem to emerge from this handful of principles.

While the pace of progress in neural network research is breathtaking, we should also emphasize that, without any question, neural networks are nowhere close to `surpass' humans. In their current form, they miss key elements of human intelligence. Whenever neural networks beat humans, they beat us at a very specific task and/or under very specific conditions.
We may need new learning paradigms to move away from `task-specific' neural networks towards `multi-functional' and continually learning neural networks.
Besides, as of today, neural networks cannot handle and combine abstract concepts nearly as flexibly as humans do. Making progress along these lines may require new breakthroughs, to give these neural networks the ability to develop a `thought language' and a sense of causal reasoning, among others.

\subsection{Graphics Processing Units}

The on-going deep learning revolution owes to other technological developments too. In the past decades, the amount of data available has greatly increased, and powerful multi-core parallel processors for general purpose computing, such as Graphics Processing Units (GPUs), have emerged \citep{owens2007survey}. Training large models on large datasets -- here is part of the recipe to make a deep neural network solve a challenging task. In the 1980s, when deep neural networks were first conceived, the lack of data and computational power to train them made it practically unfeasible to demonstrate their effectiveness.

The last decade of neural network research seems to hint at a simple and straightforward strategy to further improve performance of our AI systems: scaling. Using more memory and more compute to train larger models on larger datasets -- here is the current trend to build state-of-the-art deep learning systems. The largest model ever built thus far, GPT-3 \citep{brown2020language}, has a capacity of 175 billion parameters.

Training these neural networks requires very large amounts of computations. Standard practice today is to distribute the computations across more and more GPUs, to train larger and larger models. Still, even using thousands of GPUs working in parallel, training these neural networks can take months. For example, AlphaZero learnt to play the game of Go by playing 140 million games, which took 5000 processors and two weeks. Moreover, training such models can cost millions of dollars, just for electricity consumption, not to mention their ecological impact. Yet, even these large neural networks are only a tiny fraction of the size of the human brain. What causes such inefficiency, preventing us from building models of the size of the human brain?

\subsection{The Von Neumann Bottleneck}
\label{sec:von-neumann-bottleneck}

If at the conceptual level the neural networks used today take their overall strategy from the brain, on the hardware implementation level however, they use little of the cleverness of nature. Our current processors, on which these neural networks are trained and run, operate in fundamentally different ways than brains. They rely on the \textit{von Neumann architecture}. In this computer architecture, the memory unit where information is stored, is separated from the processing unit where calculations are done. A so-called \textit{bus} moves information back and forth between these two units. Over the course of history of computing, this computer architecture has become the norm, and today, the von Neumann architecture is used in virtually all computer systems: laptops, smartphones, and all kinds of embedded systems. The GPUs, massively used for neural network training today, also rely on the von Neuman architecture, where memory is separated from computing.

The brain on the other hand deeply merges memory and computing by using the same functional unit: the synapse. The human brain is composed of a quadrillion ($10^{15}$) synapses. In other words, the human brain has $10^{15}$ nanoprocessors working in parallel.

Training neural networks on von Neumann hardware as we do it today is extraordinarily energy inefficient in comparison with the way brains operate. The necessity to move the data back and forth between the memory and processing units in the von Neumann architecture is energy intensive and creates considerable latency. This limitation is known as the \textit{von Neumann bottleneck}. How inefficient is it compared to biological systems like the brain? The human brain is composed of $10^{11}$ neurons and consumes around 20W to conduct all of its activities \citep{attwell2001energy}. In comparison, training a BERT model (a state-of-the-art natural language processing model) on a modern supercomputer requires 1500 kW.h \citep{strubell2019energy}, which is the total amount of energy consumed by a brain in nine years. Besides, a GPU running real-time object detection with YOLO \citep{redmon2016you}, a network smaller than the brain by four orders of magnitude, consumes around 200W.

This striking mismatch holds more broadly with biological systems in general. For example, \citet{kempes2017thermodynamic} study the energy efficiency of `cellular computation' in the process of biological translation. What is the amount of energy required (in ATP equivalents) by ribosomes to assemble amino acids into proteins ?
They point out that "the best supercomputers perform a bit operation at roughly $5.27 \times 10^{-13} J$, [...] which is about five orders of magnitude less efficient than biological translation."

To sum up, our current neural networks are orders of magnitude less energy efficient than biological systems at processing information, and the von Neumann bottleneck is largely responsible for this inefficiency. If using more and more GPUs may increase speed, and thereby speed up training and inference of neural networks, this strategy however can't improve energy efficiency.

\subsection{In-Memory Computing}

In order to build massively parallel neural networks that are energy efficient and can scale to the size of the human brain, we need to fundamentally rethink the underlying computing hardware. We need to design neural networks so that computations are performed at the physical location of the synapses, where the strength of the connections (the weights of the neural network) are stored and adjusted, just like in the brain. The concept of hardware that merges memory and computing is called \textit{in-memory computing} (or \textit{in-memory processing}), and the field tackling this problem is called \textit{neuromorphic computing}. This field of research, started by Carver Mead in the 1980s \citep{mead1989analog} aims at mimicking brains at a hardware level, by building physical neurons and synapses onto a chip.

The most common approach to in-memory computing today is to use \textit{programmable resistors} as synapses. Programmable resistors, such as \textit{memristors} \citep{chua1971memristor}, are resistors whose conductance can be changed (or `programmed'). The weights of a neural network can be encoded in the conductance of such devices. In the last decade, important advances in nanotechnology were made, and a number of new technologies have emerged and have been studied as potentially promising programmable resistors \citep{burr2017neuromorphic,xia2019memristive}.

Neuromorphic computing thus explores analog computations that fundamentally depart from the standard digital computing paradigm.

\subsection{Challenges of Analog Computing}

Analog processing differs from digital processing in important ways. Whereas digital circuits manipulate binary signals with reliably distinguishable \textit{on} and \textit{off}-states, analog circuits on the other hand manipulate real-valued currents and voltages that are subject to analog noise that bounds the precision with which computation may be performed. More importantly, analog devices suffer from \textit{mismatches}, i.e. small random variations in the physical characteristics of devices, which occurs during their manufacturing. No two devices are exactly alike in their characteristics, and it is impossible to make a perfect clone of one. These variations result in behavioral differences between identically designed devices. Due to the accumulation of the mismatch errors from individual devices, it is very difficult to analytically predict the behavior of a large analog circuit.

A growing field of research in the neuromorphic literature attempts to perform in analog the operations that we normally do in software, so as to implement feedforward neural networks and the backpropagation algorithm efficiently. In this approach, the starting point is an equation of the kind of Eq.~\ref{eq:artificial-neuron}. Many of these operations are then performed in analog and combined to form the computations of a feedforward network. However, because of device mismatches, it is hard to perform such idealized operations, and as we combine many of these operations, the resulting computation may be different from the desired one. Either these idealized operations are performed with low precision, or we may spend a lot of energy trying to improve precision, e.g. by using analog-to-digital conversion.

Not coincidentally, the constraint of device nonidealities and device variability is shared with biology too. No two neurons are exactly the same. This realization demonstrates, in principle, that it is possible to train (biological) neural networks even in the presence of noise and imperfect `devices'. It invites us to rethink the learning algorithm for neural networks (and the notion of computation altogether).

\section{A Deep Learning Theory for Neural Networks Grounded in Physics}
\label{sec:deep-learning-theory}

In this thesis, we propose an alternative theoretical framework for neural network inference and training, with potential implications for neuromorphic computing. Our theoretical framework preserves the key principles that power deep learning today, such as optimization by stochastic gradient descent (SGD), but we use variational formulations of the laws of physics as first principles, so as to directly implement neural networks in physics. We present two very broad classes of neural network models, called \textit{energy-based models} and \textit{Lagrangian-based models}, whose state or dynamics derive from variational principles. The learning algorithm, called \textit{equilibrium propagation} (EqProp), enables to estimate the gradients of arbitrary loss functions in such physical systems using solely locally available information for each parameter.

\subsection{Physical Systems as Deep Learning Models}

The general idea of the manuscript is the following. We consider a physical system composed of multiple parts whose characteristics and working mechanisms may be only partially known. The system has some `adjustable parameters', some of them playing the role of `inputs', and we may read or measure a `response' on some other `output' part of the system. We can think of this black box system as performing computations and implementing a nonlinear input-to-output mapping function (which may be analytically unknown). We wish to tune the adjustable parameters of this system by stochastic gradient descent (SGD), as we normally do in deep learning. The question is: how can we compute or estimate the gradients in such a physical system, by relying on the physics of the system?

The main theoretical result of the thesis is that, for a large class of physical systems (those whose state or dynamics derive from a variational principle), there is a simple procedure to estimate the parameter gradients, which in many practical situations requires only locally available information for each parameter.
This procedure, called \textit{equilibrium propagation} (EqProp), preserves the key benefit of being compatible with SGD, while offering the possibility to directly exploit physics to implement and train neural networks.

\subsection{Variational Principles of Physics as First Principles}

Rather than Eq.~\ref{eq:artificial-neuron}, our starting point is a \textit{variational equation} of the form
\begin{equation}
    \label{eq:variational-principle}
    \frac{\partial E}{\partial s}=0,
\end{equation}
where $E$ is a scalar function. If $s$ is the state of the system, then Eq.~\ref{eq:variational-principle} is an equilibrium condition. In this case, we say that the system is an \textit{energy-based model} (EBM) and we call $E$ the \textit{energy function}.

In this thesis, we also introduce the concept of \textit{Lagrangian-based model} (LBM). Variational equations exist not just to characterize equilibrium states, but also entire trajectories. Many physical systems are such that their trajectory derives from a \textit{principle of stationary action} (e.g. a principle of least action). Denoting $\mathrm{s_t}$ the state of the system at time $t$, this means that the (continuous-time) trajectory $\mathrm{s} = \{ \mathrm{s_t} \}_{0 \leq t \leq T}$ over a time interval $[0,T]$ minimizes a functional of the form
\begin{equation}
    \mathcal{S} = \int_0^T L(\mathrm{s_t},\mathrm{\dot{s}_t}) dt,
\end{equation}
where $\mathrm{\dot{s}_t}$ is the time derivative of $\mathrm{s_t}$, $L$ is a function called the \textit{Lagrangian function} of the system, and $\mathcal{S}$ is a scalar functional called the \textit{action}. The stationarity of the action tells us that $\frac{\delta \mathcal{S}}{\delta \mathrm{s}}=0$, which is another variational equation of the kind of Eq.~\ref{eq:variational-principle}. These systems, which we call Lagrangian-based models (LBMs), are suitable in particular in the setting with time-varying inputs and can thus play the role of `recurrent neural networks'.

Equilibrium propagation (EqProp) allows to compute gradients with respect to arbitrary loss functions in these EBMs and LBMs. Furthermore, if the energy function (resp. Lagrangian function) of the system has a property called \textit{sum-separability}, meaning that it is the sum of the energies (resp. Lagrangians) of its parts, then computing the loss gradients with EqProp requires only information that is locally available for each parameter (i.e. the learning rule is \textit{local}).

\subsection{Universality of Variational Principles in Physics}

In the 1650s, Pierre de Fermat proposed the \textit{principle of least time} which states that, between two given points, the light travels the path which takes the least time. He showed that both the laws of reflection and refraction can be derived from this principle. Fermat's least time principle is an instance of what we call more generally a \textit{variational principle}. Today, the variational approach pervades much of modern physics and engineering \citep{lanczos2012variational}, with applications not only in optics, but also in mechanics, electromagnetism, thermodynamics, etc. Even at a fundamental level of description, our universe seems to behave according to variational principles: for example, Einstein's equations of general relativity can be derived from the Einstein-Hilbert action, and in a sense, Feynman's path integral formulation of quantum mechanics can be seen as a generalized principle of least action \citep{feynman2005principle}.

Thus, many physical systems qualify as energy-based models or Lagrangian-based models. This offers in principle a lot of options for implementing our proposed method on physical substrates. In this manuscript we will present one such option in details: \textit{nonlinear resistive networks}.

\subsection{Rethinking the Notion of Computation}

Interestingly, our theoretical framework invites us to rethink not only the von Neumann architecture on which our current processors rely, but also the notion of computation altogether.

Much of computer science deals with computation in the abstract, without worrying about physical implementation \citep{lee2017plato}. Our computers today rely on the computing paradigm introduced by Turing, where computers operate on digital data and carry out computations algorithmically, via step-by-step (discrete-time) processes. The von Neumann architecture was invented to implement these computations (i.e. to bridge the gap between physics and these abstract computations), but suffers from the speed and energy efficiency problems mentioned earlier (Section~\ref{sec:von-neumann-bottleneck}).

The theoretical framework presented in this manuscript can be seen as an alternative approach to computing that takes advantage of the ways by which Nature operates. We suggest a novel computing paradigm, which uses the variational formulations of the laws of physics as first principles. Together with the \textit{equilibrium propagation} (EqProp) training procedure, our approach suggests a way to implement the core principles of deep learning by exploiting physics directly. As will become apparent in Section \ref{sec:resistive-networks-algo}, the process of `computations' in EqProp is very different from the step-by-step processes of Turing's conventional computing paradigm. Although we may call EqProp a learning `algorithm', it is not an algorithm in the conventional sense (one that performs step-by-step computations).

\subsection{A Novel Differentiation Method Compatible with Variational Principles}

More technically, the main idea of the manuscript can be summarized by the following mathematical identity -- a novel differentiation method compatible with variational principles. Consider two functions $E(\theta,s)$ and $C(s)$ of the two variables $\theta$ and $s$. We wish to compute the gradient $\frac{d}{d\theta}C(s_\theta)$, where $s_\theta$ is such that $\frac{\partial E}{\partial s}(\theta,s_\theta) = 0$. To do this, we introduce the function $F(\theta,\beta,s) = E(\theta,s) + \beta \; C(s)$, where $\beta$ is a scalar. For fixed $\theta$, we further define $s_\star^\beta$ by the relationship $\frac{\partial F}{\partial s}(\theta,\beta,s_\star^\beta) = 0$, for any $\beta$ in the neighborhood of $\beta=0$. In particular $s_\star^0 = s_\theta$. Then we have the identity
\begin{equation}
    \frac{d}{d\theta}C(s_\theta) = \left. \frac{d}{d\beta} \right|_{\beta=0} \frac{\partial E}{\partial \theta}(\theta,s_\star^\beta),
\end{equation}
where $\frac{\partial E}{\partial \theta}$ denotes the partial derivative of the function $E$ with respect to its first argument.
This result is developed and proved in Chapter \ref{chapter:eqprop}.

\section{Overview of the Manuscript and Link to Prior Works}

The manuscript is organized as follows.
\begin{itemize}
    \item In Chapter \ref{chapter:eqprop}, we present EqProp in its original formulation, as a learning algorithm to train energy-based models (EBMs). We show that, provided that the energy function of the system is \textit{sum-separable}, then the learning rule of EqProp is local. This corresponds to Section 3 and Appendix A of \citet{scellier2017equilibrium}.
    \item In Chapter \ref{chapter:hopfield}, we use EqProp to train a particular class of EBMs called \textit{gradient systems}. This includes the continuous Hopfield network, a neural network model introduced by Hopfield in the 1980s, studied by both the neurosience community and the neuromorphic community. In this setting, the learning rule of EqProp is a form of contrastive Hebbian learning. The first part of this chapter corresponds to the result established in \citet{scellier2019equivalence}. The second part corresponds to sections 2, 4, and 5 of \citet{scellier2017equilibrium}.
    \item In Chapter \ref{chapter:neuromorphic}, we show that a class of analog neural networks called \textit{nonlinear resistive networks} are energy-based models: they possess an energy function called the \textit{co-content} of the circuit, as a reformulation of Kirchhoff's laws. Furthermore the co-content has the sum-separability property. Therefore we can train these nonlinear resistive networks with EqProp using a local learning rule. This chapter corresponds to \citet{kendall2020training}.
    \item In Chapter \ref{chapter:discrete-time}, we present a class of discrete-time neural network models trainable with EqProp, which is useful to accelerate computer simulations. This formulation, which uses notations closer to those used in conventional deep learning, is also more adapted to train more advanced network architectures such as convolutional networks. This chapter corresponds to \citet{ernoult2019updates} and \citet{laborieux2020scaling}.
    \item In Chapter \ref{chapter:future}, we present on-going developments. In particular, we introduce the concept of \textit{Lagrangian-based models}, a wide class of machine learning models that can serve as recurrent neural networks and can be implemented directly in physics by exploiting the \textit{principle of stationary action}. These Lagrangian-based models can also be trained with an EqProp-like training procedure. We also present an extension of EqProp to stochastic systems, which was introduced in Appendix C of \citet{scellier2017equilibrium}. Finally, we briefly present the \textit{contrastive meta-learning} framework of \citet{zucchet2021contrastive}, which uses the EqProp technique to train the meta-parameters of a meta-learning model.
\end{itemize}

%% file: eqprop.tex
\chapter{Equilibrium Propagation: A Learning Algorithm for Systems Described by Variational Equations}
\label{chapter:eqprop}

Much of machine learning today is powered by stochastic gradient descent (SGD). The standard method to compute the loss gradients required at each iteration of SGD is the backpropagation (Backprop) algorithm. Equilibrium propagation (EqProp) is an alternative to Backprop to compute the loss gradients. The difference between EqProp and Backprop lies in the class of models that they apply to: while Backprop applies to \textit{differentiable neural networks}, EqProp is broadly applicable to systems described by \textit{variational equations}, i.e. systems whose state or dynamics is a stationary point of a scalar function or functional. Since many physical systems have this property \citep{lanczos2012variational}, EqProp offers the perspective to implement and train machine learning models which use the laws of physics at their core.

In this chapter, we present EqProp in its original formulation \citep{scellier2017equilibrium}, as an algorithm to train energy-based models (EBMs). EBMs are systems whose equilibrium states are stationary points of a scalar function called the \textit{energy function}. EBMs are suitable in particular when the input data is static. In most of the manuscript, we consider for simplicity of presentation the supervised learning setting with static input, e.g. the setting of image classification where the input is an image and the target is the category of that image. However, EqProp is applicable beyond this setting. In Section \ref{sec:time-varying-setting}, we introduce the concept of \textit{Lagrangian-based models} (LBMs) which, by definition, are physical systems whose dynamics derives from a \textit{principle of stationary action}, and we show how EqProp can be applied to such systems. LBMs are suitable in the context of time-varying data, and can thus play the role of `recurrent neural networks'. We also present an extension of EqProp to stochastic systems (Section \ref{sec:stochastic-setting}), and to the setting of meta-learning (Section \ref{sec:contrastive-meta-learning}).

The present chapter is organized as follows.
\begin{itemize}
    \item In section \ref{sec:sgd}, we present the stochastic gradient descent (SGD) algorithm, which is at the heart of current deep learning. We present SGD in the setting of supervised learning, which we will consider in most of the manuscript to illustrate the ideas of the equilibrium propagation training framework. We note however that SGD is also the workhorse of state-of-the-art unsupervised and reinforcement learning algorithms.
    \item In section \ref{sec:energy-based-models} we define the notion of \textit{energy-based model} (EBM) that we will use throughout the manuscript.
    \item In section \ref{sec:loss-gradients}, we present the general formula for computing the loss gradients in an EBM, and in section \ref{sec:equilibrium-propagation}, we present the equilibrium propagation (EqProp) algorithm to estimate the loss gradients. Under the assumption that the energy function of the system satisfies a property called \textit{sum-separability}, the learning rule for each parameter is local.
    \item In section \ref{sec:ebms-examples}, we give a few examples of models trainable with EqProp, which we will study in the next chapters. Besides the well-known Hopfield model, nonlinear resistive networks, flow networks and elastic networks are examples of sum-separable energy-based models and, as such, are trainable with EqProp using a local learning rule.
    \item In section \ref{sec:remarks}, we discuss the general applicability of the framework presented here and the conditions under which EqProp is applicable.
\end{itemize}

\section{Stochastic Gradient Descent}
\label{sec:sgd}

In most of the manuscript, we consider the supervised learning setting, e.g. the setting of image classification where the data consists of images together with the labels associated to these images. In this scenario, we want to build a system that is able, given an input $x$, to `predict' the label $y$ associated to $x$. To do this, we design a \textit{parametric} system, meaning a system that depends on a set of \textit{adjustable parameters} denoted $\theta$. Given an input $x$, the system produces an output $f(\theta, x)$ which represents a `guess' for the label of $x$. Thus, the system implements a mapping function $f(\theta, \cdot)$ from an input space (the space of $x$) to an output space (the space of $y$), parameterized by $\theta$. The goal is to tune $\theta$ so that for most $x$ of interest, the output $f(\theta, x)$ is close to the target $y$. The `closeness' between $f(\theta, x)$ and $y$ is measured using a scalar function $C(f(\theta, x), y)$ called the \textit{cost function}. The overall performance of the system is measured by the expected cost $\mathcal{R}(\theta) = \mathbb E_{(x, y)}[C(f(\theta, x), y)]$ over examples $(x, y)$ from the data distribution of interest. The goal is then to minimize $\mathcal{R}(\theta)$ with respect to $\theta$.

In deep learning, the core idea and leading approach to tune the set of parameters $\theta$ is \textit{stochastic gradient descent} (SGD). The first step consists in gathering a (large) dataset of examples $\mathcal{D}_{\rm train} = \{ (x^{(i)}, y^{(i)}) \}_{1 \leq i \leq N}$, called \textit{training set}, which specifies for each input $x^{(i)}$ the correct output $y^{(i)}$. Then, each step of the training process proceeds as follows. First, a sample $(x, y)$ is drawn from the training set. Input $x$ is presented to the system, which produces $f(\theta, x)$ as output. This output is compared with $y$ to evaluate the loss
\begin{equation}
    \mathcal{L}(\theta, x, y) = C(f(\theta, x), y).
\end{equation}
Subsequently, the gradient $\frac{\partial {\mathcal L}}{\partial \theta} \left( \theta, x, y \right)$ is computed or estimated using some procedure, and the parameters are updated proportionally to the loss gradient:
\begin{equation}
    \Delta \theta = - \eta \frac{\partial {\mathcal L}}{\partial \theta} \left( \theta, x, y \right),
\end{equation}
where $\eta$ is a step-size parameter called \textit{learning rate}. This process is repeated multiple times (often millions of times) until convergence (or until desired). Once trained, the performance of the system is evaluated on a separate set of \textit{previously unseen} examples, called \textit{test set} and  denoted $\mathcal{D}_{\rm test} = \{ (x_{\rm test}^{(i)}, y_{\rm test}^{(i)}) \}_{1 \leq i \leq M}$. The \textit{test loss} is $\widehat{\mathcal{R}}_{\rm test}(\theta) = \frac{1}{M} \sum_{i=1}^M \mathcal{L} \left( \theta, x_{\rm test}^{(i)}, y_{\rm test}^{(i)} \right)$.

Several variants of SGD have been proposed, which use adaptive learning rates to accelerate the optimization process. This includes the \textit{momentum method} \citep{sutskever2013importance}
and \textit{Adam} \citep{kingma2014adam}. In some cases, these methods are not only faster, but also achieve better test performance than standard SGD. Besides, common practice is to average the loss gradients over \textit{mini-batches} of data examples before updating the weights -- a method sometimes called \textit{mini-batch gradient descent} -- but in this manuscript we consider for simplicity of presentation that training examples are processed one at a time. 

The SGD algorithm described above powers nearly all of deep learning today. There are, however, two ingredients that we have not specified so far: the `system' that implements the mapping function $f(\theta, x)$, and the `procedure' to compute the loss gradient $\frac{\partial {\mathcal L}}{\partial \theta} \left( \theta, x, y \right)$. In conventional deep learning, the `system' is a \textit{differentiable neural network}, and the loss gradients are computed with the \textit{backpropagation algorithm}. In this chapter, we present an alternative framework for optimization by SGD, where the `system' (i.e. the neural network) is an \textit{energy-based model}, and the procedure to compute the loss gradients is called \textit{equilibrium propagation} (EqProp).

\section{Energy-Based Models}
\label{sec:energy-based-models}

There exist different definitions for the concept of energy-based model in the literature -- see \citet{lecun2006tutorial} for a tutorial. In this manuscript, we reserve the term to refer to the specific class of machine learning models described in this section.

In the context of supervised learning, an \textit{energy-based model} (EBM) is specified by three variables: a parameter variable $\theta$, an input variable $x$, and a state variable $s$. An essential ingredient of an EBM is the \textit{energy function}, which is a scalar function $E$ that specifies how the state $s$ depends on the parameter $\theta$ and the input $x$. Given $\theta$ and $x$, the energy function associates to each \textit{conceivable} configuration $s$ a real number $E(\theta, x, s)$. Among all conceivable configurations, the \textit{effective} configuration of the system is by definition a state $s(\theta, x)$ such that
\begin{equation}
    \label{eq:free-equilibrium-state}
    \frac{\partial E}{\partial s}(\theta, x, s(\theta, x)) = 0.
\end{equation}
We call $s(\theta, x)$ an \textit{equilibrium state} of the system. The aim is to minimize the loss at equilibrium:
\begin{equation}
    \mathcal{L}(\theta, x, y) = C(s(\theta, x), y).
\end{equation}
In this expression, $s(\theta, x)$ is the `prediction' from the model, and plays the role of the `output' $f(\theta, x)$ of the previous section. A conceptual difference between $s(\theta, x)$ and $f(\theta, x)$ is that, in conventional deep learning, $f(\theta, x)$ is usually thought of as the output layer of the model (i.e. the last layer of the neural network), whereas here $s(\theta, x)$ represents the entire state of the system. Another difference is that $f(\theta, x)$ is usually \textit{explicitly} determined by $\theta$ and $x$ through an analytical formula, whereas here $s(\theta, x)$ is \textit{implicitly} specified through the variational equation of Eq.~(\ref{eq:free-equilibrium-state}) and may not be expressible by an analytical formula in terms of $\theta$ and $x$. In particular, there exists in general several such states $s(\theta, x)$ that satisfy Eq.~(\ref{eq:free-equilibrium-state}). We further point out that $s(\theta, x)$ need not be a minimum of the energy function $E$ ; it may be a maximum or more generally any saddle point of $E$.

We note that, just like the energy function $E$, the cost function $C$ is defined for any conceivable configuration $s$, not just the equilibrium state $s(\theta, x)$. Although $C(s, y)$ may depend on the entire state $s$, in practical situations that we will study in the next chapters, $C(s, y)$ depends only on a subset of $s$ that plays the role of `outputs'.

We also introduce another key concept: the concept of \textit{sum-separability}. Let $\theta = (\theta_1, \ldots, \theta_N)$ be the adjustable parameters of the system. For each $\theta_k$, we denote $\{x, s\}_k$ the information about $(x, s)$ which is locally available to $\theta_k$. We say that the energy function $E$ is \textit{sum-separable} if it is of the form
\begin{equation}
    \label{eq:sum-separability}
    E(\theta, x, s) = E_0(x, s) + \sum_{k=1}^N E_k(\theta_k, \{x, s\}_k),
\end{equation}
where $E_0(x, s)$ is a term that is independent of the parameters to be adjusted, and $E_k$ is a scalar function of $\theta_k$ and $\{x, s\}_k$, for each $k \in \{ 1, \ldots, N \}$. Importantly, many physical systems are energy-based models, many of which have the sum-separability property; we give examples in section \ref{sec:ebms-examples}.

\section{Gradient Formula}
\label{sec:loss-gradients}

The central ingredient of the equilibrium propagation training method is the \textit{total energy function} $F$, defined by $F = E + \beta \; C$, where $\beta$ is a real-valued variable called \textit{nudging factor}. The intuition here is that we augment the energy of the system by bringing an additional energy term $\beta C$. By varying $\beta$, the total energy $F$ is modified, and so is the equilibrium state relative to $F$. Specifically, assuming that the functions $E$ and $C$ are continuously differentiable, there exists a continuous mapping $\beta \mapsto s_\star^\beta$ such that $s_\star^0 = s(\theta, x)$ and\footnote{We note that it is also possible to define $s_\star^\beta$ differently, by the relationship $\frac{\partial E}{\partial s}(\theta, x, s_\star^\beta) + \beta \; \frac{\partial C}{\partial s}(s(\theta, x), y) = 0$, without changing the conclusions of Theorem \ref{thm:static-eqprop}. In Chapter \ref{chapter:neuromorphic} we will use this modified definition of $s_\star^\beta$.}
\begin{equation}
    \label{eq:nudged-steady-state}
    \frac{\partial E}{\partial s}(\theta, x, s_\star^\beta) + \beta \; \frac{\partial C}{\partial s}(s_\star^\beta, y) = 0
\end{equation}
for any value of the nudging factor $\beta$. Theorem \ref{thm:static-eqprop} provides a formula to compute the loss gradients by varying the nudging factor $\beta$.

\medskip

\begin{thm}[Gradient formula for energy-based models]
\label{thm:static-eqprop}
The gradient of the loss is equal to
\begin{equation}
    \label{eq:static-eqprop}
    \frac{\partial \mathcal{L}}{\partial \theta}(\theta, x, y) = \left. \frac{d}{d\beta} \right|_{\beta=0} \frac{\partial E}{\partial \theta} \left( \theta, x, s_\star^\beta \right).
\end{equation}
Furthermore, if the energy function $E$ is sum-separable, then the loss gradient for each parameter $\theta_k$ depends only on information that is locally available to $\theta_k$:
\begin{equation}
    \label{eq:static-eqprop-local}
    \frac{\partial \mathcal{L}}{\partial \theta_k}(\theta, x, y) = \left. \frac{d}{d\beta} \right|_{\beta=0} \frac{\partial E_k}{\partial \theta_k} \left( \theta_k, \{x, s_\star^\beta\}_k \right).
\end{equation}
\end{thm}

\begin{proof}
Eq.~(\ref{eq:static-eqprop}) is a consequence of Lemma \ref{lma:main} (Section \ref{chapter:fundamental-lemma}) applied to the total energy function\footnote{With the modified definition of $s_\star^\beta$, the total energy function to consider is $F(\theta, \beta, s) = E(\theta, x, s) + \beta \; C(s(\theta, x), y)$.}
$F$, defined for a fixed input-target pair $(x, y)$ by $F(\theta, \beta, s) = E(\theta, x, s) + \beta \; C(s, y)$, at the point $\beta=0$. Eq.~(\ref{eq:static-eqprop-local}) is a consequence of Eq.~(\ref{eq:static-eqprop}) and the definition of sum-separability (Eq.~(\ref{eq:sum-separability})).
\end{proof}

\section{Equilibrium Propagation}
\label{sec:equilibrium-propagation}

We can use Theorem \ref{thm:static-eqprop} to derive a learning algorithm for energy-based models. Let us assume that the energy function is sum-separable. We can estimate the loss gradients using finite differences, for example with
\begin{equation}
\label{eq:one-sided-estimator}
\widehat{\nabla}_{\theta_k}(\beta) = \frac{1}{\beta} \left( \frac{\partial E_k}{\partial \theta_k}(\theta_k, \{ x, s_\star^\beta \}_k) - \frac{\partial E_k}{\partial \theta_k}(\theta_k, \{ x, s_\star^0 \}_k) \right)
\end{equation}
to approximate the right-hand side of Eq.~(\ref{eq:static-eqprop-local}). We arrive at the following two-phase training procedure to update the parameters in proportion to their loss gradients.

\paragraph{Free phase (inference).}
The nudging factor $\beta$ is set to zero, and the system settles to an equilibrium state $s_\star^0$, characterized by Eq.~(\ref{eq:free-equilibrium-state}). We call $s_\star^0$ the \textit{free state}. For each parameter $\theta_k$, the quantity $\frac{\partial E_k}{\partial \theta_k} \left( \theta_k, \{x, s_\star^0\}_k \right)$ is measured locally and stored locally.

\paragraph{Nudged phase.}
The nudging factor $\beta$ is set to a nonzero value (positive or negative), and the system settles to a new equilibrium state $s_\star^\beta$, characterized by Eq.~(\ref{eq:nudged-steady-state}). We call $s_\star^\beta$ the \textit{nudged state}. For each parameter $\theta_k$, the quantity $\frac{\partial E_k}{\partial \theta_k} \left( \theta_k, \{x, s_\star^\beta\}_k \right)$ is measured locally.

\paragraph{Update rule.}
Finally, each parameter $\theta_k$ is updated locally in proportion to its gradient as $\Delta \theta_k = - \eta \widehat{\nabla}_{\theta_k}(\beta)$, where $\eta$ is a learning rate and $\widehat{\nabla}_{\theta_k}(\beta)$ is the gradient estimator of Eq.~(\ref{eq:one-sided-estimator}).

\bigskip

The training scheme described above is natural because the free phase and the nudged phase can be related to the standard training procedure for neural networks (the backpropagation algorithm), in which there is an inference phase (forward pass) followed by a gradient computation phase (backward pass). However, due to the approximation of derivatives by finite differences, the gradient estimator prescribed by the above training scheme is biased. As detailed in Appendix \ref{chapter:appendix}, the mismatch between this gradient estimator ($\widehat{\nabla}_{\theta_k}(\beta)$) and the true gradient ($\frac{\partial \mathcal{L}}{\partial \theta_k}$) is of the order $O(\beta)$. As proposed in \citet{laborieux2020scaling}, this bias can be reduced by means of a symmetric gradient estimator:
\begin{equation}
\label{eq:two-sided-estimator}
\widehat{\nabla}_{\theta_k}^{\rm sym}(\beta) = \frac{1}{2 \beta} \left( \frac{\partial E_k}{\partial \theta_k} \left( \theta_k, \{x, s_\star^\beta\}_k \right) -  \frac{\partial E_k}{\partial \theta_k} \left( \theta_k, \{x, s_\star^{-\beta}\}_k \right) \right).
\end{equation}
To achieve this, we can modify the above training procedure to include two nudged phases: one with positive nudging ($+\beta$) and one with negative nudging ($-\beta$). The update rule for parameter $\theta_k$ is then $\Delta \theta_k = - \eta \widehat{\nabla}_{\theta_k}^{\rm sym}(\beta)$. The mismatch between this symmetric gradient estimator and the true gradient is only of the order $O(\beta^2)$.
We note that higher order methods which use more point values of $\beta$ are also possible, to further reduce the gradient estimator bias (e.g. with $+2\beta$, $+\beta$, $-\beta$ and $-2\beta$).

We call such training procedures \textit{equilibrium propagation} (EqProp), with the intuition that the equilibrium state $s_\star^\beta$ `propagates' across the system as $\beta$ is varied.

\section{Examples of Sum-Separable Energy-Based Models}
\label{sec:ebms-examples}

We give here a few examples of sum-separable energy-based models. As a first example, the Hopfield model is useful to develop intuitions, as it is a well-known and well-studied model in both the machine learning literature and the neuroscience literature ; the case of continuous Hopfield networks will be developed in Chapter \ref{chapter:hopfield}. We then briefly present resistive networks, which will be developed in Chapter \ref{chapter:neuromorphic}. Nonlinear resistive networks are potentially promising for the development of neuromorphic hardware, towards the goal of building fast and energy efficient neural networks. We also briefly present another couple of instances of energy-based physical systems, such as flow networks, and elastic networks. All these systems can be trained with EqProp.

\paragraph{Hopfield networks.}
In the Hopfield model \citep{hopfield1982neural} and its continuous version \citep{cohen1983absolute,hopfield1984neurons}, neurons are interconnected via bi-directional synapses. Each neuron $i$ is characterised by a scalar $s_i$, and each synapse connecting neurons $i$ and $j$ is characterised by a number $W_{ij}$ representing the synaptic strength. The energy function of the model, called \textit{Hopfield energy}, is of the form
\begin{equation}
    \label{eq:hopfield-energy-generic}
    E(\theta, s) = - \sum_{i, j} W_{ij} s_i s_j
\end{equation}
(or a variant of Eq.~\ref{eq:hopfield-energy-generic}). In this expression, the vector of neural states $s = (s_1, s_2, \ldots, s_N)$ represents the state variable of the system, and the vector of synaptic strengths $\theta = \{ W_{ij} \}_{i, j}$ represents the parameter variable (the set of adjustable parameters). At inference, neurons stabilize to a minimum of the energy function, where the condition $\frac{\partial E}{\partial s} = 0$ is met. Furthermore, the Hopfield energy is sum-separable with each factor of the form $E_{ij}(W_{ij}, s_i, s_j) = - W_{ij} s_i s_j$. Since the energy gradients are equal to $\frac{\partial E_{ij}}{\partial W_{ij}} = - s_i s_j$, the Hopfield model can be trained with EqProp using a sort of contrastive Hebbian learning rule (Chapter \ref{chapter:hopfield}).

\paragraph{Resistive networks.}
A linear resistance network is an electrical circuit composed of nodes interconnected by linear resistors. Let $N$ be the number of nodes in the circuit, and denote $V = (V_1, \ldots, V_N)$ the vector of node voltages. Since the power dissipated in a resistor of conductance $g_{ij}$ is $\mathcal{P}_{ij} = g_{ij} \left(V_j-V_i \right)^2$, where $V_i$ and $V_j$ are the terminal voltages of the resistor, the total power dissipated in the circuit is
\begin{equation}
    \mathcal{P}(\theta, V) = \sum_{i, j} g_{ij} \left(V_j-V_i \right)^2,
\end{equation}
where $\theta = \{ g_{ij} \}_{i, j}$ is the set of conductances of the circuit, which plays the role of `adjustable parameters'. Notably, linear resistance networks satisfy the so-called \textit{principle of minimum dissipated power}: if the voltages are imposed at a set of input nodes, then the circuit chooses the voltages at other nodes so as to minimize the total power dissipated ($\mathcal{P}$). This implies in particular that $\frac{\partial \mathcal{P}}{\partial V_i} = 0$ for any floating node voltage $V_i$. Thus, linear resistance networks are energy-based models, with $\mathcal{P}$ playing the role of `energy function'. Furthermore, the function $\mathcal{P}$ has the sum-separability property, with each factor of the form $\mathcal{P}_{ij}(g_{ij}, V_i, V_j) = g_{ij} (V_i - V_j)^2$, and each gradient equal to $\frac{\partial \mathcal{P}_{ij}}{\partial g_{ij}} = (V_i - V_j)^2$. Crucially, as we will see in Chapter \ref{chapter:neuromorphic}, in circuits consisting of arbitrary resistive devices, there exists a generalization of the notion of power function $\mathcal{P}$ called \textit{co-content} \citep{millar1951cxvi}. Such circuits, called \textit{nonlinear resistive networks}, can implement analog neural networks, using memristors (to implement the synaptic weights), diodes (to play the role of nonlinearities), voltage sources (to set the voltages of input nodes) and current sources (to inject loss gradients as currents during training).

\paragraph{Flow networks.}
The EqProp framework may also have implications in other areas of engineering, beyond neuromorphic computing. For example, \citet{stern2020supervised} study the case of \textit{flow networks}, e.g. networks of nodes interconnected by pipes. This setting is analogous to the case of resistive networks described above. In a flow network, each node $i$ is described by its pressure $p_i$, and each pipe connecting node $i$ to node $j$ is characterized by its conductance $k_{ij}$. The total dissipated power in the network, which is minimized, is $\mathcal{P}(\theta, p) = \sum_{i, j} k_{ij} \left(p_j-p_i \right)^2$, where $\theta = \{ k_{ij} \}$ is the set of parameters to be adjusted, and $p = \{ p_{ij} \}$ plays the role of the state variable of the system.

\paragraph{Elastic networks.}
\citet{stern2020supervised} also study the case of \textit{central force spring networks}. In this setting, we have a set of $N$ nodes interconnected by linear springs. Each node $i$ is characterized by its 3D position $s_i$. The elastic energy stored in the spring connecting node $i$ to node $j$ is $E_{ij} = \frac{1}{2} k_{ij} \left( r_{ij} - \ell_{ij} \right)^2$, where $k_{ij}$ is the spring constant, $\ell_j$ is the spring's equilibrium length, and $r_{ij} = \| s_i - s_j \|$ is the Euclidean distance between nodes $i$ and $j$. Thus, the total elastic energy stored in the network, which is minimized, is given by
\begin{equation}
    E(\theta,r) = \sum_{i,j} \frac{1}{2} k_{ij} \left( r_{ij} - \ell_{ij} \right)^2,
\end{equation}
where $\theta = \{ k_{ij}, \ell_{ij} \}$ is the set of adjustable parameters, and $r = \{ r_{ij} \}$ plays the role of state variable. The energy gradients in this case are $\frac{\partial E_{ij}}{\partial k_{ij}} =  \frac{1}{2} \left( r_{ij} - \ell_{ij} \right)^2$ and $\frac{\partial E_{ij}}{\partial \ell_{ij}} =  k_{ij} \left( \ell_{ij} - r_{ij} \right)$.

\section{Fundamental Lemma}
\label{chapter:fundamental-lemma}

In this section, we present the fundamental lemma of the equilibrium propagation framework, from which Theorem \ref{thm:static-eqprop} derives.

\medskip

\begin{lma}[\citet{scellier2017equilibrium}]
\label{lma:main}
Let $F(\theta, \beta, s)$ be a twice differentiable function of the three variables $\theta$, $\beta$ and $s$. For fixed $\theta$ and $\beta$, let $s_\theta^\beta$ be a point that satisfies the stationarity condition
\begin{equation}
\label{eq:stationary}
\frac{\partial F}{\partial s}(\theta, \beta, s_\theta^\beta) = 0,
\end{equation}
and suppose that $\frac{\partial^2 F}{\partial s^2}(\theta, \beta, s_\theta^\beta)$ is invertible. Then, in the neighborhood of this point, we can define a continuously differentiable function $(\theta, \beta) \mapsto s_\theta^\beta$ such that Eq.~\ref{eq:stationary} holds for any $(\theta, \beta)$ in this neighborhood. Furthermore, we have the following identity:
\begin{equation}
\label{eq:fund-equation-main-lemma}
\frac{d}{d\theta} \frac{\partial F}{\partial \beta}(\theta, \beta, s_\theta^\beta) = \frac{d}{d\beta} \frac{\partial F}{\partial \theta}(\theta, \beta, s_\theta^\beta).
\end{equation}
\end{lma}

\begin{proof}[Proof of Lemma \ref{lma:main}]
The first statement follows from the implicit function theorem. It remains to prove Eq.~\ref{eq:fund-equation-main-lemma}. Let us consider $F \left(\theta, \beta, s_\theta^\beta \right)$ as a function of $(\theta, \beta)$ (not only through $F(\theta, \beta, \cdot)$ but also through $s_\theta^\beta$). Using the chain rule of differentiation and the stationary condition of Eq.~(\ref{eq:stationary}), we have
\begin{equation}
  \frac{d}{d\beta} F(\theta, \beta, s_\theta^\beta) = \frac{\partial F}{\partial \beta} \left( \theta, \beta, s_\theta^\beta \right)
  + \underbrace{\frac{\partial F}{\partial s} \left( \theta, \beta, s_\theta^\beta \right)}_{= \; 0} \cdot \frac{\partial s_\theta^\beta}{\partial \beta}.
\end{equation}
Similarly, we have
\begin{equation}
  \frac{d}{d\theta} F(\theta, \beta, s_\theta^\beta) = \frac{\partial F}{\partial \theta} \left( \theta, \beta, s_\theta^\beta \right)
  + \underbrace{\frac{\partial F}{\partial s} \left( \theta, \beta, s_\theta^\beta \right)}_{= \; 0} \cdot \frac{\partial s_\theta^\beta}{\partial \theta}.
\end{equation}
Combining these equations and using the symmetry of second-derivatives, we get:
\begin{equation}
\frac{d}{d\theta} \frac{\partial F}{\partial \beta}(\theta, \beta, s_\theta^\beta) = \frac{d}{d\theta} \frac{d}{d\beta} F(\theta, \beta, s_\theta^\beta) =  \frac{d}{d\beta} \frac{d}{d\theta} F(\theta, \beta, s_\theta^\beta) = \frac{d}{d\beta} \frac{\partial F}{\partial \theta}(\theta, \beta, s_\theta^\beta).
\end{equation}
\end{proof}

\section{Remarks}
\label{sec:remarks}

\paragraph{Stationary points.}
In the static setting presented in this chapter, EqProp applies to any system whose equilibrium states satisfy the stationary condition of Eq~\ref{eq:free-equilibrium-state} -- what we have called an \textit{energy-based model} (EBM). While early works \citep{scellier2017equilibrium,scellier2019equivalence} proposed to apply EqProp to EBMs whose equilibrium states are minima of the energy function (e.g. Hopfield networks), we stress here that the equilibrium states may more generally be any stationary points (saddle points or maxima) of the energy function. We note that the landscape of the energy function can contain in general exponentially many more stationary points than (local) minima. For instance, the recently introduced `modern Hopfield networks' \citep{ramsauer2020hopfield} are EBMs in the sense of Eq~\ref{eq:free-equilibrium-state}.

\paragraph{Infinite dimensions.}
In section \ref{sec:ebms-examples}, we have given examples of EBMs in which the state variable $s$ has finitely many dimensions. We note however that $s$ may also belong to an infinite dimensional space (mathematically, a Banach space). In this case, the expression $\frac{\partial F}{\partial s} \left( \theta, \beta, s_\theta^\beta \right) \cdot \frac{\partial s_\theta^\beta}{\partial \beta}$ in Lemma \ref{lma:main} must be thought of as the differential of the function $F \left( \theta, \beta, \cdot \right)$ at the point $s_\theta^\beta$, applied to the vector $\frac{\partial s_\theta^\beta}{\partial \beta}$.

\paragraph{Variational principles of physics.}
In physics, many systems can be described by a variational equation of the form of Eq.~\ref{eq:free-equilibrium-state} ; we have given examples in Section \ref{sec:ebms-examples}. In fact, the framework presented here transfers directly to time-varying physical systems (Chapter \ref{chapter:future}). In this case, $s$ must be thought of as the \textit{trajectory} of the system, $E$ as a functional called \textit{action functional}, and the stationary condition of Eq.~\ref{eq:free-equilibrium-state} as a \textit{principle of stationary action}.

\paragraph{Singularities.}
We have proved Theorem \ref{thm:static-eqprop} using Lemma \ref{lma:main}. Yet the formula of Lemma \ref{lma:main} assumes that the Hessian $\frac{\partial^2 E}{\partial s^2}(\theta, x, s(\theta, x))$ is invertible. While this assumption is likely to be valid at most iterations of training, it is also likely that, as $\theta$ evolves during training, $\theta$ goes through values where the Hessian of the energy is singular for some input $x$. At such points, it is not clear how the update rule of EqProp behaves. One branch of mathematics that studies these aspects is \textit{Catastrophe Theory}. Although these aspects raise interesting questions, diving into these questions would take us far from the main thrust of this manuscript. In this manuscript, we pretend everything is differentiable.

\paragraph{Beyond supervised learning.}
Although we have focused on supervised learning, the framework presented in this chapter can be adapted to other machine learning paradigms. For example, we note that the formula of Theorem \ref{thm:static-eqprop} can be directly transposed to compute the loss gradients with respect to input variables of the network:
\begin{equation}
    \frac{\partial \mathcal{L}}{\partial x}(\theta, x, y) = \left. \frac{d}{d\beta} \right|_{\beta=0} \frac{\partial E}{\partial x} \left( \theta, x, s_\star^\beta \right).
\end{equation}
This formula may be useful in applications where one wants to do gradient descent in the input space, e.g. image synthesis. This may also be useful in the setting of generative adversarial networks \citep{goodfellow2014generative}, in which we need to compute the loss gradients with respect to inputs of the discriminator network, to further propagate error signals in the generator network. The framework presented in this chapter may also be adapted to model-free reinforcement learning algorithms such as temporal difference (TD) learning (e.g. Q-learning). Finally, the EqProp training procedure has also been used in the meta-learning setting to train the meta-parameters of a model, a method called \textit{contrastive meta-learning} \citep{zucchet2021contrastive}. We briefly present this framework in section \ref{sec:contrastive-meta-learning}.

%% file: hopfield-model.tex
\chapter{Training Continuous Hopfield Networks with Equilibrium Propagation}
\label{chapter:hopfield}

In the previous chapter, we have presented equilibrium propagation (EqProp) as a general learning algorithm for energy-based models. In this chapter, we use EqProp to train a class of energy-based models called gradient systems. In particular we apply EqProp to continuous Hopfield networks, a class of of neural networks that has inspired neuroscience and neuromorphic computing since the 1980s. The present chapter is essentially a compilation of \citet{scellier2017equilibrium,scellier2019equivalence}, and is organized as follows.
\begin{itemize}
    \item In section \ref{sec:gradient-system}, we apply EqProp to a class of continuous-time dynamical systems called \textit{gradient systems}. In a gradient system, the state dynamics descend the gradient of a scalar function (the \textit{energy function}) and stabilise to a minimum of that energy function. Thus, in this setting, equilibrium states correspond to energy minima. We provide an analytical formula for the transient states of the system between the free state and the nudged state of the EqProp training process, and we link these transient states to the recurrent backpropagation algorithm of \citet{almeida1987learning} and \citet{pineda1987generalization}.
    \item In section \ref{sec:hopfield-model}, we apply EqProp to the continuous Hopfield model, an energy-based neural network model described by an energy function called the \textit{Hopfield energy}. The gradient dynamics associated with the Hopfield energy yields the neural dynamics in Hopfield networks: neurons are seen as leaky integrator neurons, with the constraint that synapses are bidirectional and symmetric. In addition, the update rule of EqProp for each synapse is local (more specifically Hebbian).
    \item In section \ref{sec:experiments-hopfield}, we present numerical experiments on deep Hopfield networks trained with EqProp on the MNIST digit classification task.
    \item In section \ref{sec:contrastive-hebbian-learning}, we study the relationship between EqProp and the contrastive Hebbian learning algorithm of \citet{movellan1991contrastive}.
\end{itemize}

\section{Gradient Systems}
\label{sec:gradient-system}

In this section, we present a theoretical result which holds for arbitrary energy functions and cost functions. This section, which deals with the concepts of energy and cost functions in the abstract, is largely independent of the rest of this chapter. The reader who is eager to see how Hopfield networks can be trained with EqProp may skip this section and go straight to the next one.

\subsection{Gradient Systems as Energy-Based Models}

We have seen in Chapter \ref{chapter:eqprop} that EqProp is an algorithm to train systems that possess \textit{equilibrium states}, i.e. states characterized by a variational equation of the form $\frac{\partial E}{\partial s} \left( \theta, x, s_\star \right) = 0$, where $E \left( \theta, x, s \right)$ is a scalar function called \textit{energy function}. Recall that $\theta$ is the set of adjustable parameters of the system, and $x$ is an input. We have called such systems \textit{energy-based models}. The class of energy-based models that we study here is that of systems whose dynamics spontaneously minimizes the energy function $E$ by following its gradient. In such a system, called \textit{gradient system}, the state follows the dynamics
\begin{equation}
\label{eq:continuous-time-free-phase}
\frac{d s_t}{dt} = - \frac{\partial E}{\partial s} \left( \theta, x, s_t \right).
\end{equation}
Here $s_t$ denotes the state of the system at time $t$. The energy of the system decreases until $\frac{d s_t}{dt} = 0$, and the equilibrium state $s_\star$ reached after convergence of the dynamics is an energy minimum (either local or global). The function $E$ is also sometimes called a \textit{Lyapunov function} for the dynamics of $s_t$.

In this setting, equilibrium states are \textit{stable}: if the state is slightly perturbed around equilibrium, the dynamics will tend to bring the system back to equilibrium. For this reason, such equilibrium states are also called `attractors' or `retrieval states', because the system's dynamics can `retrieve' them if they are only partially known. Thus, gradient systems can recover incomplete data, by storing `memories' in their point attractors.

In this manuscript, we are more specifically interested in the supervised learning problem, where the loss to optimize is of the form
\begin{equation}
\mathcal{L} = C(s_\star, y).
\end{equation}
After training is complete, the model can be used to `retrieve' a label $y$ associated to a given input $x$.

\subsection{Training Gradient Systems with Equilibrium Propagation}

In a gradient system, EqProp takes the following form. In the first phase, or free phase, the state of the system follows the gradient of the energy (Eq.~\ref{eq:continuous-time-free-phase}). At the end of the first phase the system is at equilibrium ($s_\star$). In the second phase, or nudged phase, starting from the equilibrium state $s_\star$, a term $- \beta \; \frac{\partial C}{\partial s}$ (where $\beta > 0$ is a hyperparameter called \textit{nudging factor}) is added to the dynamics of the state and acts as an external force nudging the system dynamics towards decreasing the cost $C$. Denoting $s_t^\beta$ the state of the system at time $t$ in the second phase (which depends on the value of the nudging factor $\beta$), the dynamics is defined as\footnote{As discussed in the case of Eq.~\ref{eq:nudged-steady-state}, we can also define $\frac{d s_t^\beta}{dt}  = -\frac{\partial E}{\partial s} \left( \theta, x, s_t^\beta \right) - \beta \; \frac{\partial C}{\partial s} \left( s_\star, y \right)$ without changing the conclusions of the theoretical results.}
\begin{equation}
\label{eq:continuous-time-nudged-phase}
    s_0^\beta = s_\star \qquad \text{and} \qquad \forall t \geq 0, \quad \frac{d s_t^\beta}{dt}  = -\frac{\partial E}{\partial s} \left( \theta, x, s_t^\beta \right) - \beta \; \frac{\partial C}{\partial s} \left( s_t^\beta, y \right).
\end{equation}
The system eventually settles to a new equilibrium state $s_\star^\beta$. Recall from Theorem \ref{thm:static-eqprop} that the gradient of the loss $\mathcal{L}$ can be estimated based on the two equilibrium states $s_\star$ and $s_\star^\beta$. Specifically, in the limit $\beta \to 0$,
\begin{equation}
\label{eqprop-gradient-system}
\lim_{\beta \to 0} \frac{1}{\beta} \left( \frac{\partial E}{\partial \theta} \left( x, s_\star^\beta, \theta \right) -  \frac{\partial E}{\partial \theta} \left( \theta, x, s_\star \right) \right) = \frac{\partial \mathcal{L}}{\partial \theta}.
\end{equation}
Furthermore, if the energy function has the sum-separability property (as defined by Eq.~\ref{eq:sum-separability}), then the learning rule for each parameter $\theta_k$ is local:
\begin{equation}
\label{eqprop-gradient-system-local}
\lim_{\beta \to 0} \frac{1}{\beta} \left( \frac{\partial E_k}{\partial \theta_k} \left( \theta_k, \{ x, s_\star^\beta \}_k \right) -  \frac{\partial E_k}{\partial \theta_k} \left( \theta_k, \{ x, s_\star \}_k \right) \right) = \frac{\partial \mathcal{L}}{\partial \theta_k}.
\end{equation}

\subsection{Transient Dynamics}

Note that the learning rule of Eqs.~\ref{eqprop-gradient-system}-\ref{eqprop-gradient-system-local} only depends on the equilibrium states $s_\star$ and $s_\star^\beta$, not on the specific trajectory that the system follows to reach them. Indeed, as we have seen in the previous chapter, EqProp applies to any energy based model, not just gradient systems. But under the assumption of a gradient dynamics, we can say more about the transient states, when the system gradually moves from the free state ($s_\star$) towards the nudged state ($s_\star^\beta$): we show that the transient states ($s_t^\beta$ for $t \geq 0$) perform gradient computation with respect to a function called the \textit{projected cost function}.

Recall that in the free phase, the system follows the dynamics of Eq.~\ref{eq:continuous-time-free-phase}. In particular, the state $s_t$ at time $t \geq 0$ depends not just on $\theta$ and $x$, but also on the initial state $s_0$ at time $t=0$. Let us define the \textit{projected cost function}
\begin{equation}
	\label{eq:projected-cost-function}
	L_t(\theta, s_0) = C \left( s_t \right),
\end{equation}
where we omit $x$ and $y$ for brevity of notations. $L_t(\theta, s_0)$ is the cost of the state projected a duration $t$ in the future, when the system starts from $s_0$ and follows the dynamics of the free phase. Note that $L_t$ depends on $\theta$ and $s_0$ (as well as $x$) implicitly through $s_t$. For fixed $s_0$, the process $\left( L_t(\theta, s_0) \right)_{t \geq 0}$ represents the successive cost values taken by the state of the system along the free dynamics when it starts from the initial state $s_0$. In particular, for $t=0$, the projected cost is simply the cost of the initial state, i.e. $L_0(\theta, s_0) = C \left( s_0 \right)$. As $t \to \infty$, we have $s_t \to s_\star$ and therefore $L_t(\theta, s_0) \to C(s_\star) = \mathcal{L}(\theta)$, i.e. the projected cost converges to the loss at equilibrium.

The following result shows that the transient states of EqProp ($s_t^\beta$) can be expressed in terms of the projected cost function ($L_t$), when $s_0 = s_\star$. 

\medskip

\begin{thm}[\citet{scellier2019equivalence}]
\label{thm:truncated-eqprop}
The following identities hold for any $t \geq 0$:
\begin{gather}
	\label{eq:truncated-eqprop-parameter}
	\lim_{\beta \to 0} \frac{1}{\beta} \left(
	\frac{\partial E}{\partial \theta} \left( \theta, s_t^\beta \right) - \frac{\partial E}{\partial \theta} \left( \theta, s_\star \right)
	\right) = \frac{\partial L_t}{\partial \theta} \left( \theta, s_\star \right), \\
	\label{eq:truncated-eqprop-state}
	\lim_{\beta \to 0} \frac{1}{\beta} \; \frac{d s_t^\beta}{d t} = - \frac{\partial L_t}{\partial s} \left( \theta, s_\star \right).
\end{gather}
Furthermore, if the energy function has the sum-separability property (as defined by Eq.~\ref{eq:sum-separability}), then
\begin{equation}
	\label{eq:truncated-eqprop-local}
	\lim_{\beta \to 0} \frac{1}{\beta} \left(
	\frac{\partial E_k}{\partial \theta_k} \left( \theta_k, \{ s_t^\beta \}_k \right) - \frac{\partial E_k}{\partial \theta_k} \left( \theta_k, \{ s_\star \}_k \right)
	\right) = \frac{\partial L_t}{\partial \theta_k} \left( \theta, s_\star \right).
\end{equation}
\end{thm}

\begin{proof}
    Eq.~\ref{eq:truncated-eqprop-parameter}-\ref{eq:truncated-eqprop-state} follow directly from Lemma \ref{lma:rec-backprop} (next subsection). Eq.~\ref{eq:truncated-eqprop-local} follows from Eq.~\ref{eq:truncated-eqprop-parameter} and the definition of sum-separability.
\end{proof}

The left-hand-side of Eq.~\ref{eq:truncated-eqprop-parameter} represents the gradient provided by EqProp if we substitute $s_t^\beta$ to $s_\star^\beta$ in the gradient formula (Eq.~\ref{eqprop-gradient-system}). This corresponds to a \textit{truncated} version of EqProp, where the second phase (nudged phase) is halted before convergence to the nudged equilibrium state. Eq.~\ref{eq:truncated-eqprop-parameter} provides an analytical formula for this truncated gradient in terms of the projected cost function, when $s_0 = s_\star$

The left-hand side of Eq.~\ref{eq:truncated-eqprop-state} is the temporal derivative of $s_t^\beta$ rescaled by a factor $\frac{1}{\beta}$. In essence, Eq.~\ref{eq:truncated-eqprop-state} shows that, in the second phase of EqProp (nudged phase), the temporal derivative of the state \textit{codes} for gradient information (namely the gradients of the projected cost function, when $s_0 = s_\star$).

\begin{figure*}[ht!]
\begin{center}
    \includegraphics[width=\linewidth]{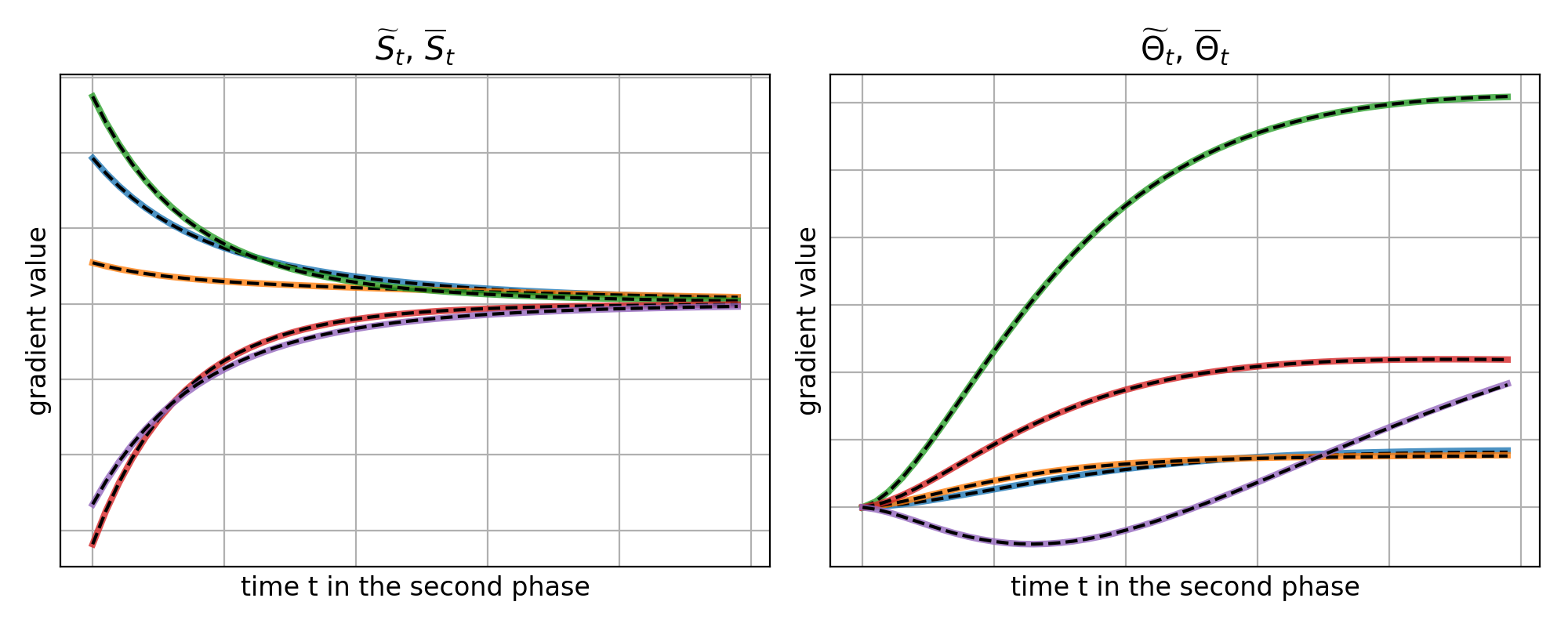}
    \caption[Illustration of Theorem \ref{thm:truncated-eqprop} on a toy example]{\textbf{Illustration of Theorem \ref{thm:truncated-eqprop} on a toy example.} Dashed lines (in black) represent five randomly chosen coordinates of $\widetilde{S}_t$ (left) and five randomly chosen coordinates of $\widetilde{\Theta}_t$ (right). Solid colored lines represent the corresponding coordinates in $\overline{S}_t$ (left) and in $\overline{\Theta}_t$ (right). The processes $\overline{S}_t$, $\widetilde{S}_t$, $\overline{\Theta}_t$ and $\widetilde{\Theta}_t$ are defined by Eqs.~\ref{eq:process-bar}-\ref{eq:process-tilde}. The figure was produced by modifying the code\footnotemark \; of \citet{ernoult2019updates}.}
    \label{fig:thm-gradient-system}
\end{center}
\end{figure*}

\subsection{Recurrent Backpropagation}
\label{sec:proof-temporal-derivatives}

In this section, we prove Theorem \ref{thm:truncated-eqprop}. In doing so, we also establish a link between EqProp and the recurrent backpropagation algorithm of \citet{almeida1987learning} and \citet{pineda1987generalization}, which we briefly present below.

First, let us introduce the temporal processes $(\overline{S}_t, \overline{\Theta}_t)$ and $(\widetilde{S}_t, \widetilde{\Theta}_t)$ defined by
\begin{equation}
    \label{eq:process-bar}
	\forall t \geq 0, \qquad \overline{S}_t = \frac{\partial L_t}{\partial s} \left( \theta, s_\star \right), \qquad \overline{\Theta}_t = \frac{\partial L_t}{\partial \theta} \left( \theta, s_\star \right),
\end{equation}
and
\begin{equation}
    \label{eq:process-tilde}
	\forall t \geq 0, \qquad \widetilde{S}_t = -\lim_{\beta \to 0} \frac{1}{\beta} \; \frac{d s_t^\beta}{d t}, \qquad
	\widetilde{\Theta}_t = \lim_{\beta \to 0} \frac{1}{\beta} \left( \frac{\partial E}{\partial \theta} \left( \theta, s_t^\beta \right) - \frac{\partial E}{\partial \theta} \left( \theta, s_\star \right) \right).
\end{equation}
The processes $\overline{S}_t$ and $\widetilde{S}_t$ take values in the state space (space of the state variable $s$). The processes $\overline{\Theta}_t$ and $\widetilde{\Theta}_t$ take values in the parameter space (space of the parameter variable $\theta$). Using these notations, Theorem \ref{thm:truncated-eqprop} states that $\overline{S}_t = \widetilde{S}_t$ and $\overline{\Theta}_t = \widetilde{\Theta}_t$ for every $t \geq 0$. These identities are a direct consequence of the following lemma.

\footnotetext{https://github.com/ernoult/updatesEPgradientsBPTT}

\medskip

\begin{restatable}{lma}{lmarbp}
	\label{lma:rec-backprop}
	Both the processes $(\overline{S}_t, \overline{\Theta}_t)$ and $(\widetilde{S}_t, \widetilde{\Theta}_t)$ are solutions of the same (linear) differential equation:
	\begin{align}
		\label{eq:Cauchy-1}
		S_0 & = \frac{\partial C}{\partial s} \left( s_\star \right),      \\
		\label{eq:Cauchy-2}
		\Theta_0 & = 0, \\
		\label{eq:Cauchy-3}
		\frac{d}{dt} S_t & = - \frac{\partial^2 E}{\partial s^2} \left( \theta, s_\star \right) \cdot S_t, \\
		\label{eq:Cauchy-4}
		\frac{d}{dt} \Theta_t & = - \frac{\partial^2 E}{\partial \theta \partial s} \left( \theta, s_\star \right) \cdot S_t.
	\end{align}
	By uniqueness of the solution, the processes $(\overline{S}_t, \overline{\Theta}_t)$ and $(\widetilde{S}_t, \widetilde{\Theta}_t)$ are equal.
\end{restatable}

We refer to \citet{scellier2019equivalence} for a proof of this result. Since the differential equation of Lemma \ref{lma:rec-backprop} is linear with constant coefficients, we can express $S_t$ and $\Theta_t$ using closed-form formulas. Specifically $S_t = \exp \left( - t \frac{\partial^2 E}{\partial s^2} \left( \theta, s_\star \right) \right) \cdot \frac{\partial C}{\partial s} \left( s_\star \right)$ and $\Theta_t = \frac{\partial^2 E}{\partial \theta \partial s} \left( \theta, s_\star \right) \cdot \left( \frac{\partial^2 E}{\partial s^2} \left( \theta, s_\star \right) \right)^{-1} \cdot \left[ \textrm{Id} - \exp \left( - t \frac{\partial^2 E}{\partial s^2} \left( \theta, s_\star \right) \right) \right] \cdot \frac{\partial C}{\partial s} \left( s_\star \right)$.

Lemma \ref{lma:rec-backprop} also suggests an alternative procedure to compute the parameter gradients of the loss $\mathcal{L}$ numerically. This procedure, known as \textit{Recurrent Backpropagation} (RBP), was introduced independently by \citet{almeida1987learning} and \citet{pineda1987generalization}. Specifically, RBP consists of the following two phases. The first phase is the same as the free phase of EqProp: $s_t$ follows the free dynamics (Eq.~\ref{eq:continuous-time-free-phase}) and relaxes to the equilibrium state $s_\star$. The state $s_\star$ is necessary for evaluating $\frac{\partial^2 E}{\partial s^2} \left( \theta, s_\star \right)$ and $\frac{\partial^2 E}{\partial \theta \partial s} \left( \theta, s_\star \right)$, which the second phase requires. In the second phase, $S_t$ and $\Theta_t$ are computed iteratively for increasing values of $t$ using Eq.~\ref{eq:Cauchy-1}-\ref{eq:Cauchy-4}. Finally, $\Theta_t$ provides the desired loss gradient in the limit $t \to \infty$. To see this, we first note that Lemma \ref{lma:rec-backprop} tells us that the vector $\Theta_t$ computed by this procedure is equal to $\widetilde{\Theta}_t$ for any $t \geq 0$. Then by definition, $\widetilde{\Theta}_t = \lim_{\beta \to 0} \frac{1}{\beta} \left( \frac{\partial E}{\partial \theta} \left( \theta, s_t^\beta \right) - \frac{\partial E}{\partial \theta} \left( \theta, s_\star \right) \right)$. It follows that, as $t \to \infty$, we have $\widetilde{\Theta}_t \to \lim_{\beta \to 0} \frac{1}{\beta} \left( \frac{\partial E}{\partial \theta} \left( \theta, s_\star^\beta \right) - \frac{\partial E}{\partial \theta} \left( \theta, s_\star \right) \right) = \frac{\partial \mathcal{L}}{\partial \theta}$.

An important benefit of EqProp over RBP it that EqProp requires only one kind of dynamics for both phases of training. RBP requires a special computational circuit in the second phase for computing the gradients.

The original RBP algorithm was described
for a general state-to-state dynamics. Here, we have presented RBP
in the particular case of gradient dynamics. We refer to \citet{lecun1988theoretical} for a more general derivation of RBP based on the adjoint method.

\section{Continuous Hopfield Networks}
\label{sec:hopfield-model}

In the previous section we have presented a theoretical result which is generic and involves the energy function $E$ and cost function $C$ in their abstract form. In this section, we study EqProp in the context of a neural network model called the continuous Hopfield model \citep{hopfield1984neurons}.

\subsection{Hopfield Energy}

A Hopfield network is a neural network with the following characteristics. The state of a neuron $i$ is described by a scalar $s_i$, loosely representing its membrane voltage. The state of a synapse connecting neuron $i$ to neuron $j$ is described by a real number $W_{ij}$ representing its efficacy (or `strength'). The notation $\sigma(s_i)$ is further used to denote the firing rate of neuron $i$. The function $\sigma$ is called \textit{activation function} ; it takes a real number as input, and returns a real number as output. Using the formalism of the previous section, the state of the system is the vector $s= \left( s_1, s_2, \ldots, s_N \right)$ where $N$ is the number of neurons in the network, and the set of parameters to be adjusted is $\theta = \{ W_{ij} \}_{ij}$. As we will see shortly, one biologically unrealistic requirement of the Hopfield model is that synapses are assumed to be bidirectional and symmetric: the synapse connecting $i$ to $j$ shares the same weight value as the synapse connecting $j$ to $i$, i.e. $W_{ij} = W_{ji}$.

\paragraph{Hopfield Energy.}
\citet{hopfield1984neurons} introduced the following energy function\footnote{The energy function of Eq.~\ref{eq:hopfield-energy} is in fact the one proposed by \citet{bengio2017stdp}. The energy function introduced by Hopfield is slightly different, but this technical detail is not essential for our purpose.}:
\begin{equation}
    \label{eq:hopfield-energy}
    E(\theta, s) = \frac{1}{2} \sum_i s_i^2 - \sum_{i < j} W_{ij} \sigma(s_i) \sigma(s_j),
\end{equation}
which we will call the \textit{Hopfield energy}. We calculate
\begin{equation}
    \label{eq:leaky-integrator-hopfield}
    -\frac{\partial E}{\partial s_i} = \sigma'(s_i) \left( \sum_{j \neq i} W_{ij} \sigma(s_j) \right) - s_i.
\end{equation}
Thus, the gradient dynamics for neuron $s_i$ with respect to the Hopfield energy is given by the formula $\frac{d s_i}{dt} = \sigma'(s_i) \left( \sum_{j \neq i} W_{ij} \sigma(s_j) \right) - s_i$. This dynamics is reminiscent of the leaky integrator neuron model, a simplified neuron model commonly used in neuroscience. The main difference with the standard leaky-integrator neuron model is the fact that synaptic weights are constrained to be bidirectional and symmetric, a biologically unrealistic constraint often referred to as the \textit{weight transport problem}. Another difference is the presence of the term $\sigma'(s_i)$ which modulates the total input to neuron $i$.

\paragraph{Squared Error.}
In the supervised setting that we study here, a set of neurons are input neurons, denoted $x$, and are always clamped to their input values. Among the `free' neurons ($s$), a subset of them are \textit{output neurons} (denoted $o$), meaning that they represent the network's output. The network's prediction is the state of output neurons at equilibrium (denoted $o_\star$). We call all other neurons the \textit{hidden neurons} and denote them $h$. Thus, the state of the network is $s = \left( h, o \right)$. The cost function considered here is the squared error
\begin{equation}
    \label{eq:sq-cost-function}
    C(s, y) = \frac{1}{2} \left\lVert o-y \right \rVert^2,
\end{equation}
which measures the discrepancy between the state of output neurons ($o$) and their target values ($y$).

\paragraph{Total Energy.}
One of the novelties of EqProp with respect to prior learning algorithms for energy-based models is the \textit{total energy function} $F$, which takes the form $F = E + \beta \; C$, where $\beta$ is a real-valued scalar (the \textit{nudging factor}). The function $C$ not only represents the cost to minimize, but also contributes to the total energy of the system by acting like an external potential for the output neurons ($o$). Thus, the total energy $F$ is the sum of two potential energies: an `internal potential' ($E$) that models the interactions within the network, and an `external potential' ($\beta \; C$) that models how the targets influence the output neurons. The resulting gradient dynamics $\frac{d s_t}{dt} = - \frac{\partial E}{\partial s} - \beta \frac{\partial C}{\partial s}$ consists of two 'forces' which act on the temporal derivative of $s_t$. The 'internal force' (induced by $E$) is that of a leaky integrator neuron (Eq.~\ref{eq:leaky-integrator-hopfield}). The 'external force' (induced by $\beta \; C$) on $s=(h,o)$ takes the form:
\begin{equation}
	\label{eq:external-force}
	- \beta \frac{\partial C}{\partial h} = 0 \qquad \text{and} \qquad
	- \beta \frac{\partial C}{\partial o} = \beta ( y-o ).
\end{equation}
This external force acts on output neurons only: it can pull them (if $\beta \geq 0$) towards their target values ($y$), or repel them (if $\beta \leq 0$). The nudging factor $\beta$ controls the strength of this interaction between output neurons and targets. In particular, when $\beta=0$, the output neurons are not sensitive to the targets.

\begin{figure*}[ht!]
\begin{center}
    \includegraphics[width=0.3\linewidth]{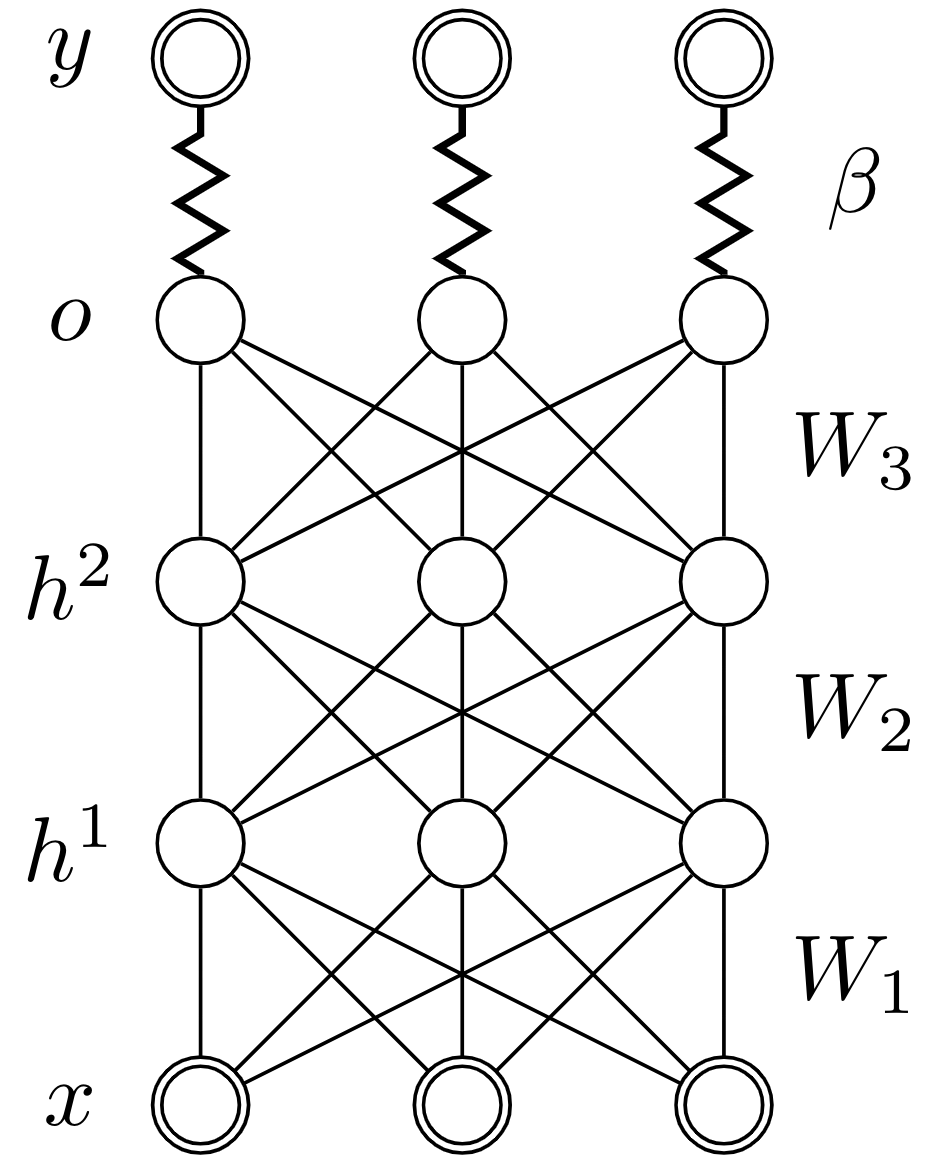}
    \caption[A deep Hopfield network (DHN)]{\textbf{A deep Hopfield network (DHN).} Input $x$ is clamped. Neurons $s$ include ``hidden layers'' $h_1$ and $h_2$, and ``output layer'' $o$ (the layer where the prediction is read). Target $y$ has the same dimension as output $o$. Connections between neurons are bidirectional and have symmetric weights. Such a network can be trained with EqProp. In the nudged phase (second phase of training), the nudging factor $\beta$ scales the ``external force'' $\beta (y-o)$ that attracts output neurons ($o$) towards their target values ($y$).}
    \label{fig:network_undirected}
\end{center}
\end{figure*}

\subsection{Training Continuous Hopfield Networks with Equilibrium Propagation}
\label{sec:eqprop-hopfield}

Consider a deep Hopfield network (DHN) of the kind depicted in Figure \ref{fig:network_undirected}. For each training example $(x, y)$ in the dataset, EqProp training proceeds as follows.

\paragraph{Free Phase.}
At inference, inputs $x$ are clamped, and both the hidden neurons ($h^1$ and $h^2$) and output neurons ($o$) evolve freely, following the gradient of the energy. The hidden and output neurons subsequently stabilize to an energy minimum, called \textit{free state} and denoted $s_\star = \left( h_\star^1, h_\star^2, o_\star \right)$. The state of output neurons at equilibrium ($o_\star$) plays the role of prediction for the model.

\paragraph{Nudged Phase.}
After relaxation to the free state $s_\star$, the target $y$ is observed, and the nudging factor $\beta$ takes on a positive value, gradually driving the state of output neurons ($o$) towards $y$. Since the external force only acts on the output neurons, the hidden layers ($h^1$ and $h^2$) are initially at equilibrium at the beginning of the nudged phase. The perturbation introduced at output neurons gradually propagates backwards along the layers of the network, until the system settles to a new equilibrium state ($s_\star^\beta$).

\bigskip

\begin{prop}[\citet{scellier2017equilibrium}]
\label{prop:eqprop-hopfield}
Denote $s_i^0$ and $s_i^\beta$ the free state and nudged state of neuron $i$, respectively. Then, we have the following formula to estimate the gradient of the loss $\mathcal{L} = \frac{1}{2} \| o_\star - y \|^2$ :
\begin{equation}
\lim_{\beta \to 0} \frac{1}{\beta} \left( \sigma \left( s_i^\beta \right) \sigma \left( s_j^\beta \right) - \sigma \left( s_i^0 \right) \sigma \left( s_j^0 \right) \right) = -\frac{\partial \mathcal{L}}{\partial W_{ij}}.
\end{equation}
\end{prop}

\begin{proof}
This is a direct consequence of the main Theorem \ref{thm:static-eqprop}, applied to the Hopfield energy function (Eq.~\ref{eq:hopfield-energy}) and the squared error cost function (Eq.~\ref{eq:sq-cost-function}). Notice that the Hopfield energy has the sum-separability property (as defined by Eq.~\ref{eq:sum-separability}), with each factor of the form $E_{ij}(W_{ij}, s_i, s_j) = - W_{ij} \sigma(s_i) \sigma(s_j)$.
\end{proof}

Proposition \ref{prop:eqprop-hopfield} suggests for each synapse $W_{ij}$ the update rule
\begin{equation}
\label{eq:global-update-hopfield}
\Delta W_{ij} = \frac{\eta}{\beta} \left( \sigma \left( s_i^\beta \right) \sigma \left( s_j^\beta \right) - \sigma \left( s_i^0 \right) \sigma \left( s_j^0 \right) \right),
\end{equation}
where $\eta$ is a learning rate. This learning rule is a form of contrastive Hebbian learning (CHL), with a Hebbian term at one equilibrium state, and an anti-Hebbian term at the other equilibrium state. We will discuss in Section \ref{sec:contrastive-hebbian-learning} the relationship between EqProp and the CHL algorithm of \citet{movellan1991contrastive}.

\subsection{`Backpropagation' of Error Signals}

It is interesting to note that EqProp is similar in spirit to the backpropagation algorithm \citep{rumelhart1988learning}. The free phase of EqProp, which corresponds to inference, plays the role of the forward pass in a feedforward net. The nudged phase of EqProp is similar to the backward pass of backpropagation, in that the target output is revealed and it involves the propagation of loss gradient signals. This analogy is even more apparent in a layered network like the one depicted in Fig.~\ref{fig:network_undirected}: in the nudged phase of EqProp, error signals (back-)propagate across the layers of the network, from output neurons to input neurons. Theorem \ref{thm:truncated-eqprop} gives a more quantitative description of how gradient computation is performed in the nudged phase, with the temporal derivatives of neural activity carrying gradient signals. Thus, like backprop, the learning process in EqProp is driven by an error signal; but unlike backprop, neural computation in EqProp corresponds to both inference and error back-propagation. The idea that error signals in neural networks can be encoded in the temporal derivatives of neural activity was also explored by \citet{hinton1988learning,movellan1991contrastive,o1996biologically}, and has been recently formulated as a hypothesis for neuroscience \citep{lillicrap2020backpropagation}.

Because error signals are propagated in the network via the neural dynamics, synaptic plasticity can be driven directly by the dynamics of the neurons. Indeed, the global update of EqProp (Eq.~\ref{eq:global-update-hopfield}) is equal to the temporal integration of infinitesimal updates
\begin{equation}
    \label{eq:real-time-update}
    dW_{ij} = \frac{\eta}{\beta} d \left( \sigma(s_i) \sigma(s_j) \right) 
\end{equation}
over the nudged phase, when the neurons gradually move from their free state ($s_\star$) to their nudged state ($s_\star^\beta$). This suggests an alternative method to implement the global weight update: in the first phase, when the neurons relax to the free state, no synaptic update occurs ($\Delta W_{ij} = 0$) ; in the second phase, the real-time update of Eq.~\ref{eq:real-time-update} is performed when the neurons evolve from their free state to their nudged state. This idea is formalized and tested numerically in \citet{ernoult2020equilibrium}.

From a biological perspective, perhaps the most unrealistic assumption in this model of credit assignment is the requirement of symmetric weights. 
This
constraint can be relaxed at the cost of computing a biased gradient \citep{scellier2018generalization,ernoult2020equilibrium,laborieux2020scaling,tristany2020equilibrium}.

\section{Numerical Experiments on MNIST}
\label{sec:experiments-hopfield}

In this section, we present the experimental results of \citet{scellier2017equilibrium}. In these simulations, we train deep Hopfield networks of the kind depicted in Fig.~\ref{fig:network_undirected}. Our networks have no skip-layer connections and no lateral connections\footnote{We stress that the models trainable by EqProp are not limited to the chain-like architecture of Fig.~\ref{fig:network_undirected}. Other works have studied the effect of adding skip-layer connections \citep{gammell2020layer} and introducing sparsity \citep{tristany2020equilibrium}.}. We recall that these Hopfield networks, unlike feedforward networks, are recurrently connected, with bidirectional and symmetric connections (i.e. the synapse from neuron $i$ to neuron $j$ shares the same weight value as the synapse from neuron $j$ to neuron $i$).

We train these Hopfield networks on the MNIST digits classification task \citep{lecun1998gradient}. The MNIST dataset (the `modified' version of the National Institute of Standards and Technology dataset) of handwritten digits is composed of 60,000 training examples and 10,000 test examples. Each example $x$ in the dataset is a $28 \times 28$ gray-scaled image and comes with a label $y \in \left\{ 0, 1, \ldots, 9 \right\}$ indicating the digit that the image represents. Given an input $x$, the network's prediction $\widehat{y}$ is the index of the output neuron (among the $10$ output neurons) whose activity at equilibrium is maximal, that is
\begin{equation}
	\widehat{y} = \underset{i \in \{ 0, 1, \ldots, 9 \}}{\arg \max} \; o_{\star, i}.
\end{equation}
The network is optimized by stochastic gradient descent (SGD). The process to perform one training iteration on a sample of the training set (i.e. to compute the corresponding gradient and to take one step of SGD) is the one described in section \ref{sec:eqprop-hopfield}. For efficiency of the experiments, we use minibatches of $20$ training examples.

\subsection{Implementation Details}

The hyperparameters chosen for each model are shown in Table \ref{table:hopfield-results}. The code is available\footnote{https://github.com/bscellier/Towards-a-Biologically-Plausible-Backprop}.

\paragraph{Architecture.}
We train deep Hopfield networks with $1$, $2$ and $3$ hidden layers. The input layer consists of $28 \times 28 = 784$ neurons. The hidden layers consist of $500$ hidden neurons each. The output layer consists of $10$ output neurons.

\paragraph{Weight initialization.}
The weights of the network are initialized\footnote{Little is known about how to initialise the weights of recurrent neural networks with static input. More exploration is needed to find appropriate initialisation schemes for such networks.} according to the Glorot-Bengio initialization scheme \citep{glorot2010understanding}, i.e. each weight matrix is initialized by drawing i.i.d. samples uniformly at random in the range $[L, U]$, where $L=-\frac{\sqrt{6}}{\sqrt{n_i + n_{i+1}}}$ and $U=\frac{\sqrt{6}}{\sqrt{n_i + n_{i+1}}}$, with $n_i$ the fan-in and $n_{i+1}$ the fan-out of the weight matrix.

\paragraph{Implementation of the neural dynamics.}
Recall that, for a fixed input-target pair $(x, y)$, the total energy is $F(\theta, \beta, s) = E(\theta, x, s) + \beta \; C(s, y)$. We implement the gradient dynamics $\frac{ds_t}{dt} = -\frac{\partial F}{\partial s} \left( \theta, \beta, s_t \right)$ using the Euler scheme, meaning that we discretize time into short time lapses of duration $\epsilon$ and iteratively update the state of the network (hidden and output neurons) according to
\begin{equation}
	\label{eq:gradient-descent}
	s_{t+1} = s_t - \epsilon \frac{\partial F}{\partial s} \left( \theta, \beta, s_t \right).
\end{equation}
This process can be thought of as one step of gradient descent (in the state space) on the total energy $F$, with learning rate $\epsilon$. In practice we find that it is necessary to restrict the space for each state variable (i.e. each neuron) to a bounded interval ; we choose the interval $[0, 1]$. This amounts to use the modified version of the Euler scheme:
\begin{equation}
	\label{eq:clipped-gradient-descent}
	s_{t+1} = \min \left( \max \left( 0, s_t - \epsilon \frac{\partial F}{\partial s} \left( \theta, \beta, s_t \right) \right), 1 \right).
\end{equation}
We choose $\epsilon = 0.5$ in the simulations. The number of iterations in the free phase is denoted $T$. The number of iterations in the nudged phase is denoted $K$.

\paragraph{Number of iterations in the free phase ($T$).}
We find experimentally that for the network to be successfully trained, it is necessary that the equilibrium state be reached with very high precision in the free phase (otherwise the gradient estimate of EqProp is unreliable). As a consequence, we require a large number of iterations (denoted $T$) to reach this equilibrium state. Moreover we find that $T$ grows fast as the number of layers increases (see Table \ref{table:hopfield-results}). Nevertheless, we will see in Chapter \ref{chapter:discrete-time} that we can experimentally cut down the number of iterations by a factor five by rewriting the free phase dynamics differently. Importantly, we stress that the large number of time steps required in the free phase is only a concern for computer simulations ; we will see in Chapter \ref{chapter:neuromorphic} that inference can potentially be extremely fast if performed appropriately on analog hardware (by using the physics of the circuit, rather than numerical optimization on conventional computers).

\paragraph{Number of iterations in the nudged phase ($K$).}
During the second phase of training, we find experimentally that full relaxation to the nudged equilibrium state is not necessary. This observation is also partly justified by Theorem \ref{thm:truncated-eqprop}, which gives an explicit formula for the `truncated gradient' provided by EqProp when the nudged phase is halted before convergence. As a heuristic, we choose $K$ (the number of iterations in the nudged phase) proportional to the number of layers, so that the `error signals' are able to propagate from output neurons back to input neurons.

\paragraph{Nudging factor ($\beta$).}
In spite of its intrinsic bias (Lemma \ref{lma:gradient-estimators}), we find that the one-sided gradient estimator performs well on MNIST (as also observed by \citet{ernoult2019updates}). We choose $\beta=1$ in the experiments. Although it is not crucial, we find that the test accuracy is slightly improved by choosing the sign of $\beta$ at random in the nudged phase of each training iteration (with probability $p(\beta=1)=1/2$ and $p(\beta=-1)=1/2$). Randomizing $\beta$ indeed has the effect of cancelling on average the $O(\beta)$-error term of the one-sided gradient estimator.

While this is not necessary on MNIST, we will see in Chapter \ref{chapter:discrete-time} that on a more complex task such as CIFAR-10, unbiasing the gradient estimator is necessary, and that the symmetric nudging estimator (Eq.~\ref{eq:two-sided-estimator}) further helps stabilize training and improve test accuracy.

\paragraph{Learning rates.}
We find experimentally that we need different learning rates for the weight matrices of different layers. We choose these learning rates heuristically as follows. Denote by $h^0, h^1, \cdots, h^N$ the layers of the network (where $h^0 = x$ and $h^N = o$) and by $W_k$ the weight matrix between the layers $h^{k-1}$ and $h^k$. We choose the learning rate $\alpha_k$ for $W_k$ proportionally to $\frac{\left\lVert W_k \right \rVert}{\mathbb{E} \left[ \left\lVert \nabla_{W_k} \right \rVert \right]}$, where $\mathbb{E} \left[ \left\lVert \nabla_{W_k} \right \rVert \right]$ represents the norm of the EqProp gradient for layer $W_k$, averaged over training examples.

\subsection{Experimental Results}

Table \ref{table:hopfield-results} (top) presents the experimental results of \citet{scellier2017equilibrium}. These experiments aim at demonstrating that the EqProp training scheme is able to perfectly (over)fit the training dataset, i.e. to get the error rate on the training set down to $0.00\%$. To achieve this, we use the following trick to reach the equilibrium state of the first phase more easily: at each epoch of training, for each example in the training set, we store the corresponding equilibrium state (i.e. the state of the hidden and output neurons at the end of the free phase), and we use this configuration as a starting point for the next free phase relaxation on that example. This method, which is similar to the PCD (Persistent Contrastive Divergence) algorithm for sampling from the equilibrium distribution of the Boltzmann machine \citep{tieleman2008training}, enables to speed up the first phase and reach the equilibrium state with higher precision.

However, this technique hurts generalization performance. Table~\ref{table:hopfield-results} (bottom) shows the experimental results of \citet{ernoult2019updates}, which do not use this technique: during training, for each training example in the dataset, the state of the network is initialized to zero at the beginning of each free phase relaxation. The resulting test error rate is lower, though the number of iterations required in the free phase to converge to equilibrium is larger.

\begin{table}[ht!]
\centering
$\begin{array}{|c|c|cc|ccccc|cccc|}
\hline
	\hbox{Model} & {\rm cached} & {\rm Test \; er.} & {\rm Train \; er.} & T   & K & \epsilon & \beta & {\rm Epochs} & \alpha_1 & \alpha_2 & \alpha_3 & \alpha_4 \\
\hline
	\hbox{DHN-1h} & Y & \sim 2.5 \; \%                   & 0.00 \; \%                     & 20  & 4 & 0.5 & 1.0 & 25           & 0.1      & 0.05     &          &          \\
	\hbox{DHN-2h} & Y & \sim 2.3 \; \%                    & 0.00 \; \%                     & 100 & 6 & 0.5 & 1.0 & 60           & 0.4      & 0.1      & 0.01     &          \\
	\hbox{DHN-3h} & Y & \sim 2.7 \; \%                    & 0.00 \; \%                    & 500 & 8 & 0.5 & 1.0 & 150          & 0.128    & 0.032    & 0.008    & 0.002    \\
\hline
	\hbox{DHN-1h} & N & 2.06 \; \%                   & 0.13 \; \%                     & 100  & 12 & 0.2 & 0.5 & 30           & 0.1      & 0.05     &          &          \\
	\hbox{DHN-2h} & N & 2.01 \; \%                    & 0.11 \; \%                     & 500 & 40 & 0.2 & 0.8 & 50           & 0.4      & 0.1      & 0.01     &          \\
\hline
\end{array}$
\caption[Experimental results of \citet{scellier2017equilibrium} on deep Hopfield networks trained on MNIST.]{
"DHN-$\#$h" stands for Deep Hopfield Network with $\#$ hidden layers. `cached' refers to whether or not the equilibrium states are cached and reused as a starting point at the next free phase relaxation. $T$ is the number of iterations in the free phase. $K$ is the number of iterations in the nudged phase. $\epsilon$ is the step size for the dynamics of the state variable $s$. $\beta$ is the value of the nudging factor in the nudged phase. $\alpha_k$ is the learning rate for updating the parameters in layer $k$. \textbf{Top.} Experimental results of \citet{scellier2017equilibrium} with the caching trick. Test error rates and train error rates are reported on single trials. \textbf{Bottom.} Experimental results of \citet{ernoult2019updates} without the caching trick. Test error rates and train error rates are averaged over five trials.
}
\label{table:hopfield-results}
\end{table}

Since these early experiments, thanks to new insights, new ideas and more perseverance, new results have been obtained which improve in terms of simulation speed, test accuracy, and complexity of the task solved. We present these more recent experimental results in Chapter \ref{chapter:discrete-time}. In addition, we stress that the real potential of EqProp is more likely to shine on neuromorphic substrates (Chapter \ref{chapter:neuromorphic}), rather than on digital computers.

\section{Contrastive Hebbian Learning (CHL)}
\label{sec:contrastive-hebbian-learning}

In the setting of continuous Hopfield networks studied in this Chapter, EqProp is similar to the generalized recirculation algorithm (GeneRec) \citep{o1996biologically}. The main novelty of EqProp with respect to GeneRec is the formalism based on the concepts of nudging factor ($\beta$) and total energy function ($F$), which enables to formulate a general framework for training energy-based models (Chapter \ref{chapter:eqprop}) and Lagrangian-based models (Chapter \ref{chapter:future}), applicable not just to the continuous Hopfield model, but also many more network models, including nonlinear resistive networks (Chapter \ref{chapter:neuromorphic}) and convolutional networks (Chapter \ref{chapter:discrete-time}).

EqProp is also similar in spirit to the contrastive Hebbian learning algorithm (CHL), which we present in this section. The CHL algorithm was originally introduced in the case of the Boltzmann machine \citep{ackley1985learning} and then extended to the case of the continuous Hopfield network \citep{movellan1991contrastive,baldi1991contrastive}.
We note that Boltzmann machines may be trained with EqProp, via the stochastic version presented in Section \ref{sec:stochastic-setting}.

\subsection{Contrastive Hebbian Learning in the Continuous Hopfield Model}

Like EqProp, the CHL algorithm proceeds in two phases and uses a free phase. But unlike EqProp, it uses a \textit{clamped phase} as a second phase for training, instead of a \textit{nudged phase}. Bringing this modification to the EqProp training procedure described in section \ref{sec:eqprop-hopfield}, we arrive at the following algorithm, proposed by \citet{movellan1991contrastive}.

\paragraph{Free phase.}
As in EqProp, the first phase is a free phase (also called `negative phase'): inputs $x$ are clamped, and both the hidden and output neurons evolve freely, following the gradient of the energy function. The hidden and output neurons stabilize to an energy minimum called free state and denoted $\left( h_\star^-, o_\star^- \right)$. We write $s^- = \left( x, h_\star^-, o_\star^- \right)$. At the free state, every synapse undergoes an anti-Hebbian update. That is, for any synapse $W_{ij}$ (connecting neuron $i$ to neuron $j$), we perform the weight update $\Delta W_{ij} = - \eta \; \sigma \left( s_i^- \right) \sigma \left( s_j^- \right)$.

\paragraph{Clamped phase.}
The second phase is a `clamped phase' (also called `positive phase'): not only inputs are clamped, but also outputs are now clamped to their target value $y$. The hidden neurons evolve freely and stabilize to another energy minimum $h_\star^+$. We write $s^+ = \left( x, h_\star^+, y \right)$ and call this configuration the \textit{clamped state}. At the clamped state, every synapse undergoes a Hebbian update. That is, for any synapse $W_{ij}$ (connecting neuron $i$ to neuron $j$), we perform the weight update $\Delta W_{ij} = + \eta \; \sigma \left( s_i^+ \right) \sigma \left( s_j^+ \right)$.

\paragraph{Global update.}
Putting the weight updates of the free phase and clamped phase together, we get the global update of the CHL algorithm:
\begin{equation}
    \label{eq:global-chl-update}
    \Delta W_{ij} = \eta \left( \sigma \left( s_i^+ \right) \sigma \left( s_j^+ \right) - \sigma \left( s_i^- \right) \sigma \left( s_j^- \right) \right).
\end{equation}

\subsection{An Intuition Behind Contrastive Hebbian Learning}

Both CHL and EqProp have the desirable property that learning stops when the network correctly predicts the target. Specifically, in CHL, when the equilibrium state of the free phase (the free state) matches the equilibrium state of the clamped phase (the clamped state), the two terms of the weight update (Eq.~\ref{eq:global-chl-update}) cancel out, thus yielding an effective weight update of zero. In other words, if the network already provides the correct output, then no learning occurs.

It is instructive to verify that EqProp preserves this property, even in its general formulation (Chapter \ref{chapter:eqprop}). Suppose that the equilibrium state ($s_\star^0$) corresponding to an input $x$ provides the correct answer ($y$), i.e. suppose that $s_\star^0$ is a minimum of the function $s \mapsto C(s, y)$. This implies that $\frac{\partial C}{\partial s}(s_\star^0, y)=0$. Using the fact that $\frac{\partial E}{\partial s}(\theta, x, s_\star^0) = 0$ by definition of $s_\star^0$, we get $\frac{\partial E}{\partial s}(\theta, x, s_\star^0) + \beta \; \frac{\partial C}{\partial s}(s_\star^0, y)=0$ for any value of $\beta$. This implies that $s_\star^\beta = s_\star^0$ for any $\beta$, by definition of $s_\star^\beta$. As a consequence, the two terms in the learning rule of EqProp cancel out. We note that this property remains true in the case of the symmetric difference estimator (Eq.~\ref{eq:two-sided-estimator}).

\subsection{A Loss Function for Contrastive Hebbian Learning}

The global learning rule of the CHL algorithm rewrites in terms of the energy function (the Hopfield energy) as
\begin{equation}
    \label{eq:chl-learning-rule}
    \Delta \theta = \eta \left( - \frac{\partial E}{\partial \theta} \left( \theta, x, h_\star^+, y \right) + \frac{\partial E}{\partial \theta} \left( \theta, x, h_\star^-, o_\star^- \right) \right).
\end{equation}
Here $\left( x, h_\star^+, y \right)$ is the clamped state, and $\left( x, h_\star^-, o_\star^- \right)$ is the free state. In this form, the CHL update rule stipulates to decrease the energy value of the clamped state and to increase the energy value of the free state. Since low-energy configurations correspond to preferred states of the model under the gradient dynamics, the CHL update rule thus increases the likelihood that the model produces the correct output ($y$), and decreases the likelihood that it generates again the same output ($o_\star$).

\medskip

\begin{prop}[\citet{movellan1991contrastive}]
\label{prop:loss-chl}
The CHL update rule (Eq.~\ref{eq:chl-learning-rule}) is equal to
\begin{equation}
\Delta \theta = - \eta \frac{\partial \mathcal{L}^{\rm CHL}}{\partial \theta}(\theta, x, y),
\end{equation}
where $\mathcal{L}^{\rm CHL}$ is the loss defined by
\begin{equation}
    \mathcal{L}^{\rm CHL}(\theta, x, y) = E \left( \theta, x, h_\star^+, y \right) - E \left( \theta, x, h_\star^-, o_\star^- \right).
\end{equation}
\end{prop}

The loss $\mathcal{L}^{\rm CHL}$ has the problem that the two phases of the CHL algorithm may stabilize in different modes of the energy function. \citet{movellan1991contrastive} points out that when this happens, the weight update is inconsistent and learning usually deteriorates. Similarly, \citet{baldi1991contrastive} note abrupt discontinuities due to basin hopping phenomena.

EqProp solves this problem by optimizing the loss $\mathcal{L} = C(s_\star, y)$, whose gradient can be estimated using nudged states ($s_\star^\beta$) that are infinitesimal continuous deformations of the free state ($s_\star$), and are thus in the same `mode' of the energy landscape.

%% file: resistive-networks.tex
\chapter{Training Nonlinear Resistive Networks with Equilibrium Propagation}
\label{chapter:neuromorphic}

In the previous chapter, we have discussed the bio-realism of EqProp in the setting of Hopfield networks. Learning in this context is achieved using solely leaky integrator neurons (in both phases of training) and a local (Hebbian) weight update. These bio-realistic features are of interest not only for neuroscience, but also for neuromorphic computing, towards the goal of building fully analog neural networks supporting on-chip learning. Recently, several works have proposed analog implementations of EqProp in the context of Hopfield networks \citep{zoppo2020equilibrium,foroushani2020analog,ji2020towards} and spiking variants \citep{o2019training,martin2020eqspike}.

Here we investigate a different approach to implement EqProp on neuromorphic chips. We emphasize that EqProp is not limited to the Hopfield model and the gradient systems of Chapter \ref{chapter:hopfield}, but more broadly applies to any system whose equilibrium state $s_\star$ is a solution of a variational equation $\frac{\partial E}{\partial s}(s_\star) = 0$, where $E(s)$ is a scalar function -- what we have called an \textit{energy-based model} (EBM) in Chapter \ref{chapter:eqprop}. Importantly, many physical systems can be described by variational principles, as a reformulation of the physical laws characterizing their state. This suggests a path to build highly efficient energy-based models grounded in physics, with EqProp as a learning algorithm for training.

In this chapter, we exploit the fact that a broad class of analog neural networks called \textit{nonlinear resistive networks} can be described by such a variational principle. Nonlinear resistive network are electrical circuits consisting of nodes interconnected by (linear or nonlinear) resistive elements. These circuits can serve as analog neural networks, in which the weights to be adjusted are implemented by the conductances of programmable resistive devices such as memristors \citep{chua1971memristor}, and the nonlinear transfer functions (or `activation functions') are implemented by nonlinear components such as diodes. The `energy function' in these nonlinear resistive networks is a quantity called the \textit{co-content} \citep{millar1951cxvi} or \textit{total pseudo-power} \citep{johnson2010nonlinear} of the circuit, and its existence can be derived directly from Kirchhoff's laws. Moreover, this energy function has the sum-separability property: the total pseudo-power of the circuit is the sum of the pseudo-powers of its individual elements. As a consequence, we can train these analog networks with EqProp, and the update rule for each conductance, which follows the gradient of the loss, is local. Specifically, we show mathematically that the gradient with respect to a conductance can be estimated using solely the voltage drop across the corresponding resistor. This theoretical result provides a principled method to train end-to-end analog neural networks by stochastic gradient descent, thus suggesting a path towards the development of ultra-fast, compact and low-power learning-capable neural networks.

The present chapter, which is essentially a rewriting of \citet{kendall2020training}, is articulated as follows.
\begin{itemize}
\item In section \ref{sec:analog-neural-network}, we briefly present a class of analog neural networks called \textit{nonlinear resistive networks}, as well as the concept of \textit{programmable resistors} that play the role of synapses.
\item In section \ref{sec:nonlinear-resistive-network-ebm}, we show that these nonlinear resistive networks are energy-based models: at inference, the configuration of node voltages chosen by the circuit corresponds to the minimum of a mathematical function (the \textit{energy function}) called the \textit{co-content} (or \textit{total pseudo-power}) of the circuit, as a consequence of Kirchhoff's laws (Lemma \ref{lma:power}). This suggests an implementation of energy-based neural networks grounded in electrical circuit theory, which also bridges the conceptual gap between energy functions (at a mathematical level\footnote{In an energy-based model, the \textit{energy function} is a mathematical abstraction of the model, not a physical energy.}), and physical energies\footnote{Specifically the power dissipated in resistive devices.} (at a hardware level).
\item In section \ref{sec:nonlinear-resistive-network-eqprop}, we show how these nonlinear resistive networks can be trained with EqProp, and we derive the formula for updating the conductances (the synaptic weights) in proportion to their loss gradients, using solely the voltage drops across the corresponding resistive devices (Theorem \ref{lma:gradients}).
\item In section \ref{sec:analog-network-model}, as a proof of concept of what is possible with this neuromorphic hardware methodology, we propose an analog network architecture inspired by the deep Hopfield network, which alternates linear and nonlinear processing stages (Fig.~\ref{fig:network}).
\item In section \ref{sec:numerical-simulations}, we present numerical simulations on the MNIST dataset, using a SPICE-based framework to simulate the circuit's dynamics.
\end{itemize}

By explicitly decoupling the training procedure (EqProp in Section \ref{sec:nonlinear-resistive-network-eqprop}) from the specific neural network architecture presented (Section \ref{sec:analog-network-model}), we stress that this optimization method is applicable to any resistive network architecture, not just the one of Section \ref{sec:analog-network-model}. This modular approach thus offers the possibility to explore the design space of analog network architectures trainable with EqProp, in essentially the same way as deep learning researchers explore the design space of differentiable neural networks trainable with backpropagation.

\section{Nonlinear Resistive Networks as Analog Neural Networks}
\label{sec:analog-neural-network}

Nonlinear resistive networks are electrical circuits consisting of arbitrary two-terminal resistive elements -- see \citet[Chapter~3]{muthuswamy2018introduction} for an introduction. We can use such circuits to build neural networks. In the supervised learning scenario, we use a subset of the nodes of the circuit as input nodes, and another subset of the nodes as output nodes. We use voltage sources to impose the voltages at input nodes: after the circuit has settled to steady state, the voltages of output nodes indicate the `prediction'. The circuit thus implements an input-to-output mapping function, with the node voltages representing the state of the network. This mapping function can be nonlinear if we include nonlinear resistive elements such as diodes in the circuit, and the conductance values of resistors can be thought of as parameterizing this mapping function.

A \textit{programmable resistor} is a resistor whose conductance can be changed (or `programmed'), and thus can play the role of a `weight' to be adjusted. Programmable resistors can thus implement the synapses of a neural network. In the last decade, many technologies have emerged, and have been proposed and studied as programmable resistors. We refer to \citet{burr2017neuromorphic} and \citet{xia2019memristive} for reviews on existing technologies, their working mechanisms, and how they are used for neuromorphic computing. For convenience, in most of this chapter we will think of programmable resistors as ideally tunable, which is a convenient concept to formalize mathematically the goal of learning in nonlinear resistive networks. However, this is an ideal and unrealistic assumption: in practice, far from being ideally tunable, these programmable resistive devices currently present important challenges for the coming decade of research to solve. We refer to \citet{chang2017mitigating} for an analysis of these challenges to be overcome. In this manuscript, we will not discuss how the programming of a conductance can be done and implemented in hardware.

We note that nonlinear resistive networks have been studied as neural network models since the 1980s \citep{hutchinson1988computing,harris1989resistive}.

\section{Nonlinear Resistive Networks are Energy-Based Models}
\label{sec:nonlinear-resistive-network-ebm}

In this section we show that, in a nonlinear resistive network, the steady state of the circuit imposed by Kirchhoff's laws is a stationary point of a function called the \textit{co-content}, or \textit{total pseudo-power} (Lemma \ref{lma:power}). Thus, nonlinear resistive networks are energy-based models whose energy function is the total pseudo-power. Furthermore, the total pseudo-power has the sum-separability property, being by definition the sum of the pseudo-powers of its individual components.

We first present in section \ref{sec:linear-resistance-network} the case of linear resistance networks. Although this model is functionally not very useful (as a neural network model), studying it is helpful to gain understanding of the working mechanisms of analog neural networks: it helps understand the limits of linear resistances and the need to introduce nonlinear elements (section \ref{sec:resistive-elements}). In section \ref{sec:nonlinear-resistive}, we derive the general result for nonlinear resistive networks.

\subsection{Linear Resistance Networks}
\label{sec:linear-resistance-network}

A \textit{linear resistance network} is an electrical circuit whose nodes are linked pairwise by \textit{linear resistors}, i.e. resistors that satisfy Ohm's law. We recall that, in a linear resistor, Ohm's law states that $I_{ij} = g_{ij} (V_i-V_j)$, where $I_{ij}$ is the current through the resistor, $g_{ij}$ is its conductance ($g_{ij} = \frac{1}{R_{ij}}$ where $R_{ij}$ is the resistance), and $V_i$ and $V_j$ are its terminal voltages.

Consider the following question: we impose the voltages at a set of input nodes, and we want to know what are the voltages at other nodes of the circuit. We can answer this question by writing Ohm's law in every branch, Kirchhoff's current law at every node, and by solving the set of equations obtained for all node voltages and all branch currents. But there is a more elegant way to characterize the steady state of the circuit. Kirchhoff's current law gives $\sum_j I_{ij} = 0$ for every node $i$. Combined with Ohm's law, we get $\sum_j g_{ij} (V_i-V_j) = 0$. Now note that the left-hand side of this expression is equal to $\frac{1}{2} \frac{\partial \mathcal{P}}{\partial V_i}$, where $\mathcal{P}(V_1, V_2, \ldots, V_N)$ is the functional defined by
\begin{equation}
    \mathcal{P}(V_1, V_2, \ldots, V_N) = \sum_{i<j} g_{ij} \left(V_j-V_i \right)^2.
    \label{eq:power-linear-resistance-network}
\end{equation}
This means that, among all \textit{conceivable} configurations of node voltages, the configuration that is physically realized is a stationary point of the functional $\mathcal{P}(V_1, V_2, \ldots, V_N)$. Therefore, linear resistance networks are energy-based models, with the configuration of node voltages $V = (V_1, V_2, \ldots, V_N)$ playing the role of state variable, and the functional $\mathcal{P}(V_1, V_2, \ldots, V_N)$ playing the role of energy function.

The functional $\mathcal{P}(V_1, V_2, \ldots, V_N)$ is called the \textit{power functional}, because it represents the total power dissipated in the circuit, with $\frac{1}{2} g_{ij} \left(V_j-V_i \right)^2$ being the power dissipated in the resistor connecting node $i$ to node $j$. Since $\mathcal{P}$ is convex, the steady state of the circuit is not just a stationary point of $\mathcal{P}$, but also the global minimum. This well-known result of circuit theory is called the \textit{principle of minimum dissipated power}: if we impose the voltages at a set of input nodes, the circuit will choose the voltages at other nodes so as to minimize the total power dissipated in the resistors (Fig.~\ref{fig:minimum-power}).

However, linear resistance networks are not very useful as neural network models since they cannot implement nonlinear operations. Rewriting Kirchhoff's current law at node $i$, we get $V_i = \frac{\sum_j g_{ij} V_j}{\sum_j g_{ij}}$. This operation resembles the usual multiply-accumulate operation of artificial neurons in conventional deep learning, but with the notable difference that there is no nonlinear activation function. Another difference is the presence of the factor $G_i = \sum_j g_{ij}$ at the denominator, which replaces the usual weighted sum by a weighted mean: each floating node voltage $V_i$ is a weighted mean of its neighbors.

From this analysis, it appears that nonlinear elements such as diodes are necessary to perform nonlinear operations. In the rest of this section, we generalize the result of this subsection to the setting of nonlinear resistive networks.

\begin{figure*}[ht!]
\begin{center}
\includegraphics[width=0.5\textwidth]{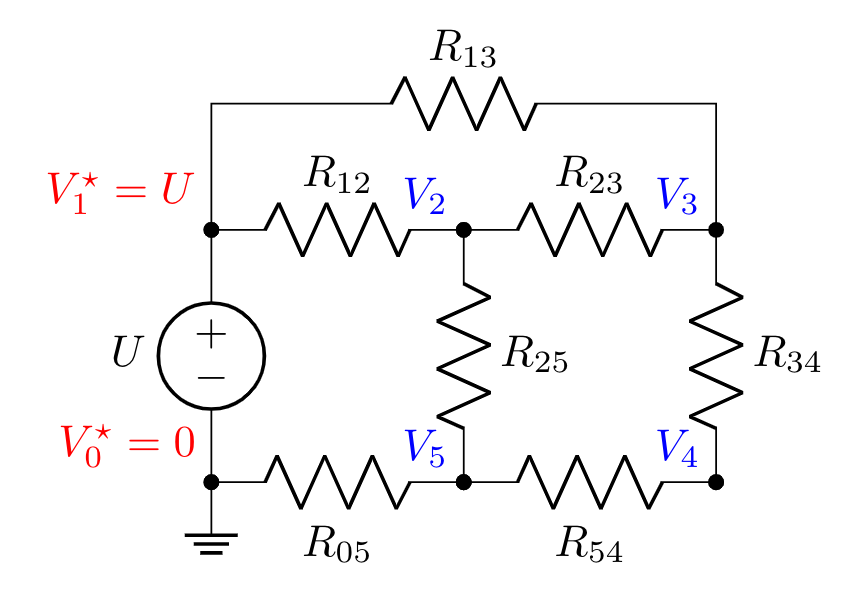}
\end{center}
\caption[Illustration of the principle of minimum dissipated power in a linear resistance network]{
\textbf{Principle of minimum dissipated power.} In a linear resistance network, if we impose the voltages at a set of input nodes ($V_0 = 0$ and $V_1 = U$ here), the voltages at other nodes ($V_2$, $V_3$, $V_4$ and $V_5$ here) is such that the total power dissipated in the resistors is minimized. A generalization of this result to nonlinear resistive networks exists (Lemma \ref{lma:power}).
}
\label{fig:minimum-power}
\end{figure*}

\subsection{Two-Terminal Resistive Elements}
\label{sec:resistive-elements}

In this subsection, we follow the method of \citet{johnson2010nonlinear} to generalize the notion of `power dissipated in a linear resistor' to arbitrary two-terminal resistive elements.

\paragraph{Current-voltage characteristic.}
Consider a two-terminal resistive element with terminals $i$ and $j$, characterised by a well-defined and continuous current-voltage characteristic $\gamma_{ij}$. The function $\gamma_{ij}$ takes as input the voltage drop $\Delta V_{ij} = V_i - V_j$ across the component and returns the current $I_{ij} = \gamma_{ij} \left( \Delta V_{ij} \right)$ moving from node $i$ to node $j$ in response to $\Delta V_{ij}$. Since the current flowing from $i$ to $j$ is the negative of the current flowing from $j$ to $i$, we have by definition:
\begin{equation}
    \label{eq:antisymmetry}
    \forall i, j, \qquad \gamma_{ij} \left( \Delta V_{ij} \right) = - \gamma_{ji} \left( \Delta V_{ji} \right)
\end{equation}
where $\Delta V_{ji} = - \Delta V_{ij}$.

For example, the current-voltage characteristic of a linear resistor of conductance $g_{ij}$ linking node $i$ to node $j$ is, by Ohm's law, $I_{ij} = g_{ij} \Delta V_{ij}$. By definition of $\gamma_{ij}$, this implies that
\begin{equation}
    \label{eq:IVresistor}
    \gamma_{ij} \left( \Delta V_{ij} \right) = g_{ij} \Delta V_{ij}.
\end{equation}

\paragraph{Pseudo-power.}
For each two-terminal element with current-voltage characteristic $I_{ij} = \gamma_{ij}(\Delta V_{ij})$, we define $p_{ij}(\Delta V_{ij})$ as the primitive function of $\gamma_{ij}(\Delta V_{ij})$ that vanishes at $0$, i.e.
\begin{equation}
    \label{eq:pseudo-power}
    p_{ij}(\Delta V_{ij}) = \int_0^{\Delta V_{ij}} \gamma_{ij}(v) dv.
\end{equation}
The quantity $p_{ij} \left( \Delta V_{ij} \right)$ has the physical dimensions of power, being a product of a voltage and a current. We call $p_{ij} \left( \Delta V_{ij} \right)$ the \textit{pseudo-power} along the branch from $i$ to $j$, following the terminology of \citet{johnson2010nonlinear}. Note that as a consequence of Eq.~\ref{eq:antisymmetry} we have
\begin{equation}
    \label{eq:symmetry}
    \forall i, j, \qquad p_{ij}(\Delta V_{ij}) = p_{ji}(\Delta V_{ji}),
\end{equation}
i.e. the pseudo-power from $i$ to $j$ is equal to the pseudo-power from $j$ to $i$. We call this property the \textit{pseudo-power symmetry}.

For example, in the case of a linear resistor of conductance $g_{ij}$ linking node $i$ to node $j$, the pseudo-power corresponding to the current-voltage characteristic of Eq.~\ref{eq:IVresistor} is:
\begin{equation}
\label{eq:pseudo-resistor}
p_{ij}(\Delta V_{ij}) = \frac{1}{2} g_{ij} \Delta V_{ij}^2.
\end{equation}
In this case, the pseudo-power is half the physical power dissipated in the resistor.

\subsection{Nonlinear Resistive Networks}
\label{sec:nonlinear-resistive}

A \textit{nonlinear resistive network} is a circuit consisting of interconnected two-terminal resistive elements. We number the nodes of the circuit $i=1, 2, \ldots, N$.

\paragraph{Configuration.}
We call a vector of voltage values $V = \left( V_1, V_2, \ldots, V_N \right)$ a \textit{configuration}. Importantly, a configuration can be any vector of voltage values, even those that are not compatible with Kirchhoff's current law (KCL).

\paragraph{Total pseudo-power (also called co-content).}
Recall the definition of the pseudo-power of a two-terminal element (Eq.~\ref{eq:pseudo-power}). We define the \textit{total pseudo-power} of a configuration $V = \left( V_1, V_2, \ldots, V_N \right)$ as the sum of pseudo-powers along all branches:
\begin{equation}
    \label{eq:total-pseudo-power}
    \mathcal{P}(V_1, \cdots, V_N) = \sum_{i<j} p_{ij}(V_i - V_j).
\end{equation}
We note that the pseudo-power symmetry (Eq.~\ref{eq:symmetry}) guarantees that this definition does not depend on node ordering. In the case of a linear resistance network, the total pseudo-power of the circuit is half the power functional of Eq.~\ref{eq:power-linear-resistance-network}.

We stress that $\mathcal{P}$ is a mathematical function defined on any configuration $V_1, V_2, \ldots, V_N$, even those that are not compatible with KCL.

\paragraph{Steady state.}
We denote $V_1^\star$, $V_2^\star$, $\ldots$, $V_N^\star$ the configuration of node voltages imposed by Kirchhoff's current law (KCL), and we call $V^\star = \left( V_1^\star, V_2^\star, \ldots, V_N^\star \right)$ the \textit{steady state} of the circuit. Specifically, for every (internal or output) floating node $i$, KCL implies $\sum_{j=1}^N I_{ij} = 0$, which rewrites
\begin{equation}
    \label{eq:KCL}
    \sum_{j=1}^N \gamma_{ij} \left( V_i^\star-V_j^\star \right) = 0.
\end{equation}

The following result, known since \citet{millar1951cxvi}, shows that the circuit is an energy-based model, whose energy function is the total pseudo-power.

\medskip

\begin{lma}
\label{lma:power}
The steady state of the circuit, denoted $\left( V_1^\star, V_2^\star, \ldots, V_N^\star \right)$, is a stationary point\footnote{With further assumptions on the current-voltage characteristics $\gamma_{ij}$, \citet{christianson2007dirichlet}, as well as \citet{johnson2010nonlinear}, show that the function $\mathcal{P}$ is convex, so that the steady state is the global minimum of $\mathcal{P}$. However, in the context of EqProp, all one needs is the first order condition, i.e. the fact that the steady state is a stationary point of $\mathcal{P}$, not necessarily a minimum.}
of the total pseudo-power: for every floating node $i$, we have
\begin{equation}
    \frac{\partial \mathcal{P}}{\partial V_i} \left( V_1^\star, V_2^\star, \ldots, V_N^\star \right) = 0.
\end{equation}
\end{lma}

\begin{proof}[Proof of Lemma \ref{lma:power}]
We use the definition of the total pseudo-power (Eq.~\ref{eq:total-pseudo-power}), the pseudo-power symmetry (Eq.~\ref{eq:symmetry}), the definition of the pseudo-power (Eq.~\ref{eq:pseudo-power}) and the fact that the steady state of the circuit satisfies Kirchhoff's current law (Eq.~\ref{eq:KCL}). For every floating node $i$ we have:
\begin{equation}
\frac{\partial \mathcal{P}}{\partial V_i} \left( V_1^\star, V_2^\star, \ldots, V_N^\star \right) = \sum_j \frac{\partial p_{ij}}{\partial V_i}(V_i^\star-V_j^\star) = \sum_j \gamma_{ij}(V_i^\star-V_j^\star) = 0.
\end{equation}
\end{proof}

Equipped with this result, we can now derive a procedure to train nonlinear resistive networks with EqProp.

\section{Training Nonlinear Resistive Networks with Equilibrium Propagation}
\label{sec:nonlinear-resistive-network-eqprop}

\subsection{Supervised Learning Setting}

In the supervised learning setting, a subset of the nodes of the circuit are \textit{input nodes}, at which input voltages (denoted $X$) are sourced. All other nodes -- the \textit{internal nodes} and \textit{output nodes} -- are left floating: after the voltages of input nodes have been set, the voltages of internal and output nodes settle to their steady state. The output nodes, denoted $\widehat{Y}$, represent the readout of the system, i.e. the model prediction. The architecture and the components of the circuit determine the $X \mapsto \widehat{Y}$ mapping function. Specifically, the conductances of the programmable resistors, denoted $\theta$, parameterize this mapping function. That is, $\widehat{Y}$ can be written as a function of $X$ and $\theta$ in the form $\widehat{Y}(\theta, X)$. Training such a circuit consists in adjusting the values of the conductances ($\theta$) so that the voltages of output nodes ($\widehat{Y}$) approach the target voltages ($Y$). Formally, we cast the goal of training as an optimization problem in which the loss to be optimized (corresponding to an input-target pair $(X, Y)$) is of the form:
\begin{equation}
    \label{eq:loss}
    \mathcal{L}(\theta, X, Y) = C \left( \widehat{Y}(\theta, X), Y \right).
\end{equation}

We have seen that nonlinear resistive networks are energy-based models (Lemma \ref{lma:power}) and that the energy function (the total pseudo-power) is sum-separable, by definition (Eq.!\ref{eq:total-pseudo-power}). This enables us to use EqProp in such analog neural networks to compute the gradient of the loss. Theorem \ref{lma:gradients} below provides a formula for computing the loss gradient with respect to a conductance using solely the voltage drop across the corresponding resistor.

\subsection{Training Procedure}
\label{sec:resistive-networks-algo}

Given an input $X$ and associated target $Y$, EqProp proceeds in the following two phases.

\paragraph{Free phase.}
At inference, input voltages are sourced at input nodes ($X$), while all other nodes of the circuit (the internal nodes and output nodes) are left floating. All internal and output node voltages are stored\footnote{On practical neuromorphic hardware, this can be achieved using a capacitor or sample-and-hold amplifier (SHA) circuit, for instance. We note that we only need one SHA per node (neuron), not per synapse. We will not discuss these aspects of implementation here.
}. In particular, the voltages of output nodes ($\widehat{Y}$) corresponding to prediction are compared with the target ($Y$) to compute the loss $\mathcal {L} = C ( \widehat{Y}, Y)$.

\paragraph{Nudged phase.}
For each output node $\widehat{Y}_k$, a current $I_k = - \beta \frac{\partial C}{\partial \widehat{Y}_k}$ is sourced at $\widehat{Y}_k$, where $\beta$ is a positive or negative scaling factor (the \textit{nudging factor}). All internal node voltages and output node voltages are measured anew.

\medskip

\begin{thm}[\citet{kendall2020training}]
\label{lma:gradients}
Consider a two-terminal component whose terminals are $i$ and $j$. Denote $\Delta V_{ij}^0$ the voltage drop across this two-terminal component in the free phase (when no current is sourced at output nodes), and $\Delta V_{ij}^\beta$ the voltage drop in the nudged phase (when a current $I_k = - \beta \frac{\partial C}{\partial \widehat{Y}_k}$ is sourced at each output node $\widehat{Y}_k$). Let $w_{ij}$ denote an adjustable parameter of this component, and $p_{ij}$ its pseudo-power (which depends on $w_{ij}$). Then, the gradient of the loss ${\mathcal L}  = C \left( \widehat{Y}, Y \right)$ with respect to $w_{ij}$ can be estimated as
\begin{equation}
\label{eq:device-gradient}
\frac{\partial {\mathcal L}}{\partial w_{ij}} =
\lim_{\beta \to 0} \frac{1}{\beta} \left( \frac{\partial p_{ij} \left( \Delta V^\beta_{ij} \right)}{\partial w_{ij}} -  \frac{\partial p_{ij} \left( \Delta V^0_{ij} \right)}{\partial w_{ij}} \right).
\end{equation}
In particular, if the component is a linear resistor of conductance $g_{ij}$, then the loss gradient with respect to $g_{ij}$ can be estimated as
\begin{equation}
\label{eq:conductance-gradient}
\frac{\partial {\mathcal L}}{\partial g_{ij}} = \lim_{\beta \to 0} \frac{1}{2 \beta} \left( \left( \Delta V^\beta_{ij} \right)^2 -  \left( \Delta V^0_{ij} \right)^2 \right).
\end{equation}
\end{thm}

\begin{proof}
For simplicity, we have stated Theorem \ref{lma:gradients} in the case where the cost function $C(\widehat{Y}, Y)$ depends only on output node voltages ($\widehat{Y}$). But this result can be directly generalized to the case of a cost function $C(V, Y)$ that depends on any node voltages ($V$), not just output node voltages. In this case, in the nudged phase of EqProp, currents $I_k = - \beta \frac{\partial C}{\partial V_k}$ must be sourced at every node $V_k$ (not just at output nodes).

Let $\theta$ denote the vector of adjustable parameters (e.g. the conductances), $X$ the voltages of input nodes, and $V$ the voltages of floating nodes (which includes the internal nodes and output nodes). Further let $\mathcal{P}(\theta, X, V)$ denote the total pseudo-power of the circuit in the free phase. By Lemma \ref{lma:power}, the steady state $V_\star$ of the free phase is such that $\frac{\partial \mathcal{P}}{\partial V}(\theta, X, V_\star) = 0$. In the nudged phase, when a current $I_k = -\beta  \; \frac{\partial C}{\partial V_k}(V_\star, Y)$ is sourced at every floating node $V_k$, Kirchhoff's current law at the steady state $V_\star^\beta$ implies that $\frac{\partial \mathcal{P}}{\partial V}(\theta, X, V_\star^\beta) + \beta \; \frac{\partial C}{\partial V}(V_\star, Y) = 0$. Furthermore, the total pseudo-power (Eq.~\ref{eq:total-pseudo-power}) has the sum-separability property: an adjustable parameter $w_{ij}$ of a component whose terminals are $i$ and $j$ contributes to $\mathcal{P}(\theta, X, V)$ only through the pseudo-power $p_{ij}(V_i - V_j)$ of that component. Therefore, Eq.~\ref{eq:device-gradient} follows from the main Theorem \ref{thm:static-eqprop}.

In the case of a linear resistor, the adjustable parameter is $w_{ij} = g_{ij}$ and the pseudo-power is given by Eq.~\ref{eq:pseudo-resistor}. Thus, Eq.~\ref{eq:conductance-gradient} follows from Eq.~\ref{eq:device-gradient} and the fact that $\frac{\partial p_{ij} \left( \Delta V_{ij} \right)}{\partial g_{ij}} = \frac{1}{2} \left( \Delta V_{ij} \right)^2$.
\end{proof}

As explained in the general setting (Section \ref{sec:equilibrium-propagation}), it is possible to reduce the bias and the variance of the gradient estimator by performing two nudged phases: one with a positive nudging ($+\beta$) and one with a negative nudging ($-\beta$).

Although the framework we have presented here is deterministic, we note that analog circuits in practice are affected by noise. In section \ref{sec:stochastic-setting} we present a stochastic version of EqProp which can model such forms of noise and incorporate effects of thermodynamics.

\subsection{On the Loss Gradient Estimates}

\paragraph{Computing the sign of the gradients.}
Theorem \ref{lma:gradients} provides a formula for computing the gradient of a given device, assuming that the pseudo-power gradient ($\frac{\partial p_{ij}}{\partial w_{ij}}$) of this device is known, and that its terminal voltages can be measured, stored\footnote{The node voltages must be measured and stored at the end of the first phase, since they are no longer physically available after the second phase, at the moment of the weight update. We can achieve this with a sample and hold amplifier circuit.} and retrieved with arbitrary precision. In practice however, these conditions are too stringent.

A piece of good news is that there is empirical evidence that training neural networks by stochastic gradient descent (SGD) works well, even if for each weight, only the sign of the weight gradient is known. Variants of SGD which use the sign of the gradient rather than its exact value work well in practice. At each step of this training procedure, the weight update for $\theta_k$ takes the form $\Delta \theta_k = - \eta \; \text{sign} \left( \frac{\partial {\mathcal L}}{\partial \theta_k} \right)$. The effectiveness of this optmization method has been shown empirically in the context of differentiable neural networks trained with backpropagation \citep{bernstein2018signsgd}.

In the context of nonlinear resistive networks trained with EqProp, if we aim to get the correct sign (rather than its exact value) of the gradient for a given resistor, precise knowledge of the voltage values at the terminals is not necessary. As a corollary of Theorem \ref{lma:gradients}, the sign of the gradient can be obtained by comparing $\left| \Delta V_{ij}^0 \right|$ and $\left| \Delta V_{ij}^\beta \right|$, i.e. the absolute values of the voltages across the resistor in the free phase and the nudged phase\footnote{Equivalently, we can compare $\left| I_{ij}^0 \right|$ and $\left| I_{ij}^\beta \right|$, i.e. the currents through the resistor in the free phase and the nudged phase.}. This means that, if we aim to compute the sign of the gradient, we only need to perform a `compare' operation reliably.

\paragraph{Robustness to characteristics variability.}
A large body of works aims at implementing the backpropagation algorithm in analog \citep{burr2017neuromorphic,xia2019memristive}. However, the weight gradients computed by backpropagation are sensitive to characteristics variability of analog devices. This is because the mathematical derivation of the backpropagation algorithm relies on a \textit{global coordination of elementary operations}: if any of the elementary operations of the algorithm is inaccurate, then the gradients computed are inaccurate (i.e. biased). 

Although there is no experimental evidence for this fact yet, there are reasons to believe that the gradient estimates of EqProp are more robust to device mismatches than the gradients of Backprop. The reason is that, in EqProp, the same circuit is used in both phases of training. Intuitively, any device mismatch will affect the steady states of both phases (free phase and nudged phase), and since the gradient estimate depends on the difference between the measurements of the two phases, the effects of the mismatch will cancel out. More precisely, Theorem \ref{lma:gradients} tells us that the quality of the gradient estimate for a given device does not depend on the characteristics of other devices in the circuit.

We note that this argument holds only for the computation/estimation of the weight gradients. In EqProp like in Backprop, the challenge of weight update asymmetry of programmable resistors remains.

\section{Example of a Deep Analog Neural Network Architecture}
\label{sec:analog-network-model}

The theory of Section \ref{sec:nonlinear-resistive-network-eqprop} applies to any nonlinear resistive network. In this section, as an example of what is possible with this general method, we present the neural network architecture proposed by \citet{kendall2020training}, inspired by the deep Hopfield network model (Figure~\ref{fig:network}). It is composed of multiple layers, alternating linear and non-linear processing stages. The linear transformations are performed by crossbar arrays of programmable resistors, that play the role of weight matrices that parameterize the transformations. The nonlinear transfer function is implemented using a pair of diodes, followed by a linear amplifier. These crossbar arrays of programmable resistors and these nonlinear transfer functions are alternated to form a deep network.

\begin{figure*}[ht!]
\begin{center}
\includegraphics[width=\textwidth]{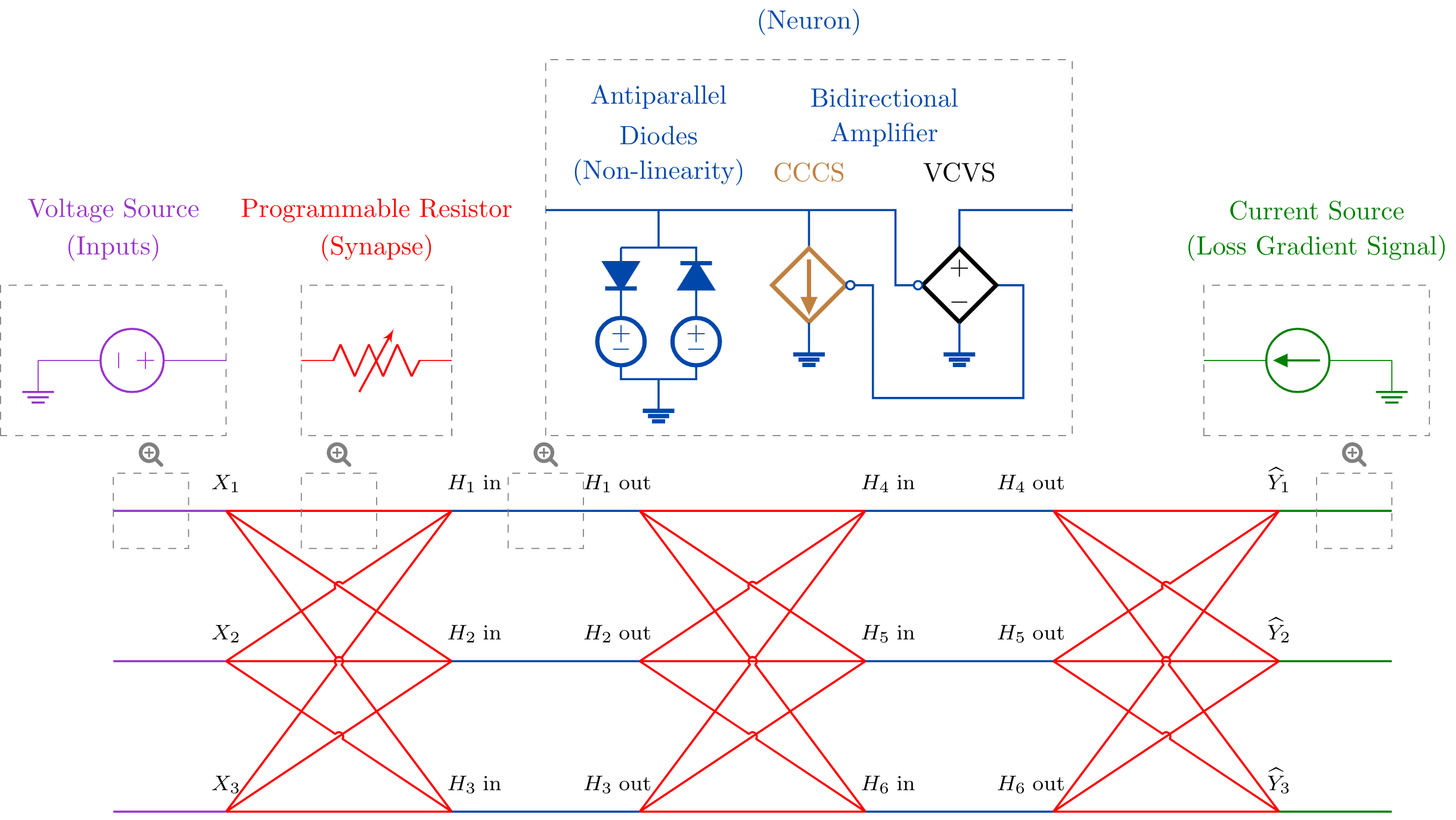}
\end{center}
\caption[Example of a deep analog neural network]{
\textbf{Deep analog neural network} with three input nodes ($X_1$, $X_2$ and $X_3$), two layers of three hidden neurons each ($H_1$, $H_2$, $H_3$, and $H_4$, $H_5$, $H_6$) and three output nodes ($\widehat{Y}_1$, $\widehat{Y}_2$ and $\widehat{Y}_3$). Blue branches and red branches represent neurons and synapses, respectively. Each synapse is a programmable resistor, whose conductance represents a parameter to be adjusted. Each neuron is formed of a nonlinear transfer function and a bidirectional amplifier. The nonlinear transfer function is implemented by a pair of antiparallel diodes (in series with voltage sources), which forms a sigmoidal function in its voltage response (section \ref{sec:diodes}). The bidirectional amplifier consists of a current-controlled current source (CCCS, shown in brown) and a voltage-controlled voltage source (VCVS, shown in black), allowing signals to propagate in both directions without a decay in amplitude (section \ref{sec:amplifiers}). Output nodes are linked to current sources (shown in green) which serve to inject loss gradient signals during training (section \ref{sec:current-sources}). \textbf{Equilibrium Propagation} (EqProp) allows to compute the gradient of a loss $\mathcal{L} = C(\widehat{Y}, Y)$, where $Y$ is the desired target (section \ref{sec:nonlinear-resistive-network-eqprop}). In the free phase (inference), input voltages are sourced at input nodes and the current sources are set to zero current. In the nudged phase (training), for each output node $\widehat{Y}_k$ the corresponding current source is set to $I_k = - \beta \frac{\partial C}{\partial \widehat{Y}_k}$, where $\beta$ is a scaling factor called nudging factor (a hyperparameter). The update rule to adjust the conductances of programmable resistors is local (Theorem \ref{lma:gradients}).
}
\label{fig:network}
\end{figure*}

\subsection{Antiparallel Diodes}
\label{sec:diodes}

We propose to implement the neuron nonlinearities (or `activation functions') as shunt conductances. To do this, we place two diodes antiparallel between the neuron's node and ground. Each diode is placed in series with a voltage source, used to shift the bounds of the activation function. The diodes ensure that the neuron's voltage remains bounded even as its input current grows large, because for any additional input current, one of the diodes turns on and sinks the extra current to ground.

\subsection{Bidirectional Amplifiers}
\label{sec:amplifiers}

In a circuit composed only of resistors and diodes, voltages decay through the resistive layers. This \textit{vanishing signals effect} can be explained by the fact that currents always flow from high electric potential to low electric potential. Thus, extremal voltage values are necessarily reached at input nodes, whose voltages are set.

To counter signal decay, one option is to use voltage-controlled voltage sources (VCVS) to amplify the voltages of hidden neurons in the forward direction. Current-controlled current sources (CCCS) can also be used to amplify currents in the backward direction, to better propagate error signals in the nudged phase. We call such a combination of a forward-directed VCVS and a backward-directed CCCS a `bidirectional amplifier'.

\subsection{Positive Weights}
\label{sec:positive-weights}

Unlike conventional neural networks trained in software whose weights are free to take either positive or negative values, one constraint of analog neural networks is that the conductances of programmable resistors (which represent the weights) are positive. Several approaches are proposed in the literature to overcome this structural constraint. One approach consists in decomposing each weight as the difference of two (positive) conductances \citep{wang2019reinforcement}. Another approach is to shift the mean of the weight matrix by a constant factor \citep{hu2016dot}.

A third approach proposed here consists in doubling the number of input nodes, and to duplicate input values by inverting one set. We also double the number of output nodes so that, in a classification task with $K$ classes, the network has two output nodes for each class $k$, denoted $\widehat{Y}_k^+$ and $\widehat{Y}_k^-$, with $\widehat{Y}_k^+ - \widehat{Y}_k^-$ representing a score assigned to class $k$. The prediction of the model is then
\begin{equation}
    \widehat{Y}_{\rm pred} = \underset{0 \leq k \leq K}{\arg \max} \left( \widehat{Y}_k^+ - \widehat{Y}_k^- \right).
\end{equation}
We optimize the loss associated to the squared error cost function, i.e. $\mathcal{L} = C(V_\star, Y)$, where the target vector $Y = (Y_1, Y_2, \ldots, Y_K)$ is the one-hot code of the class label, and
\begin{equation}
\label{eq:loss3}
C(\widehat{Y}, Y) = \frac{1}{2} \sum_{k=1}^K \left( \widehat{Y}_k^+-\widehat{Y}_k^--Y_k \right)^2.
\end{equation}

\subsection{Current Sources}
\label{sec:current-sources}

The nudged phase requires to inject currents $I_k^+$ and $I_k^-$ at output nodes $\widehat{Y}_k^+$ and $\widehat{Y}_k^-$. These currents must be proportional to the gradients of output node voltages $\widehat{Y}_k^+$ and $\widehat{Y}_k^-$, i.e.
\begin{equation}
    I_k^+ = - \beta \frac{\partial C}{\partial \widehat{Y}_k^+} = \beta \left( Y_k+\widehat{Y}_k^--\widehat{Y}_k^+ \right), \qquad I_k^- = - \beta \frac{\partial C}{\partial \widehat{Y}_k^-} = \beta \left( \widehat{Y}_k^+-\widehat{Y}_k^- - Y_k \right),
\end{equation}
where the nudging factor $\beta$ has the physical dimensions of a conductance. We can inject these currents in the nudged phase using current sources. In the free phase, these current sources are set to zero current and do not influence the voltages of output nodes, acting like open circuits.

\section{Numerical Simulations on MNIST}
\label{sec:numerical-simulations}

\citet{kendall2020training} present simulations on the MNIST digits classification task, performed using the high-performance SPICE-class parallel circuit simulator \textit{Spectre} \citep{2020spectre}. SPICE (simulation program with integrated circuit emphasis) is a framework for realistic simulations of circuit dynamics \citep{vogt2020ngspice}. Specifically, SPICE is used in the simulations to perform the free phase and the nudged phase of the EqProp training process. The other operations are performed in Python: this includes weight initialization (before training starts), calculating loss and gradient currents (between the free phase and the nudged phase), weight gradient calculation (at the end of the nudged phase) and performing the weight updates (resistances are updated in software). We refer to \citet{kendall2020training} for full details of the implementation and simulation results.

Simulations are performed on a small network with a single hidden layer of $100$ neurons. Training is stopped after 10 epochs, when the SPICE network achieves a test error rate of $3.43\%$. For comparison, \citet{lecun1998gradient} report results with different kinds of linear classifiers and logistic regression models (corresponding to different pre-processing methods), all performing $> 7\%$ test error, which is significantly worse than the SPICE network. This demonstrates that the SPICE network benefits from the non-linearities offered by the diodes.

%% file: discrete-time.tex
\chapter{Training Discrete-Time Neural Network Models with Equilibrium Propagation}
\label{chapter:discrete-time}

In the previous chapters, we have presented EqProp in its general formulation (Chapter \ref{chapter:eqprop}), and we have applied it to gradient systems (such as the continuous Hopfield model, Chapter \ref{chapter:hopfield}), and to physical systems that can be described by a variational principle (such as nonlinear resistive networks, Chapter \ref{chapter:neuromorphic}). Although EqProp is a potentially promising tool for training neuromorphic hardware, developing such hardware is still in the future. Whereas in the previous two chapters we have mostly focused on neuroscience and neuromorphic considerations, it is also essential to demonstrate the potential of EqProp to solve practical tasks. In this chapter, we focus on the scalability of EqProp in software, to demonstrate its usefulness as a learning strategy.

When simulated on digital computers, the models presented in the previous chapters are very slow and require long inference times to converge to equilibrium. More importantly, these models have thus far not been proved to scale to tasks harder than MNIST. In this chapter, we present a class of models trainable with EqProp, specifically aimed at accelerating simulations in software, and at scaling EqProp training to larger models and more challenging tasks. As a consequence of this change of perspective, some of the techniques introduced in this chapter can be viewed as a step backward from biorealism and neuromorphic considerations (e.g. the use of shared weights in the convolutional network models). However, the introduction of such techniques allows us to broaden the scope of EqProp and to benchmark it against more advanced models of deep learning.

The present chapter, which is essentially a compilation and a rewriting of \citet{ernoult2019updates} and \citet{laborieux2020scaling}, is organized as follows.
\begin{itemize}
  \item In Section \ref{sec:discrete-time-EqProp}, we present a discrete-time formulation of EqProp, which allows training neural network models closer to those used in conventional deep learning.
  \item In Section \ref{sec:discrete-time-models}, we present discrete-time neural network models trainable with EqProp, including a fully-connected model (close in spirit to the Hopfield model) and a convolutional one. In contrast with previous chapters where we have only considered the squared error as a cost function, we present here a method to optimize the cross-entropy loss commonly used for classification tasks.
  \item In Section \ref{sec:discrete-time-experiments}, we present the experimental results of \citet{ernoult2019updates} and \citet{laborieux2020scaling}. Compared to the experiments of Chapter \ref{chapter:hopfield}, discrete-time models allow one to reduce the computational cost of inference, and enable to scale EqProp to deeper architectures and more challenging tasks. In particular, a ConvNet model trained with EqProp achieves $11.68\%$ test error rate on CIFAR-10. Furthermore, these experiments highlight the importance of reducing the bias and variance of the loss gradient estimators on complex tasks. We also discuss some challenges to overcome in order to unlock the scaling of EqProp to larger models and harder tasks, as well as some promising avenues towards this goal.
  \item In Section \ref{sec:discrete-time-transient-dynamics}, we present a theoretical result linking the transient states in the second phase of EqProp to the partial derivatives of the loss to optimize (Theorem \ref{thm:gdd} and Fig.~\ref{fig:gdd}). This property, which we call the \textit{gradient descending dynamics} (GDD) property, is useful in practice as it prescribes a criterion to decide when the dynamics of the first phase of training has converged to equilibrium.
\end{itemize}

\section{Discrete-Time Dynamical Systems with Static Input}
\label{sec:discrete-time-EqProp}

In this section, we apply EqProp to a class of discrete-time dynamical systems, as proposed by \citet{ernoult2019updates}.

\subsection{Primitive Function}

Consider an energy function of the form
\begin{equation}
    \label{eq:primitive-function}
    E(\theta, x, s) = \frac{1}{2}\| s \|^2 - \Phi(\theta, x, s),
\end{equation}
where $\Phi$ is a scalar function that we will choose later. With this choice of energy function, the equilibrium condition $\frac{\partial E}{\partial s} \left( \theta, x, s_\star \right) = 0$ of Eq.~\ref{eq:free-equilibrium-state} rewrites as a fixed point condition:
\begin{equation}
\label{eq:free-fixed-point}
s_\star = \frac{\partial \Phi}{\partial s} \left( \theta, x, s_\star \right).
\end{equation}
Assuming that the function $s \mapsto \frac{\partial \Phi}{\partial s} \left( \theta, x, s \right)$ is contracting, by the contraction mapping theorem, the sequence of states $s_1$, $s_2$, $s_3$, $\ldots$ defined by
\begin{equation}
\label{eq:free-phase-discrete-time}
s_{t+1} = \frac{\partial \Phi}{\partial s} \left( \theta, x, s_t \right)
\end{equation}
converges to $s_\star$. This dynamical system can be viewed as a recurrent neural network (RNN) with static input $x$ (meaning that the same input $x$ is fed to the RNN at each time step) and transition function $F = \frac{\partial \Phi}{\partial s}$. Because $\Phi$ is a primitive function of the transition function $F$, we call $\Phi$ the \textit{primitive function} of the system. In light of Eq.~\ref{eq:free-fixed-point}, in this chapter we will call $s_\star$ a \textit{fixed point} (rather than an equilibrium state).

The question of necessary and sufficient conditions on $\Phi$ for the dynamics of Eq.~\ref{eq:free-phase-discrete-time} to converge to a fixed point is out of the scope of the present manuscript. We refer to \citet{scarselli2009graph} where conditions on the transition function are discussed.

\subsection{Training Discrete-Time Dynamical Systems with Equilibrium Propagation}

Recall that we want to optimize a loss of the form
\begin{equation}
\mathcal{L} = C \left( s_\star, y \right),
\end{equation}
where $C(s, y)$ is a scalar function called \textit{cost function}, defined for any state $s$. In the discrete-time setting, EqProp takes the following form.

\paragraph{Free Phase.}
In the free phase, the dynamics of Eq.~\ref{eq:free-phase-discrete-time} is run for $T$ time steps, until the sequence of states $s_1, s_2, s_3, \ldots, s_T$ has converged. At the end of the free phase, the network is at the \textit{free fixed point} $s_\star$ characterized by Eq.~\ref{eq:free-fixed-point}, i.e. $s_T = s_\star$.

\paragraph{Nudged Phase.}
In the nudged phase, starting from the free fixed point $s_\star$, an additional term $- \beta \; \frac{\partial C}{\partial s}$ is introduced in the dynamics of the neurons, where $\beta$ is a positive or negative scalar, called \textit{nudging factor}. This term acts as an external force nudging the system dynamics towards decreasing the cost function $C$. Denoting $s_0^\beta, s_1^\beta, s_2^\beta, \ldots$ the sequence of states in the second phase (which depends on the value of $\beta$), we have
\begin{equation}
  s_0^\beta = s_\star \qquad \text{and} \qquad \forall t \geq 0, \quad s_{t+1}^\beta  = \frac{\partial \Phi}{\partial s} \left( \theta, x, s_t^\beta \right) - \beta \; \frac{\partial C}{\partial s} \left( s_t^\beta, y \right).
  \label{eq:nudged-phase-discrete-time}
\end{equation}
The network eventually settles to a new fixed point $s_\star^\beta$, called \textit{nudged fixed point}.

\paragraph{Update Rule.}
In this context, the formula for estimating the loss gradients using the two fixed points $s_\star$ and $s_\star^\beta$ takes the form
\begin{equation}
\lim_{\beta \to 0} \frac{1}{\beta} \left( \frac{\partial \Phi}{\partial \theta} \left( \theta, x, s_\star^\beta \right) -  \frac{\partial \Phi}{\partial \theta} \left( \theta, x, s_\star \right) \right) = -\frac{\partial \mathcal{L}}{\partial \theta}.
\label{eq:eqprop-phi-grad}
\end{equation}
Furthermore, if the primitive function $\Phi$ has the sum-separability property, i.e. if it is of the form $\Phi(\theta, x, s) = \Phi_0(x, s) + \sum_{k=1}^N \Phi_k(\theta_k, x, s)$ where $\theta=(\theta_1, \theta_2, \ldots, \theta_N)$, then
\begin{equation}
\lim_{\beta \to 0} \frac{1}{\beta} \left( \frac{\partial \Phi_k}{\partial \theta_k} \left( \theta_k, x, s_\star^\beta \right) -  \frac{\partial \Phi_k}{\partial \theta_k} \left( \theta_k, x, s_\star \right) \right) = -\frac{\partial \mathcal{L}}{\partial \theta_k}.
\label{eq:eqprop-phi-grad-local}
\end{equation}
Eq.~\ref{eq:eqprop-phi-grad} follows directly from Theorem \ref{thm:static-eqprop} and the definition of $\Phi$ in terms of $E$ (Eq.~\ref{eq:primitive-function}). Eq.~\ref{eq:eqprop-phi-grad-local} follows from Eq.~\ref{eq:eqprop-phi-grad} and the definition of sum-separability.

In the discrete-time setting, as in the other settings, we can reduce the bias and the variance of the gradient estimate by using a symmetrized gradient estimator (see Eq.~\ref{eq:two-sided-estimator}). This requires two nudged phases: one with a positive nudging ($+\beta$) and one with a negative nudging ($-\beta$).

\subsection{Recovering Gradient Systems}

We note that if we choose a primitive function $\Phi$ of the form $\Phi(\theta, x, s) = \frac{1}{2} \| s \|^2 - \epsilon \; \widetilde{E}(\theta, x, s)$, where $\epsilon$ is a positive hyperparameter and $\widetilde{E}$ is a scalar function, then the dynamics of Eq.~\ref{eq:free-phase-discrete-time} rewrites $s_{t+1} = s_t - \epsilon \frac{\partial \widetilde{E}}{\partial s} \left(\theta, x, s_t\right)$. This is the Euler scheme with discretization step $\epsilon$ of the gradient dynamics $\frac{d}{dt} s_t = - \frac{\partial \widetilde{E}}{\partial s} \left( \theta, x, s_t \right)$, which was used in the simulations of Chapter \ref{chapter:hopfield}. In this sense, the setting of gradient systems of Chapter \ref{chapter:hopfield} can be seen as a particular case of the discrete-time formulation presented in this chapter.

\section{RNN Models with Static Input}
\label{sec:discrete-time-models}

The algorithm presented in the previous section is generic and holds for arbitrary primitive function $\Phi$ and cost function $C$. In this section, we present two models corresponding to different choices of primitive function. The first model is a vanilla RNN with static input and symmetric weights (a variant of the Hopfield model). The second model is a convolutional RNN model. We also propose different choices of cost function: whereas in Chapter \ref{chapter:hopfield} and Chapter \ref{chapter:neuromorphic} we have only considered the squared error between outputs and targets, here we also present an implementation of the cross-entropy cost function.

\subsection{Fully Connected Layers}

To implement a fully connected layer, we consider the following primitive function:
\begin{equation}
    \Phi_k^{\rm fc}(w_k, h^{k-1}, h^k) = (h^k)^\top \cdot w_k \cdot h^{k-1}.
\end{equation}
In this expression, $h^{k-1}$ and $h^k$ are two consecutive layers of neurons, and $w_k$ is a weight matrix of size $\dim(h^k) \times \dim(h^{k-1})$ connecting $h^{k-1}$ to $h^k$. We note that $\Phi_k^{\rm fc}$ is closely related to the Hopfield energy of Eq.~\ref{eq:hopfield-energy}.

By stacking several of these fully connected layers, we can form a network of multiple layers of the kind considered in the previous chapters (e.g. as depicted in Fig~\ref{fig:network_undirected}). The corresponding primitive function is obtained by summing together the primitive functions of individual pairs of layers. For example, consider $\Phi = \cdots + \Phi_k^{\rm fc} + \Phi_{k+1}^{\rm fc} + \cdots$, i.e.
\begin{equation}
\Phi = \cdots + (h^k)^\top \cdot w_k \cdot h^{k-1} + (h^{k+1})^\top \cdot w_{n+1} \cdot h^k + \cdots.
\label{eq:model-simple-ep-phi}
\end{equation}
For this choice of primitive function, we have $\frac{\partial \Phi}{\partial h^k} = w_k \cdot h^{k-1} + w_{n+1}^\top \cdot h^{k+1}$. In practice, we find that it is necessary that the values of the state variable be bounded. For this reason, we apply an activation function $\sigma$ and arrive at the following dynamics in the free phase, which is a discrete-time variant of the dynamics of the Hopfield network studied in Chapter \ref{chapter:hopfield}:
\begin{equation}
h^k_{t+1} = \sigma \left( w_k \cdot h^{k-1}_t + w_{n+1}^\top \cdot h^{k+1}_t \right).
\label{eq:EqProp-simple-first-phase}
\end{equation}

\subsection{Convolutional Layers}
\label{sec:convolutional-layer}

\citet{ernoult2019updates} propose to implement a convolutional layer with the following primitive function:
\begin{equation}
\label{eq:phiCNN}
    \Phi_k^{\rm conv}(w_k, h^{k-1}, h^k) = h^k\bullet\mathcal{P}\left(w_k\star h^{k-1}\right).
\end{equation}
In this expression, $w_k$ is the kernel (convolutional weights) for that layer, $\star$ is the convolution operator, $\mathcal{P}$ is a pooling operation, and $\bullet$ is the canonical scalar product for pairs of tensors with same dimension. In particular, in this expression, $h^k$ and $\mathcal{P}\left(w_k\star h^{k-1}\right)$ are tensors with same size. This implementation is similar to the one proposed by \citet{lee2009convolutional} in the context of restricted Boltzmann machines.

By stacking several of these convolutional layers we can form a deep ConvNet (specifically a recurrent convolutional network with static input and symmetric weights). Consider a primitive function of the form $\Phi = \cdots + \Phi_k^{\rm conv} + \Phi_{k+1}^{\rm conv} + \cdots$, i.e.
\begin{equation}
\Phi = \cdots + h^k\bullet\mathcal{P}\left(w_k\star h^{k-1}\right) + h^{k+1}\bullet\mathcal{P}\left(w_{n+1}\star h^k\right) + \cdots
\end{equation}
We have $\frac{\partial \Phi}{\partial h^k} = \mathcal{P}\left(w_k\star h^{k-1}\right) + \tilde{w}_{k+1}\star \mathcal{P}^{-1}\left(h^{k+1}\right)$, where $\mathcal{P}^{-1}$ is an `inverse pooling' operation, and $\tilde{w}_k$ is the flipped kernel, which forms the transpose convolution. We refer to \citet{ernoult2019updates} where these operations are defined in details. After restricting the space of the state variables by using the hardsigmoid activation function $\sigma$ to clip the states, we obtain the following dynamics for layer $h^k$:
\begin{equation}
\label{eq:conv-dynamics}
h^k_{t+1} = \sigma \left( \mathcal{P}\left(w_k\star h^{k-1}_t\right) + \tilde{w}_{n+1}\star \mathcal{P}^{-1}\left(h^{k+1}_{t}\right)\right).
\end{equation}

We can also combine convolutional layers, followed by fully connected layers, to form a more practical deep ConvNet. Denoting $N^{\rm conv}$ and $N^{\rm fc}$ the number of convolutional layers and fully connected layers, the total number of layers is $N^{\rm tot} = N^{\rm conv} + N^{\rm fc}$ and the primitive function is
\begin{equation}
    \Phi(\theta, x, s) = \sum_{k=1}^{N^{\rm conv}} \Phi_k^{\rm conv}(w_k, h^{k-1}, h^k) + \sum_{k = N^{\rm conv}+1}^{N^{\rm tot}} \Phi_k^{\rm fc}(w_k, h^{k-1}, h^k),
\end{equation}
where the set of parameters is $\theta = \{w_k\}_{1 \leq k \leq N^{\rm tot}}$, the input is $x=h^0$, and the state variable is $s = \{h^k\}_{1 \leq k \leq N^{\rm tot}}$.

\subsection{Squared Error}

We have already studied in Chapters \ref{chapter:hopfield} and \ref{chapter:neuromorphic} the case where we optimize the loss associated to the squared error cost function. In this setting, the state variable of the network is of the form $s=(h, o)$, where $h$ represents the \textit{hidden neurons} and $o$ the \textit{output neurons}, and the cost function is
\begin{equation}
    C(o, y) = \frac{1}{2} \left\| o - y \right\|^2.
\end{equation}
The nudged phase dynamics of the hidden neurons and output neurons read, in this context:
\begin{equation}
    h_{t+1}^{\beta} = \frac{\partial \Phi}{\partial h}(\theta, x, h_t^{\beta}, o_t^{\beta}), \qquad o_{t+1}^{\beta} = \frac{\partial \Phi}{\partial o}(\theta, x, h_t^{\beta}, o_t^{\beta}) + \beta \; (y - o_t^{\beta}).
\end{equation}

\subsection{Cross-Entropy}

\citet{laborieux2020scaling} present a method to implement the output layer of the neural network as a softmax output, which can be used in conjunction with the cross-entropy loss. 
In this setting, the state of output neurons ($o$) are not a part of the state variable ($s$), but are instead viewed as a readout, which is a function of $s$ and of a weight matrix $w_{\rm out}$ of size $\dim(y) \times \dim(s)$. Specifically, the state of output neurons at time step $t$ is defined by the formula:
\begin{equation}
    o_t = \mbox{softmax}(w_{\rm out}\cdot s_t).
\end{equation}
Denoting $M=\dim(y)$ the number of categories in the classification task of interest, the cross-entropy cost function associated with the softmax output is then:
\begin{equation}
C(s, y, w_{\rm out}) = - \sum_{i=1}^M y_i \log(\textrm{softmax}_i(w_{\rm out} \cdot s)).
\end{equation}
Using the fact that $\frac{\partial C}{\partial s}(s, y, w_{\rm out}) = w_{\rm out}^\top \cdot \left( \textrm{softmax}(w_{\rm out} \cdot s) - y \right)$, the nudged phase dynamics corresponding to the cross-entropy cost function read
\begin{equation}
s_{t+1}^{\beta} = \frac{\partial \Phi}{\partial s}(\theta, x, s_t^\beta) + \beta \; w_{\rm out}^\top \cdot \left( y - o_t^\beta \right),
\end{equation}
where $o_t^\beta = \textrm{softmax}(w_{\rm out} \cdot s_t^\beta)$. Note that in this context the loss $\mathcal{L} = C(s_\star, y, w_{\rm out})$ also depends on the parameter $w_{\rm out}$. The loss gradient with respect to $w_{\rm out}$ is given by
\begin{equation}
\frac{\partial \mathcal{L}}{\partial w_{\rm out}} = s_\star^\top \cdot \left( y - o_\star \right),
\end{equation}
where $o_\star = \textrm{softmax}(w_{\rm out} \cdot s_\star)$.

In practice, the state variable is of the form $s = (h^1, h^2, \ldots, h^N)$, where $h^1, h^2, \ldots, h^N$ are the hidden layers of the network, and $w_{\rm out}$ connects only the last hidden layer $h^N$ (not all hidden layers) to the output layer $o$. The weight matrix $w_{\rm out}$ has size $\dim(y) \times \dim(h^N)$ in this case.

\section{Experiments on MNIST and CIFAR-10}
\label{sec:discrete-time-experiments}

In this section, we present the experimental results of \citet{ernoult2019updates} and \citet{laborieux2020scaling}, on the MNIST (Table~\ref{table:mnist-convnet-results}) and the CIFAR-10 (Table~\ref{table:cifar-convnet-results}) classification tasks, respectively. The CIFAR-10 dataset \citep{krizhevsky2009learning} consists of $60,000$ colour images of $32 \times 32$ pixels. These images are split in $10$ classes (each corresponding to an object or animal), with $6,000$ images per class. The training set consists of $50,000$ images and the test set of $10,000$ images.

Experiments are performed on different network architectures (composed of multiple fully-connected and/or convolutional layers), using different cost functions (either the squared error or the cross-entropy loss) and different loss gradient estimators. Using the notations of this chapter, the \textit{one-sided} gradient estimator ($\widehat{\nabla}_\theta(\beta)$) and the \textit{symmetric} gradient estimator ($\widehat{\nabla}_\theta^{\rm sym}(\beta)$) presented in Chapter~\ref{chapter:eqprop} take the form
\begin{align}
\label{eq:one-sided-phi}
\widehat{\nabla}_\theta(\beta) & = \frac{1}{\beta} \left( \frac{\partial \Phi}{\partial \theta}(\theta, x, s_\star^\beta) - \frac{\partial \Phi}{\partial \theta}(\theta, x, s_\star) \right), \\
\label{eq:symmetric-phi}
\widehat{\nabla}_\theta^{\rm sym}(\beta) & = \frac{1}{2 \beta} \left( \frac{\partial \Phi}{\partial \theta} \left( \theta, x, s_\star^\beta \right) -  \frac{\partial \Phi}{\partial \theta} \left( \theta, x, s_\star^{-\beta} \right) \right).
\end{align}
We refer to \citet{ernoult2019updates} and \citet{laborieux2020scaling} for the implementation and simulation details. Finally, since the models considered here are RNNs, we can also train them with the more conventional \textit{backpropagation through time} (BPTT) algorithm\footnote{In this case, BPTT is used on RNNs of a very specific kind. The RNN models considered here have a transition function of the form $F = \frac{\partial \Phi}{\partial s}$, a static input $x$ at each time step, and a single target $y$ at the final time step.}, and use BPTT as a benchmark for EqProp.

\begin{table}[ht!]
\centering
$\begin{array}{|cc|ccc|ccc|cc|}
\hline
	 &  & \multicolumn{3}{c|}{\rm EqProp \; Error \; (\%)} & & & & \multicolumn{2}{c|}{\rm BPTT \; Error \; (\%)} \\
	\hbox{Model} & \hbox{Loss} & \hbox{Estimator} & {\rm Test} & {\rm Train} & T   & K & {\rm Epochs} & {\rm Test} & {\rm Train} \\
\hline
    \hbox{DHN-1h} & \multirow{2}{*}{Squared Error} & \multirow{2}{*}{One-sided} & 2.06 & 0.13  & 100 & 12  & 30 & 2.11 & 0.46 \\
    \hbox{DHN-2h} & & & 2.01 & 0.11 & 500 & 40  & 50 & 2.02 & 0.29 \\
\hline
    \hbox{FC-1h} & \multirow{4}{*}{Squared Error} & \multirow{4}{*}{One-sided} & 2.00 & 0.20 & 30 & 10  &  30 & 2.00 & 0.55 \\
    \hbox{FC-2h} & & & 1.95 & 0.14 & 100 & 20 &  50   & 2.09 & 0.37 \\
    \hbox{FC-3h} & & & 2.01 & 0.10 & 180 & 20  &  100 & 2.30 & 0.32 \\
    \hbox{ConvNet} & & & 1.02 & 0.54 & 200 & 10 &  40 & 0.88 & 0.12 \\
\hline
\end{array}$
\caption[Experimental results of \citet{ernoult2019updates} on discrete-time neural network models trained on MNIST.]{
Experimental results of \citet{ernoult2019updates} on MNIST. EqProp is benchmarked against BPTT.
`DHN' stands for the `deep Hopfield networks' of Chapter \ref{chapter:hopfield}.
`FC' means `fully connected', and `-$\#$h' stands for the number of hidden layers.
The test error rates and training error rates (in \%) are averaged over five trials. $T$ is the number of iterations in the first phase. $K$ is the number of iterations in the second phase. All these results are obtained with the squared error and the one-sided gradient estimator.
}
\label{table:mnist-convnet-results}
\end{table}

\begin{table}[ht!]
\centering
$\begin{array}{|cc|ccc|ccc|cc|}
\hline
	 &  & \multicolumn{3}{c|}{\rm EqProp \; Error \; (\%)} & & & & \multicolumn{2}{c|}{\rm BPTT \; Error \; (\%)} \\
	\hbox{Model} & \hbox{Loss} & \hbox{Estimator} & {\rm Test} & {\rm Train} & T   & K & {\rm Epochs} & {\rm Test} & {\rm Train} \\
\hline
	\multirow{3}{*}{\rm ConvNet} & \multirow{3}{*}{\rm Squared Error} & \hbox{One-sided} & 86.64  & 84.90 & 250           & 30      & 120 & \multirow{3}{*}{11.10}  & \multirow{3}{*}{3.69}   \\
	                             & & \hbox{Random Sign} & \color{blue}{12.61^\star} & \color{blue}{8.64^\star} & 250 & 30 & 120 & & \\
	                             & & \hbox{Symmetric}   & 12.45 & 7.83 & 250 & 30 & 120 & & \\
\hline
	\hbox{ConvNet} & \hbox{Cross-Ent.} & \hbox{Symmetric} & 11.68  & 4.98 & 250 & 25 & 120 & 11.12  & 2.19 \\
\hline
\end{array}$
\caption[Experimental results of \citet{laborieux2020scaling} on ConvNets trained on CIFAR-10.]{Experimental results of \citet{laborieux2020scaling} on CIFAR-10. EqProp is benchmarked against BPTT. The test error rates and training error rates (in \%) are averaged over five trials. $T$ is the number of iterations in the first phase. $K$ is the number of iterations in the second phase.
The `one-sided' and `symmetric' gradient estimators refer to Eq.~\ref{eq:one-sided-phi} and Eq.~\ref{eq:symmetric-phi}, respectively. `random sign' refers to the one-sided estimator with $\beta$ being positive or negative with even probability.\\
\color{blue}{$^\star$In the simulations with random $\beta$, the training process collapsed in one trial out of five, leading to a performance similar to the one-sided estimator. The test error mean and train error mean reported here include only the four trials that worked fine.}
}
\label{table:cifar-convnet-results}
\end{table}

Table~\ref{table:mnist-convnet-results} compares the performance on MNIST of the discrete-time models presented in this chapter (FC-\#h and ConvNet) with the continuous-time Hopfield networks of Chapter \ref{chapter:hopfield} (DHN-\#h). No degradation of accuracy is observed when using discrete-time rather than continuous-time networks, although the former require many less time steps in the first phase of training ($T$). The lowest test error rate ($\sim 1 \%$) is achieved with the ConvNet model. 

Table~\ref{table:cifar-convnet-results} shows the performance of a ConvNet model on CIFAR-10, for different gradient estimators and different loss functions. Unlike in the MNIST experiments, the one-sided gradient estimator with a nudging factor ($\beta$) of constant sign works poorly on CIFAR-10: training is unstable and the network is unable to fit the training data (84.90\% train error). The bias of the one-sided gradient estimator can be reduced on average by choosing the sign of $\beta$ at random in the second phase: with this technique, \citet{laborieux2020scaling} report that training proceeded well in four runs out of five, yielding a mean test error of 12.61\%, but training collapsed in the last run in a way similar to the one-sided gradient estimator with constant sign. The symmetric difference estimator allows to reduce not only the bias but also the variance, and to stabilize the training process consistently across runs (12.45\% test error). Finally, the best test error rate, $11.68\%$, is obtained with the cross-entropy loss, and approaches the performance of BPTT with less than 0.6 \% degradation in accuracy.


\subsection{Challenges with EqProp Training}
\label{eq:difficulties}

The theoretical guarantee that EqProp can approximate with arbitrary precision the gradient of arbitrary loss functions for a very broad class of models (energy-based models) suggests that EqProp could eventually train large networks on challenging tasks, as was proved feasible in the last decade with other deep learning training methods (e.g. backpropagation) relying on stochastic gradient descent. Nevertheless, EqProp training on current processors (GPUs) presents several challenges.

One difficulty encountered with EqProp training is that, although the gradient formula (Eq.~\ref{eq:eqprop-phi-grad}) requires that $\frac{\partial \Phi}{\partial \theta}$ be measured \textit{exactly} at the fixed points, in many situations however, these fixed points are only approached up to certain precision. Empirically, we observe that for learning to work, the fixed point of the first phase of training must be approximated with very high accuracy ; otherwise the gradient estimate is of poor quality and does not enable to optimize the loss function. This implies that the equations of the first phase of training need to be iterated a large number of time steps, until convergence to the fixed point. Table~\ref{table:mnist-convnet-results} and Table~\ref{table:cifar-convnet-results} show that hundreds of iterations are required for the networks to converge, even though these networks consist of just a few layers.

Various methods have been investigated to accelerate convergence, none of which has proved really satisfying so far. \citet{scellier2016towards} propose a method based on variational inference, in which the state variables are split in two groups (specifically the layers of odd indices and the layers of even indices): at each iteration, one group of state variables remains fixed, while the other group is updated by solving for the stationarity condition. \citet{bengio2016feedforward} give a sufficient condition so that initialization of the network with a forward pass provides sensible initial states for inference ; the condition is that any two successive layers must form a `good autoencoder'. \citet{o2018initialized} use a side network to learn these initial states for inference (in the main network).

One promising avenue to solve the problem of long inference times is offered by the recent work of \citet{ramsauer2020hopfield}, which shows that for a certain class of \textit{modern Hopfield networks}, equilibrium states are reached in exactly one step.
This idea
could considerably accelerate simulations in software and demonstrate EqProp training on more advanced architectures and harder tasks. We emphasize however that the difficulty of long inference times is specific to numerical simulations (i.e. simulations on digital computers), and may not be a problem for neuromorphic hardware, where energy minimization is performed by the physics of system (Chapter \ref{chapter:neuromorphic}).

A second difficulty with EqProp training is due to the saturation of neurons. All experiments so far have found that for EqProp training to be effective, the neurons' states need to be clipped to a closed interval, typically $[0, 1]$. This is achieved in most experiments by applying the hard-sigmoid activation function $\sigma(s) = \min(\max(0, s), 1)$ after each iteration during inference. Due to this technique however, many neurons `saturate', i.e. they have a value of exactly $0$ or $1$ at equilibrium. In the second phase of training, due to these saturated neurons, error signals have difficulty propagating from output neurons across the network, when the nudging factor $\beta$ is small. To mitigate this problem, in most experiments $\beta$ is chosen large enough so as to amplify and better propagate error signals along the layers, at the cost of degrading the quality of the gradient estimate. To counter this problem, \citet{o2018initialized} suggest to use a modified activation function which includes a leak term, namely $\sigma^{\rm mod}(s) = \sigma(s) + 0.01 s$. Another avenue to further reduce the saturation effect is to search weight initialization schemes specifically meant for the kind of network models trained with EqProp. The weight initialisation schemes that dominate deep learning today have been designed to fit feedforward nets \citep{he2015delving} and RNNs \citep{saxe2013exact} trained with automatic differentiation. Finding appropriate weight initialization schemes for the kind of bidirectional networks studied in our context is an area of research largely unexplored.

The third difficulty with EqProp training is hyperparameter tuning, due to the high sensitivity of the training process to some of the hyperparameters. Initial learning rates for example need to be tuned layer-wise. In addition to the usual hyperparameters (architecture, learning rates, ...), EqProp requires tuning some additional hyperparameters: the number of iterations in the free phase ($T$), the number of iterations in the nudged phase ($K$), the value of the nudging factor ($\beta$), ... In the next section, we present a theoretical result called the \textit{GDD property} that can help accelerate hyperparameter search. As we will see, the GDD property provides a criterion to decide whether the fixed point of the first phase has been reached or not.

\section{Gradient Descending Dynamics (GDD)}
\label{sec:discrete-time-transient-dynamics}

The gradient formula of Eq.~\ref{eq:eqprop-phi-grad} depends only on the fixed points $s_\star$ and $s_\star^\beta$, not on the specific trajectory that the network follows to reach them. But similarly to the real-time setting of Chapter \ref{chapter:hopfield}, assuming the dynamics of Eq.~\ref{eq:nudged-phase-discrete-time} when the neurons gradually move from their free fixed point values ($s_\star$) towards their nudged fixed point values ($s_\star^\beta$), we can show that the transient states of the network ($s_t^\beta$ for $t \geq 0$) perform step-by-step gradient computation.

\subsection{Transient Dynamics}

First, note that the gradient of EqProp (Eq.~\ref{eq:eqprop-phi-grad}), which is equal to the gradient of the loss in the limit $\beta \to 0$, can be decomposed as a telescoping sum:
\begin{equation}
    \label{eq:telescoping-sum}
    \frac{1}{\beta} \left( \frac{\partial \Phi}{\partial \theta} \left( \theta, x, s_\star^\beta \right) - \frac{\partial \Phi}{\partial \theta} \left( \theta, x, s_\star \right) \right) = \sum_{t=0}^\infty \frac{1}{\beta} \left( \frac{\partial \Phi}{\partial \theta} \left( \theta, x, s_{t+1}^\beta \right) - \frac{\partial \Phi}{\partial \theta} \left( \theta, x, s_t^\beta \right) \right).
\end{equation}

Second, we rewrite the dynamics of the free phase (Eq~\ref{eq:free-phase-discrete-time}) in the form
\begin{equation}
    s_{t+1} = \frac{\partial \Phi}{\partial s} \left( \theta_{t+1} = \theta, x, s_t \right),
\end{equation}
where $\theta_t$ denotes the parameter of the model at time step $t$, the value $\theta$ being shared across all time steps. We consider the loss after $T$ time steps:
\begin{equation}
    \mathcal{L}_T = C \left( s_T, y \right).
\end{equation}
$\mathcal{L}_T$ is what we have called the \textit{projected cost function} in the setting of real-time dynamics (Eq.~\ref{eq:projected-cost-function}). Rewriting the free phase dynamics this way allows us to define the partial derivative $\frac{\partial \mathcal{L}_T}{\partial \theta_t}$ as the sensitivity of the loss $\mathcal{L}_T$ with respect to $\theta_t$, when $\theta_1, \ldots \theta_{t-1}, \theta_{t+1}, \ldots \theta_T$ remain fixed (set to the value $\theta$). With these notations, the full gradient of the loss can be decomposed as
\begin{equation}
\label{eq:total-gradient}
\frac{\partial \mathcal{L}_T}{\partial \theta} = \frac{\partial \mathcal{L}_T}{\partial \theta_1} + \frac{\partial \mathcal{L}_T}{\partial \theta_2} + \cdots + \frac{\partial \mathcal{L}_T}{\partial \theta_T}.
\end{equation}

The following result links the right-hand sides of Eq.~\ref{eq:telescoping-sum} and Eq.~\ref{eq:total-gradient} term by term.

\medskip

\begin{thm}[\citet{ernoult2019updates}]
\label{thm:gdd}
Let $s_0, s_1, \ldots, s_T$ be the sequence of states in the free phase. Suppose that the sequence has converged to the fixed point $s_\star$ after $T-K$ time steps for some $K \geq 0$, i.e. that $s_\star = s_T = s_{T-1} = \ldots s_{T-K}$. Then, the following identities hold at any time $t = 0, 1, \ldots, K-1$ in the nudged phase:
\begin{gather}
\label{eq:gdd-theta}
\lim_{\beta \to 0} \frac{1}{\beta} \left( \frac{\partial \Phi}{\partial \theta} \left( \theta, x, s_{t+1}^\beta \right) - \frac{\partial \Phi}{\partial \theta} \left( \theta, x, s_t^\beta \right) \right) = - \frac{\partial \mathcal{L}_T}{\partial \theta_{T-t}},\\
\label{eq:gdd-s}
\lim_{\beta \to 0} \frac{1}{\beta} \left( s_{t+1}^\beta - s_t^\beta \right) = - \frac{\partial \mathcal{L}_T}{\partial s_{T-t}}.
\end{gather}
\end{thm}

We refer to \citet{ernoult2019updates} for a proof. Theorem \ref{thm:gdd} relates neural computation to gradient computation, and is as such a discrete-time variant of Theorem \ref{thm:truncated-eqprop}. In essence, Theorem \ref{thm:gdd} shows that in the nudged phase of EqProp, the temporal variations in neural activity and incremental weight updates represent loss gradients. Since the sequence of states in the nudged phase satisfies $s_{t+1}^\beta = s_t^\beta - \beta \; \frac{\partial \mathcal{L}_T}{\partial s_{T-t}} + o(\beta)$ as $\beta \to 0$, descending the gradients of the loss $\mathcal{L}_T$, we call this property the \textit{gradient descending dynamics} (GDD) property.

As mentioned in section \ref{eq:difficulties}, one of the challenges with EqProp training comes from the empirical observation that learning is successful only if we are \textit{exactly} at the fixed point at the end of the first phase, although in practice we use numerical methods to \textit{approximate} this fixed point. In particular, we need a criterion to `decide' when the fixed point has been reached with high enough accuracy. Theorem \ref{thm:gdd} provides such a criterion: a necessary condition for the fixed point of the first phase to be reached is that the identities of Eqs.~\ref{eq:gdd-theta}-\ref{eq:gdd-s} hold.

\subsection{Backpropagation Through Time}

On the one hand we can define the neural and weight increments of EqProp as follows, which we can compute in the second phase of EqProp:
\begin{align}
   \Delta_\theta^{\rm EP}(\beta, t) & = \frac{1}{\beta} \left( \frac{\partial \Phi}{\partial \theta} \left( \theta, x, s_{t+1}^\beta \right) -  \frac{\partial \Phi}{\partial \theta} \left( \theta, x, s_t^\beta \right) \right), \\
   \Delta_s^{\rm EP}(\beta, t) & = \frac{1}{\beta} \left( s_{t+1}^\beta-s_t^\beta \right).
\end{align}

On the other hand, the loss gradients $\frac{\partial \mathcal{L}_T}{\partial s_{T-t}}$ and $\frac{\partial \mathcal{L}_T}{\partial \theta_{T-t}}$ appearing on the right-hand sides of Eqs.~\ref{eq:gdd-theta}-\ref{eq:gdd-s} can be computed by automatic differentiation. Specifically, these loss gradients are the `partial derivatives' computed by the \textit{backpropagation through time} (BPTT) algorithm. Here BPTT is applied in the very specific setting of an RNN with transition function $F = \frac{\partial \Phi}{\partial s}$, with static input $x$ at each time step, and with target $y$ at the final time step. In particular, there is no time-dependence in the data. We denote these partial derivatives computed by BPTT:
\begin{align}
 \nabla^{\rm BPTT}_\theta(t) & = \frac{\partial \mathcal{L}_T}{\partial \theta_{T-t}}, \\
 \nabla^{\rm BPTT}_s(t) & = \frac{\partial \mathcal{L}_T}{\partial s_{T-t}}.
\end{align}

Using these notations, the GDD property (Theorem \ref{thm:gdd}) states that under the condition that $s_\star = s_T = s_{T-1} = \ldots s_{T-K}$, we have for every $t = 0, 1, \ldots K-1$ that $\lim_{\beta \to 0} \Delta_s^{\rm EP}(\beta, t) = - \nabla^{\rm BPTT}_s(t)$ and $\lim_{\beta \to 0} \Delta_\theta^{\rm EP}(\beta, t) = - \nabla^{\rm BPTT}_\theta(t)$. The GDD property is illustrated in Figure \ref{fig:gdd}.

We note that the GDD property also implies that the gradient computed by `truncated EqProp' (i.e. EqProp where the second phase is halted before convergence to the second fixed point) corresponds to truncated BPTT.

\begin{figure}[ht!]
\begin{center}
   \fbox{\includegraphics[width=\textwidth]{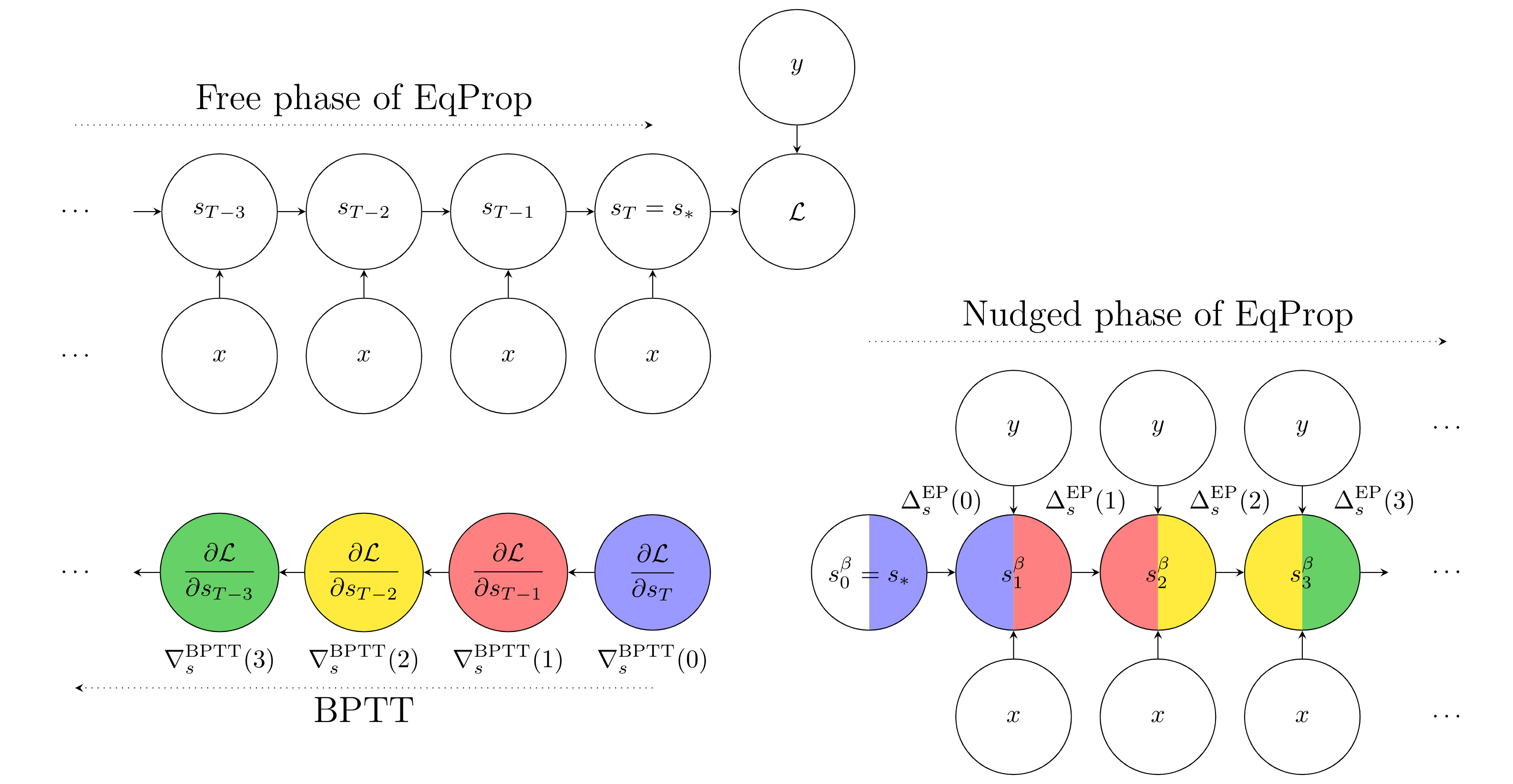}}
\end{center}
  \caption[Illustration of the gradient-descending dynamics (GDD) property]{\textbf{Illustration of the gradient-descending dynamics (GDD) property (Theorem \ref{thm:gdd}).}
  \textbf{Top left.} Free phase. The final state $s_T$ is the fixed point $s_\star$. \textbf{Bottom left.} Backpropagation through time (BPTT), in the very specific setting of an RNN with static input $x$. \textbf{Bottom right.} Nudged phase of EqProp. The starting state in the nudged phase is the final state of the free phase, i.e. the fixed point $s_\star$. \textbf{Theorem \ref{thm:gdd}.} Step by step correspondence between the neural increments $\Delta_s^{\rm EP}(t)$ in the nudged phase of EqProp and the gradients $\nabla_s^{\rm BPTT}(t)$ of BPTT. Corresponding computations in EqProp and BPTT at timestep $t=0$ (resp. $t=1, 2, 3$) are colored in blue (resp. red, yellow, green). Forward-time computation in EqProp corresponds to backward-time computation in BPTT.}
  \label{fig:gdd}
\end{figure}

%% file: on-going-developments.tex
\chapter{Extensions of Equilibrium Propagation}
\label{chapter:future}

In this chapter, we present research directions for the development of the equilibrium propagation framework. In section \ref{sec:time-varying-setting}, we present a general framework for training dynamical systems with time-varying inputs, which exploits the principle of least action. In section \ref{sec:stochastic-setting}, we adapt the EqProp framework to the setting of stochastic systems. In section \ref{sec:contrastive-meta-learning}, we briefly present the \textit{contrastive meta-learning} framework of \citet{zucchet2021contrastive}, where they use the EqProp method to train the meta-parameters of a meta-learning model.

\section{Equilibrium Propagation in Dynamical Systems with Time-Varying Inputs}
\label{sec:time-varying-setting}

In Chapter \ref{chapter:eqprop}, we have derived the EqProp training procedure for a class of models called energy-based models (EBMs). A key element of the theory is the fact that, in EBMs, equilibrium states are characterized by variational equations. In this section, we show that the EqProp training strategy can be applied to other situations where variational equations appear. The equations of motion of many physical systems can also be characterized by variational equations -- their trajectory can be derived through a \textit{principle of stationary action} (e.g. a principle of least action). In such systems, the quantity that is stationary is not the energy function (as in an EBM), but the \textit{action functional}, which is by definition the time integral of the Lagrangian function. Such systems, which we call \textit{Lagrangian-based models} (LBMs), can play the role of machine learning models with time-varying inputs. This idea was first proposed by \citet{baldi1991contrastive} in the context of the \textit{contrastive learning framework}.

\subsection{Lagrangian-Based Models}

A Lagrangian-based model (LBM) is specified by a set of adjustable parameters, denoted $\theta$, a time-varying input, and a state variable. We write $\mathbf{x}_t$ the input value at time $t$, and $\mathbf{s}_t$ the state of the model at time $t$. We study the evolution of the system over a time interval $[0,T]$, and we write $\mathbf{x}$ and $\mathbf{s}$ the entire input and state trajectories over this time interval. The model is further described in terms of a \textit{functional} $\mathcal{S}$ which, given a parameter value $\theta$ and an input trajectory $\mathbf{x}$, associates to each trajectory $\mathbf{s}$ the real number
\begin{equation}
    \mathcal{S}(\theta, \mathbf{x}, \mathbf{s}) = \int_0^T L(\theta, \mathbf{x}_t, \mathbf{s}_t, \dot{\mathbf{s}}_t) dt,
\end{equation}
where $L(\theta, \mathbf{x}_t, \mathbf{s}_t, \dot{\mathbf{s}}_t)$ is a scalar function of the parameters ($\theta$), the external input ($\mathbf{x}_t$), the state of the system ($\mathbf{s}_t$) as well as its time derivative ($\dot{\mathbf{s}_t}$). The function $L$ is called the \textit{Lagrangian function} of the system, and $\mathcal{S}$ is called the \textit{action functional}. The action functional is defined for any \textit{conceivable} trajectory $\mathbf{s}$, but, among all conceivable trajectories, the \textit{effective} trajectory of the system (subject to $\theta$ and $\mathbf{x}$), denoted $\mathbf{s}(\theta, \mathbf{x})$, satisfies by definition $\left. \frac{d}{d \epsilon} \right|_{\epsilon=0} \mathcal{S}(\theta, \mathbf{x}, \mathbf{s}(\theta, \mathbf{x}) + \epsilon \mathbf{s}) = 0$ for any trajectory $\mathbf{s}$ that satisfies the boundary conditions $\mathbf{s}_0=0$ and $\mathbf{s}_T=0$. We write for short
\begin{equation}
    \label{eq:free-effective-trajectory}
    \frac{\delta \mathcal{S}}{\delta \mathbf{s}}(\theta, \mathbf{x}, \mathbf{s}(\theta, \mathbf{x})) = 0.
\end{equation}
Intuitively, $\delta \mathcal{S}$ can be thought of as the variation of $\mathcal{S}$ associated to a small variation $\delta \mathbf{s}$ around the trajectory $\mathbf{s}(\theta, \mathbf{x})$. Mathematically, $\frac{\delta \mathcal{S}}{\delta \mathbf{s}}(\theta, \mathbf{x}, \mathbf{s}(\theta, \mathbf{x}))$ represents the differential of the function $\mathcal{S}(\theta, \mathbf{x}, \cdot)$ at the point $\mathbf{s}(\theta, \mathbf{x})$. We say that the effective trajectory is stationary with respect to the action functional, and that the dynamics of the system derives from a \textit{principle of stationary action}. Since the action functional ($\mathcal{S}$) is defined in terms of the Lagrangian of the system ($L$), we call such a time-varying system a \textit{Lagrangian-based model}.

The loss to be minimized is an integral of cost values of the states along the effective trajectory:
\begin{equation}
    \mathcal{L}_0^T(\theta, \mathbf{x}, \mathbf{y}) = \int_0^T c_t(\mathbf{s}_t(\theta, \mathbf{x}), \mathbf{y}_t) dt.
\end{equation}
In this expression, $\mathbf{y}_t$ is the desired target at time $t$, $\mathbf{y}$ is the corresponding target trajectory, $\mathbf{s}_t(\theta, \mathbf{x})$ is the state at time $t$ along the effective trajectory $\mathbf{s}(\theta, \mathbf{x})$, and $c_t(\mathbf{s}_t, \mathbf{y}_t)$ is a scalar function (the \textit{cost function} at time $t$).

Similarly to the setting of EBMs, the concept of \textit{sum-separability} is useful in LBMs. Let $\theta = (\theta_1, \ldots, \theta_N)$ be the adjustable parameters of the system. Let $\{ \mathbf{x}_t, \mathbf{s}_t, \dot{\mathbf{s}}_t \}_k$ denote the information about $(\mathbf{x}_t, \mathbf{s}_t, \dot{\mathbf{s}}_t)$ at time $t$ which is locally available to parameter $\theta_k$. We say that the Lagrangian $L$ is \textit{sum-separable} if it is of the form
\begin{equation}
    \label{eq:sum-separability-Lagrangian}
    L(\theta, \mathbf{x}_t, \mathbf{s}_t, \dot{\mathbf{s}}_t) = L_0(\mathbf{x}_t, \mathbf{s}_t, \dot{\mathbf{s}}_t) + \sum_{k=1}^N L_k( \theta_k, \{ \mathbf{x}_t, \mathbf{s}_t, \dot{\mathbf{s}}_t \}_k),
\end{equation}
where $L_0(\mathbf{x}_t, \mathbf{s}_t, \dot{\mathbf{s}}_t)$ is a term that is independent of the parameters to be adjusted, and $L_k$ is a scalar function of $\theta_k$ and $\{\mathbf{x}_t, \mathbf{s}_t, \dot{\mathbf{s}}_t\}_k$ for each $k \in \{ 1, \ldots, N \}$.

\subsection{Gradient Formula}

Similarly to the static case (Section \ref{sec:loss-gradients}), we introduce the \textit{total action functional}
\begin{equation}
    \mathcal{S}^\beta(\mathbf{s}) = \int_0^T \left( L(\theta, \mathbf{x}_t, \mathbf{s}_t, \dot{\mathbf{s}}_t) + \beta \; c_t(\mathbf{s}_t, \mathbf{y}_t) \right) dt,
\end{equation}
defined for any value of the nudging factor $\beta$ (for fixed $\theta$, $\mathbf{x}$ and $\mathbf{y}$). Intuitively, by varying $\beta$, the action functional $\mathcal{S}^\beta$ is modified, and so is the stationary solution of the action, i.e. the effective trajectory of the system. Specifically, let us denote $\mathbf{s}^\beta$ the trajectory characterized by the stationarity condition $\frac{\delta \mathcal{S}^\beta}{\delta \mathbf{s}}(\mathbf{s}^\beta) = 0$. Note in particular that for $\beta=0$ we have $\mathbf{s}^0 = \mathbf{s}(\theta, \mathbf{x})$.

\medskip

\begin{thm}[Gradient formula for Lagrangian-based models]
\label{thm:time-varying-eqprop}
The gradient of the loss can be computed using the following formula:
\begin{equation}
    \label{eq:time-varying-eqprop}
    \frac{\partial \mathcal{L}_0^T}{\partial \theta}(\theta, \mathbf{x}, \mathbf{y}) = \left. \frac{d}{d\beta} \right|_{\beta=0} \int_0^T  \frac{\partial L}{\partial \theta} \left( \theta, \mathbf{x}_t, \mathbf{s}^\beta_t, \dot{\mathbf{s}}^\beta_t \right) dt
\end{equation}
Furthermore, if the Lagrangian function $L$ is sum-separable, then the gradient for each parameter $\theta_k$ depends only on information that is locally available to $\theta_k$:
\begin{equation}
    \label{eq:time-varying-eqprop-local}
    \frac{\partial \mathcal{L}_0^T}{\partial \theta_k}(\theta, \mathbf{x}, \mathbf{y}) = \left. \frac{d}{d\beta} \right|_{\beta=0} \int_0^T \frac{\partial L_k}{\partial \theta_k} (\theta_k, \{ \mathbf{x_t}, \mathbf{s}^\beta_t, \dot{\mathbf{s}}^\beta_t \}_k) dt.
\end{equation}
\end{thm}

\begin{proof}[Proof of Theorem \ref{thm:time-varying-eqprop}]
We derive Theorem \ref{thm:time-varying-eqprop} as a corollary of Theorem \ref{thm:static-eqprop}. Recall that the action functional is by definition $\mathcal{S}(\theta, \mathbf{x}, \mathbf{s}) = \int_0^T L(\theta, \mathbf{x}_t, \mathbf{s}_t, \dot{\mathbf{s}}_t)$ and that the effective trajectory $\mathbf{s}(\theta, \mathbf{x})$ satisfies the stationary condition $\frac{\delta \mathcal{S}}{\delta \mathbf{s}}(\theta, \mathbf{x}, \mathbf{s}(\theta, \mathbf{x})) = 0$. We can define a cost functional $\mathcal{C}$ on any conceivable trajectory $\mathbf{s}$ by the formula $\mathcal{C}(\mathbf{s}, \mathbf{y}) = \int_0^T c_t(\mathbf{s}_t, \mathbf{y}_t) dt$. The loss $\mathcal{L}_0^T$ then rewrites $\mathcal{L}_0^T(\theta, \mathbf{x}, \mathbf{y}) = \mathcal{C}(\mathbf{s}(\theta, \mathbf{x}), \mathbf{y})$, the total action functional rewrites $\mathcal{S}^\beta(\mathbf{s}) = \mathcal{S}(\theta, \mathbf{x}, \mathbf{s}) + \beta \; \mathcal{C}(\mathbf{s}, \mathbf{y})$, and the nudged trajectory $\mathbf{s}^\beta$ satisfies the stationarity condition $\frac{\delta \mathcal{S}}{\delta \mathbf{s}}(\theta, \mathbf{x}, \mathbf{s}^\beta) + \beta \; \frac{\delta \mathcal{C}}{\delta \mathbf{s}}(\mathbf{s}^\beta, \mathbf{y}) = 0$.

Using these notations, the first formula to be proved (Eq.~\ref{eq:time-varying-eqprop}) rewrites
\begin{equation}
    \frac{\partial \mathcal{L}_0^T}{\partial \theta}(\theta, \mathbf{x}, \mathbf{y}) = \left. \frac{d}{d\beta} \right|_{\beta=0} \frac{\partial \mathcal{S}}{\partial \theta}\left(\theta, \mathbf{x}, \mathbf{s}^\beta \right),
\end{equation}
which is exactly the first formula of Theorem \ref{thm:static-eqprop}.
Finally, the second formula to be proved (Eq.~\ref{eq:time-varying-eqprop-local}) is a direct consequence of Eq.~\ref{eq:time-varying-eqprop} and the definition of sum-separability (Eq.~\ref{eq:sum-separability-Lagrangian}).

\end{proof}

\subsection{Training Sum-Separable Lagrangian-Based Models}

Theorem \ref{thm:time-varying-eqprop} suggests the following EqProp-like training procedure for Lagrangian-based models, to update the parameters in proportion to their loss gradients. Let us assume that the Lagrangian function has the sum-separability property. 

\paragraph{Free phase (inference).}
Set the system in some initial state $(\mathbf{s}_0,\dot{\mathbf{s}}_0)$ at time $t=0$, and set the nudging factor $\beta$ to zero. Play the input trajectory $\mathbf{x}$ over the time interval $[0, T]$, and let the system follow the trajectory $\mathbf{s}^0$ (i.e. the effective trajectory characterized by Eq.~\ref{eq:free-effective-trajectory}). We call $\mathbf{s}^0$ the \textit{free trajectory}. For each parameter $\theta_k$, the quantity $\frac{\partial L_k}{\partial \theta_k} (\theta_k, \{ \mathbf{x_t}, \mathbf{s}^0_t, \dot{\mathbf{s}}^0_t \}_k)$ is measured and integrated from $t=0$ to $t=T$, and the result is stored locally.

\paragraph{Nudged phase.}
Set the system in the same initial state $(\mathbf{s}_0,\dot{\mathbf{s}}_0)$ as in the free phase, and set now the nudging factor $\beta$ to some positive or negative (nonzero) value. Play again the input trajectory $\mathbf{x}$ over the time interval $[0, T]$, as well as the target trajectory $\mathbf{y}$, and let the system follow the trajectory $\mathbf{s}^\beta$ (i.e. the effective trajectory that is stationary with respect to $\mathcal{S}^\beta$). For each parameter $\theta_k$, the quantity $\frac{\partial L_k}{\partial \theta_k} (\theta_k, \{ \mathbf{x_t}, \mathbf{s}^\beta_t, \dot{\mathbf{s}}^\beta_t \}_k)$ is measured and integrated from $t=0$ to $t=T$.

\paragraph{Update rule.}
Finally, each parameter $\theta_k$ is updated locally in proportion to its gradient, i.e.
$\Delta \theta_k = - \eta \widehat{\nabla}_{\theta_k}(\beta)$, where $\eta$ is a learning rate, and
\begin{equation}
\widehat{\nabla}_{\theta_k}(\beta) = \frac{1}{\beta} \left( \int_0^T \frac{\partial L_k}{\partial \theta_k} (\theta_k, \{ \mathbf{x_t}, \mathbf{s}^\beta_t, \dot{\mathbf{s}}^\beta_t \}_k) dt - \int_0^T  \frac{\partial L_k}{\partial \theta_k} (\theta_k, \{ \mathbf{x_t}, \mathbf{s}^0_t, \dot{\mathbf{s}}^0_t \}_k) dt \right).
\end{equation}

As in the static setting (Chapter \ref{chapter:eqprop}), it is possible to reduce the bias and the variance of the gradient estimator by using the symmetrized version
\begin{equation}
\widehat{\nabla}_{\theta_k}^{\rm sym}(\beta) = \frac{1}{2 \beta} \left( \int_0^T \frac{\partial L_k}{\partial \theta_k} (\theta_k, \{ \mathbf{x_t}, \mathbf{s}^\beta_t, \dot{\mathbf{s}}^\beta_t \}_k) dt - \int_0^T  \frac{\partial L_k}{\partial \theta_k} (\theta_k, \{ \mathbf{x_t}, \mathbf{s}^{-\beta}_t, \dot{\mathbf{s}}^{-\beta}_t \}_k) dt \right).
\end{equation}
This requires two nudged phases: one with a positive nudging ($+\beta$) and one with a negative nudging ($-\beta$).

Although the EqProp training method for Lagrangian-based models requires running the input trajectory twice (in the free phase and in the nudged phase), we stress that we do not require to store the past states of the system, unlike the backpropagation through time (BPTT) algorithm used to train conventional recurrent neural networks.

\subsection{From Energy-Based to Lagrangian-Based Models}

Conceptually, we have the following correspondence between the static setting (energy-based models) and the time-varying setting (Lagrangian-based models).
\begin{itemize}
\item The concept of \textit{configuration} ($s$) is replaced by that of \textit{trajectory} ($\mathbf{s}$). A trajectory $\mathbf{s}$ is a function from the time interval $[0, T]$ to the space of configurations, which assigns to each time $t \in [0, T]$ a configuration $\mathbf{s}_t$.
\item The concept of \textit{energy function} ($E$) is replaced by that of \textit{action functional} ($\mathcal{S}$). Whereas an energy function $E$ assigns a real number $E(s)$ to each configuration $s$, an action functional $\mathcal{S}$ assigns a real number $\mathcal{S}(\mathbf{s})$ to each trajectory $\mathbf{s}$.
\item The concept of \textit{equilibrium state} (denoted $s(\theta, x)$ or $s_\star$) is replaced by that of \textit{effective trajectory} (denoted $\mathbf{s}(\theta, \mathbf{x})$). Whereas an equilibrium state is characterized by the stationarity of the energy ($\frac{\partial E}{\partial s} = 0$), an effective trajectory is characterized by the stationarity of the action ($\frac{\delta \mathcal{S}}{\delta \mathbf{s}} = 0$).
\end{itemize}

\subsection{Lagrangian-Based Models Include Energy-Based Models}

Consider a Lagrangian-based model whose Lagrangian function does not depend on $\dot{\mathbf{s}}_t$, i.e. $L$ is of the form
\begin{equation}
L(\theta,\mathbf{x}_t,\mathbf{s}_t,\dot{\mathbf{s}}_t) = E(\theta,\mathbf{x}_t,\mathbf{s}_t).
\end{equation}
Further suppose that the input signal $\mathbf{x}$ is static, i.e. $\mathbf{x}_t = x$ for any $t$. Denote $s_\star$ the equilibrium state characterized by $\frac{\partial E}{\partial s}(\theta,x,s_\star) = 0$. Then the trajectory $\mathbf{s}$ constantly equal to $s_\star$ (i.e. such that $\mathbf{s}_t = s_\star$ for all $t$) is a stationary solution of the action functional
\begin{equation}
\mathcal{S}(\theta,x,\mathbf{s}) = \int_0^T E(\theta,x,\mathbf{s}_t) dt.
\end{equation}
Indeed, for any variationa $\delta \mathbf{s}$ around $\mathbf{s}$, we have $\delta \mathcal{S} = \int_0^T \delta E(\theta,x,\mathbf{s}_t) dt = \int_0^T \frac{\partial E}{\partial s}(\theta,x,s_\star) \cdot \delta \mathbf{s_t} dt = 0$.
In this sense, energy-based models are special instances of Lagrangian-based models. Furthermore, assuming that the target signal $\mathbf{y}$ and the cost function $c_t$ are also static (i.e. $\mathbf{y}_t = y$ and $c_t = c$ at any time $t$), then the loss is equal to $\mathcal{L}_0^T = \int_0^T c(s_\star,y) dt$, which is the loss in the static setting (up to a constant $T$). In this case, the EqProp learning algorithm for Lagrangian-based models boils down to the EqProp learning algorithm for energy-based models (up to a constant $T$).

\section{Equilibrium Propagation in Stochastic Systems}
\label{sec:stochastic-setting}

Unlike neural networks trained on digital computers which can reliably process information in a deterministic way, physical systems (including analog circuits and biological networks) are subject to noise. In this section we present an extension of the equilibrium propagation framework to stochastic systems, which allows us to take such forms of noise into account, and may therefore be useful both from the neuromorphic and neuroscience points of view.

We note that the question whether the brain is stochastic or deterministic is controversial. However, even if the brain were deterministic, the precise trajectory of the neural activity is likely to be fundamentally unpredictable (i.e. chaotic) and thus easier to study statistically. In this case, the brain can still be usefully modelled with probability distributions (using probability theory or ergodic theory).


\subsection{From Deterministic to Stochastic Systems}

In the stochastic setting, when presented with an input $x$, instead of an equilibrium state $s_\star$, the model defines a probability distribution $p_\star(s)$ over the space of possible configurations $s$. Thus, rather than a stationary condition of the form $\frac{\partial E}{\partial s}(\theta, x, s_\star) = 0$, we now have an equilibrium distribution $p_\star(s)$ such that
\begin{equation}
	\label{eq:free-boltzmann-distribution}
	p_\star(s) = \frac{e^{-E(\theta, x, s)}}{Z_\star}, \qquad \text{with} \qquad Z_\star = \int e^{-E(\theta, x, s)}ds.
\end{equation}
The probability distribution defined by $p_\star(s)$ is called the Boltzmann distribution (or Gibbs distribution), and the normalizing constant $Z_\star$ is called the \textit{partition function}.
In this setting, the loss that we want to minimize is the expected cost over the equilibrium distribution
\begin{equation}
	\label{eq:loss-stochastic}
	\mathcal{L}_{\rm sto}(\theta, x, y) = \mathbb{E}_{s \sim p_\star(s)} \left[ C \left( s, y \right) \right].
\end{equation}
We note that $\mathcal{L}_{\rm sto}$ depends on $\theta$ and $x$ through the equilibrium distribution $p_\star(s)$.

\subsection{Gradient Formula}

As in the deterministic framework, the stochastic version of equilibrium propagation makes use of the total energy function $E(\theta, x, s) + \beta \; C(s, y)$. The notion of nudged equilibrium state ($s_\star^\beta$) is replaced accordingly by a nudged equilibrium distribution $p_\star^\beta(s)$, which is the Boltzmann distribution associated to the total energy function, i.e.
\begin{equation}
    \label{eq:nudged-boltzmann-distribution}
    p_\star^\beta(s) = \frac{e^{-E(\theta, x, s) - \beta \; C(s, y)}}{Z_\star^\beta}, \qquad \text{with} \qquad Z_\star^\beta = \int e^{-E(\theta, x, s) - \beta \; C(s, y)} ds.
\end{equation}

The following theorem extends Theorem \ref{thm:static-eqprop} to stochastic systems.

\medskip

\begin{thm}[\citet{scellier2017equilibrium}]
    \label{thm:equi-prop-langevin}
    The gradient of the objective function with respect to $\theta$ is equal to
    \begin{equation}
        \label{eq:thm-stochastic}
        \frac{\partial \mathcal{L}_{\rm sto}}{\partial \theta}(\theta, x, y) =
        \left. \frac{d}{d \beta} \right|_{\beta=0} \mathbb{E}_{s \sim p_\star^\beta(s)} \left[
        \frac{\partial E}{\partial \theta} \left( \theta, x, s \right) \right].
    \end{equation}
    Furthermore, if the energy function is sum-separable (in the sense of Eq.~\ref{eq:sum-separability}), then
    \begin{equation}
        \frac{\partial \mathcal{L}_{\rm sto}}{\partial \theta_k}(\theta, x, y) =
        \left. \frac{d}{d \beta} \right|_{\beta=0} \mathbb{E}_{s \sim p_\star^\beta(s)} \left[
        \frac{\partial E_k}{\partial \theta_k} \left( \theta_k, \{ x, s \}_k \right) \right].
    \end{equation}
\end{thm}

\begin{proof}[Proof of Theorem \ref{thm:equi-prop-langevin}]
    Recall that the total energy function is by definition $F(\theta, \beta, s) = E(\theta, x, s) + \beta \; C(s, y)$, where the notations $x$ and $y$ are dropped for simplicity (since they do not play any role in the proof). We also (re)define $Z_\theta^\beta = \int e^{-F(\theta, \beta, s)}ds$, the partition function, as well as $p_\theta^\beta(s) = \frac{e^{-F(\theta, \beta, s)}}{Z_\theta^\beta}$, the corresponding Boltzmann distribution. Recalling the definition of the loss $\mathcal{L}_{\rm sto}$ (Eq.~\ref{eq:loss-stochastic}), and using the fact that $\frac{\partial F}{\partial \beta} = C$ and that $\frac{\partial F}{\partial \theta} = \frac{\partial E}{\partial \theta}$, the formula to show (Eq.~\ref{eq:thm-stochastic}) is a particular case of the following formula, evaluated at the point $\beta=0$:
    \begin{equation}
        \label{eq:lemma-eqprop-stochastic}
        \frac{d}{d\theta} \mathbb{E}_{s \sim p_\theta^\beta(s)} \left[ \frac{\partial F}{\partial \beta} \left( \theta, \beta, s \right) \right]
        = \frac{d}{d\beta} \mathbb{E}_{s \sim p_\theta^\beta(s)} \left[ \frac{\partial F}{\partial \theta} \left( \theta, \beta, s \right) \right].
    \end{equation}
    Therefore, in order to prove Theorem \ref{thm:equi-prop-langevin}, it is sufficient to prove  Eq.~\ref{eq:lemma-eqprop-stochastic}. We do this in two steps. First, the cross-derivatives of the log-partition function $\ln \left( Z_\theta^\beta \right)$ are equal:
    \begin{equation}
        \label{eq:cross-derivatives}
        \frac{d}{d \theta} \frac{d}{d \beta} \ln \left( Z_\theta^\beta \right)
        = \frac{d}{d \beta} \frac{d}{d \theta} \ln \left( Z_\theta^\beta \right).
    \end{equation}
    Second, we have
    \begin{equation}
      \label{eq:derivative-beta}
      \frac{d}{d\beta} \ln \left( Z_\theta^\beta \right)
      = \mathbb{E}_{s \sim p_\theta^\beta(s)} \left[ \frac{\partial F}{\partial \beta}(\theta, \beta, s) \right],
    \end{equation}
    and similarly
    \begin{equation}
      \label{eq:derivative-theta}
      \frac{d}{d\theta} \ln \left( Z_\theta^\beta \right)
      = \mathbb{E}_{s \sim p_\theta^\beta(s)} \left[ \frac{\partial F}{\partial \theta}(\theta, \beta, s) \right].
    \end{equation}
    Plugging Eq.~\ref{eq:derivative-beta} and Eq.~\ref{eq:derivative-theta} in Eq.~\ref{eq:cross-derivatives}, we get Eq.~\ref{eq:lemma-eqprop-stochastic}.
    Hence the result.
\end{proof}

We note that Eq.~\ref{eq:lemma-eqprop-stochastic} is a stochastic variant of the fundamental lemma of EqProp (Lemma \ref{lma:main}).
The quantity $-\ln \left( Z_\theta^\beta \right)$ is called the \textit{free energy} of the system.

\subsection{Langevin Dynamics}

The prototypical dynamical system to sample from the equilibrium distribution $p_\star(s)$ (Eq.~\ref{eq:free-boltzmann-distribution}) is the \textit{Langevin dynamics}, which we describe here. Recall from Chapter \ref{chapter:hopfield} the gradient dynamics $\frac{d s_t}{dt} = -\frac{\partial E}{\partial s}(\theta, x, s_t)$. To go from this (deterministic) gradient dynamics to the (stochastic) Langevin dynamics, we add a new term (a Brownian term) which models a form of noise:
\begin{equation}
    \label{eq:free-langevin-dynamics}
    d S_t = - \frac{\partial E}{\partial s} \left( \theta, x, S_t \right) dt + \sqrt{2} \; dB_t.
\end{equation}
In this expression, $B_t$ is a mathematical object called a \textit{Brownian motion}. Instead of defining $B_t$ formally, we give here an intuitive definition. Intuitively, each increment $dB_t$ (between time $t$ and time $t+dt$) can be thought of as a normal random variable with mean $0$ and variance $dt$, which is "independent of past increments". By following this noisy form of gradient descent with respect to the energy function $E$, the state of the system ($S_t$) settles to the Boltzmann distribution. This can be proved using the Kolmogorov forward equation (a.k.a. Fokker-Planck equation) for diffusion processes.

Here we have chosen the constant $\sqrt{2}$ in the Langevin dynamics, so that the `temperature' of the system is $1$. More generally, if the Brownian motion is scaled by a factor $\sigma(\theta, x)$, i.e. if the dynamics is of the form $d S_t = - \frac{\partial E_\theta}{\partial s} \left( \theta, x, S_t \right) dt + \sigma(\theta, x) \cdot dB_t$, then the exponent in the Boltzmann distribution needs to be rescaled by a factor $\frac{1}{2}\sigma^2(\theta, x)$. We call this modified equilibrium distribution the Boltzmann distribution with temperature $T=\frac{1}{2}\sigma^2(\theta, x)$. We note that if $\sigma(\theta, x)=\sqrt{2}$ then $T=1$.

\subsection{Equilibrium Propagation in Langevin Dynamics}

In the setting of Langevin dynamics, EqProp takes the following form.

\paragraph{Free phase (inference).}
In the free phase, the network is shown an input $x$ and the state of the system follows the Langevin dynamics of Eq.~\ref{eq:free-langevin-dynamics}. `Free samples' are drawn
from the equilibrium distribution $p_\star(s) \propto e^{-E(\theta, x, s)}$.

\paragraph{Nudged phase.}
In the nudged phase, a term $-\beta \; \frac{\partial C}{\partial s}$ is added to the dynamics of Eq.~\ref{eq:free-langevin-dynamics}, where $\beta$ is a scalar hyperparameter (the nudging factor). Denoting $S_t^\beta$ the state of the network at time $t$ in the nudged phase, the dynamics reads:
\begin{equation}
	d S_t^\beta = \left[ - \frac{\partial E}{\partial s} \left( \theta, x, S_t^\beta \right) - \beta \; \frac{\partial C}{\partial s} \left( S_t^\beta, y \right) \right] dt + \sqrt{2} \; dB_t.
\end{equation}
Here for readability we use the same notation $B_t$ for the Brownian motion of the nudged phase, but it should be understood that this is a new Brownian motion, independent of the one used in the free phase.
`Nudged samples' are drawn from the nudged distribution $p_\star^\beta(s) \propto e^{-E(\theta, x, s)- \beta \; C(s, y)}$.

\paragraph{Gradient estimate.}
Finally, the gradient of the loss $\mathcal{L}_{\rm sto}$ of Eq.~\ref{eq:loss-stochastic} can be approximated using the samples from the free and nudged distributions to estimate:
\begin{equation}
    \widehat{\nabla}_\theta(\beta) = \frac{1}{\beta} \left( \mathbb{E}_{s \sim p_\star^\beta(s)} \left[
    \frac{\partial E}{\partial \theta} \left( \theta, x, s \right) \right]
    - \mathbb{E}_{s \sim p_\star(s)} \left[
    \frac{\partial E}{\partial \theta} \left( \theta, x, s \right) \right] \right).
\end{equation}

\section{Contrastive Meta-Learning}
\label{sec:contrastive-meta-learning}

Recently, \citet{zucchet2021contrastive} introduced the \textit{contrastive meta-learning} framework, where they propose to train the meta-parameters of a meta-learning model using the EqProp method. In this section, we briefly present the setting of meta-learning and show how \citet{zucchet2021contrastive} derive the contrastive meta-learning rule.

\subsection{Meta-Learning and Few-Shot Learning}

Meta learning, or \textit{learning to learn}, is a broad field that encompasses \textit{hyperparameter optimization}, \textit{few-shot learning}, and many other use cases. Here, for concreteness, we present the setting of few-shot learning.

In the setting of few-shot learning, the aim is to build a system that is able to learn (or `adapt' to) a given task $\mathcal{T}$ when only very limited data is available for that task.
The system should be able to do so for a variety of tasks coming from a distribution of tasks $p(\mathcal{T})$.
In this setting, the system has two types of parameters: a \textit{meta-parameter} $\theta$ which is shared across all tasks, and a \textit{task-specific parameter} $\phi$ which can be adapted to a given task. In the \textit{adaptation phase}, the task-specific parameter $\phi$ adapts to some task $\mathcal{T}$ using a training set $\mathcal{D}_{\rm train}$ corresponding to that task. The resulting value of the task-specific parameter after this adaptation phase is denoted $\phi(\theta,\mathcal{D}_{\rm train})$, which depends on both $\theta$ and $\mathcal{D}_{\rm train}$. The performance of the resulting $\phi(\theta,\mathcal{D}_{\rm train})$ is then evaluated on a test set $\mathcal{D}_{\rm test}$ from the same task $\mathcal{T}$. This performance is denoted $L(\phi(\theta,\mathcal{D}_{\rm train}),\mathcal{D}_{\rm test})$, where $L$ is a loss function. The goal of meta-learning is then to find the value of the meta-parameter $\theta$ that minimizes the expected loss $\mathcal{R}(\theta) = \mathbb{E}_{(\mathcal{D}_{\rm train},\mathcal{D}_{\rm test})} \left[ L(\phi(\theta,\mathcal{D}_{\rm train}),\mathcal{D}_{\rm test}) \right]$ over pairs of training/test sets $(\mathcal{D}_{\rm train},\mathcal{D}_{\rm test})$ coming from the distribution of tasks $p(\mathcal{T})$. In other words, the goal is to find $\theta$ that generalizes well across tasks from the distribution $p(\mathcal{T})$.

\subsection{Contrastive Meta-Learning}

The idea of the \textit{contrastive meta-learning} framework of \citet{zucchet2021contrastive} is the following. In the adaptation phase, the task-specific parameter $\phi$ minimizes an \textit{inner loss} $L^{\rm in}$, so that
\begin{equation}
    \phi(\theta,\mathcal{D}_{\rm train}) = \underset{\phi}{\arg\min} \; L^{\rm in}(\theta,\mathcal{D}_{\rm train},\phi).
\end{equation}
In the setting of \textit{regularization learning} for example, the inner loss is of the form $L^{\rm in}(\theta,\mathcal{D}_{\rm train},\phi) = L(\phi,\mathcal{D}_{\rm train}) + R(\theta,\phi)$, where $L$ is the same loss as the one used on the test set, and $R(\theta,\phi)$ is a regularization term. Exploiting the fact that, at the end of the adaptation phase, the task-specific parameter $\phi(\theta,\mathcal{D}_{\rm train})$ satisfies the `equilibrium condition' $\frac{\partial L^{\rm in}}{\partial \phi}(\theta,\mathcal{D}_{\rm train},\phi(\theta,\mathcal{D}_{\rm train})) = 0$, \citet{zucchet2021contrastive} then propose to use the EqProp method to compute the gradients (with respect to $\theta$) of the meta loss
\begin{equation}
    \mathcal{L}_{\rm meta} = L^{\rm out}(\phi(\theta,\mathcal{D}_{\rm train}),\mathcal{D}_{\rm test}),
\end{equation}
where $L^{\rm out}$ is a so-called \textit{outer loss}, e.g. $L^{\rm out} = L$ in regularization learning. To this end, they introduce the `total loss'
\begin{equation}
    L^{\rm total}(\theta,\beta,\phi) = L^{\rm in}(\theta,\mathcal{D}_{\rm train},\phi) + \beta \; L^{\rm out}(\phi,\mathcal{D}_{\rm test}),
\end{equation}
where $\beta$ is a scalar parameter (the nudging factor). Then they consider
\begin{equation}
    \phi_\star^\beta = \underset{\phi}{\arg\min} \; L^{\rm total}(\theta,\beta,\phi),
\end{equation}
defined for any $\beta$, and they exploit the formula
\begin{equation}
    \frac{\partial \mathcal{L}_{\rm meta}}{\partial \theta} = \left. \frac{d}{d\beta} \right|_{\beta=0} \frac{\partial L^{\rm in}}{\partial \theta} \left( \theta, \mathcal{D}_{\rm train}, \phi_\star^\beta \right),
\end{equation}
which is a reformulation of Theorem \ref{thm:static-eqprop}.

More generally, the contrastive meta-learning method applies to any functions $L^{\rm in}$ and $L^{\rm out}$, and any bilevel optimization problem where the aim is to optimize $\mathcal{L}_{\rm meta}(\theta) = L^{\rm out}(\theta,\phi(\theta))$ with respect to $\theta$, under the constraint that $\frac{\partial L^{\rm in}}{\partial \phi}(\theta,\phi(\theta)) = 0$.

%% file: conclusion.tex
\chapter{Conclusion}

In this thesis, we have presented a mathematical framework that applies to systems that are described by variational equations, while maintaining the benefits of backpropagation. This framework may have implications both for neuromorphic computing and for neuroscience.

\section{Implications for Neuromorphic Computing}

Current deep learning research is grounded on a very general and powerful mathematical principle: automatic differentiation for backpropagating error gradients in differentiable neural networks. This mathematical principle is at the heart of all deep learning libraries (TensorFlow, PyTorch, Theano, etc.). The emergence of such software libraries has greatly eased deep learning research and fostered the large scale development of parallel processors for deep learning (e.g. GPUs and TPUs). However, these processors are power inefficient by orders of magnitude, if we take the brain as a benchmark. The rapid increase in energy consumption raises concerns as the use of deep learning systems in society keeps growing \citep{strubell2019energy}.

At a more abstract level of description, the backpropagation algorithm of conventional deep learning allows to train neural networks by stochastic gradient descent (SGD). In this thesis, we have presented a mathematical framework which allows to preserve the key benefits of SGD, but opens a path for implementation on neuromorphic processors which directly exploit physics and the in-memory computing concept to perform the desired computations. Building neuromorphic systems that can match the performance of current deep learning systems is still in the future, but the potential speedup and power reduction is extremely appealing. This would also allow us to scale neural networks to sizes beyond the reach of current GPU-based deep learning models. Besides, by mimicking the working mechanisms of the brain more closely, such neuromorphic systems could also inform neuroscience.

\section{Implications for Neuroscience}

How do the biophysical mechanisms of neural computation give rise to intelligence? Ultimately, if we want to explain how our thoughts, memories and behaviours emerge from neural activities, we need a mathematical theory. Here, we explain how the mathematical framework presented in this thesis may help for this purpose.

\subsection{Variational Formulations of Neural Computation ?}

A number of ideas at the core of today's deep learning systems draw inspiration from the brain. However, these deep neural networks are not biologically realistic in details. In particular, the neuron models may look overly simplistic from a neurophysiological point of view. In these models, the state of a neuron is described by a single number, which can be thought of as its firing rate. A real neuron on the other hand, like any other biological cell, is an extraordinarily complex machinery, composed of a very large quantity of proteins interacting in complex ways. Because of this complexity, the hope to ever come up with a mathematical theory of the brain may seem vain.

This complexity should not discourage us, however.
One key point is that not all details of neurobiology may be relevant to explain the fundamental working mechanisms of the brain that give rise to emerging properties such as memory and learning. \citet{hertz1991introduction} puts it in these words: "Just as most of the details of the separate parts of a large ship are unimportant in understanding the behaviour of the ship (e.g. that it floats or transport cargo), so many details of single nerve cells may be unimportant in understanding the collective behaviour of a network of cells". Which biophysical characteristics of neural computation are essential to explain how information is processed in brains, and which can be abstracted away? While current deep learning systems use \textit{rate models} (i.e. neuron models relying on the neuron's firing rate), a simple but more realistic neuron model is the leaky-integrate and fire (LIF) model, which accounts for the \textit{spikes} (a.k.a. \textit{action potentials}) and the electrical activity of neurons at each point in time. A more elaborated model is the Hodgkin-Huxley model of action potentials, which takes into account ion channels to describe how spikes are initiated. At a more detailed level, real neurons have a spatial layout, and each part of the neuron has its own voltage value and ion concentration values. In recent years, more realistic neuron models that include spikes \citep{zenke2017superspike,payeur2020burst} and multiple compartments \citep{bengio2016feedforward,guerguiev2017towards,sacramento2018dendritic,richards2019dendritic,payeur2020burst} have been proposed for deep learning. Can we figure out which elements of neurobiology are essential to explain the mechanisms underlying intelligence, abstracting out those that are not necessary to understand these mechanisms?

In this thesis, we have presented a mathematical theory which applies to a broad class of systems whose state or dynamics is the solution of a variational equation. Given the predominance of variational principles in physics, a question arises: can neural dynamics in the brain be derived from variational principles too? We note that various variational principles for neuroscience modelling have been proposed \citep{friston2010free,betti2019cognitive,dold2019lagrangian,kendall2021gradient}.

\subsection{SGD Hypothesis of Learning}

Today, the neural networks of conventional deep learning are trained by stochastic gradients descent (SGD), using the backpropagation algorithm to compute the loss gradients. The backpropagation algorithm is not biologically realistic as it requires that neurons emit two quite different types of signals: an activation signal in the forward pass, and a signed gradient signal in the backward pass. Real neurons on the other hand communicate with only one sort of signals -- the \textit{spikes}. Worse, the backpropagation through time (BPTT) algorithm used in recurrent networks requires storing past hidden states of the neurons.

Although these deep neural networks are not biologically realistic in details, they have proved to be valuable not just for AI applications, but also as models for neuroscience. In recent years, deep learning models have been used for neuroscience modelling of the visual and auditory cortex. Deep neural networks have been found to outperform other biologically plausible models at matching neural representations in the visual cortex \citep{mante2013context,cadieu2014deep,kriegeskorte2015deep,sussillo2015neural,yamins2016using,pandarinath2018inferring} and at predicting auditory neuron responses \citep{kell2018task}.
Because SGD-optimized neural networks are state-of-the-art at solving a variety of tasks in AI, and also state-of-the-art models at predicting neocortical representations, a hypothesis emerges which is that the cortex may possess general purpose learning algorithms that implement SGD. More generally, a view emerges, which is that the fundamental principles of current deep learning systems may provide a useful theoretical framework for gaining insight into the principles of neural computation \citep{richards2019deep}.

While the backpropagation algorithm is not biologically realistic, a more reasonable hypothesis is that the brain uses a different mechanism to compute the loss gradients required to perform SGD. A long standing idea is that the loss gradients may be encoded in the difference of neural activities to drive synaptic changes \citep{hinton1988learning,lillicrap2020backpropagation}. If variational principles for neural dynamics exist, and if their corresponding energy function or Lagrangian function have the sum-separability property, then EqProp would suggest a learning mechanism involving local learning rules and suitable with optimization by SGD. Whereas in the setting of energy-based models, EqProp suggests that gradients are encoded in the difference of \textit{neural activities} (as hypothesized by \citet{lillicrap2020backpropagation}), in the Lagrangian-based setting, EqProp suggests that gradients are encoded in the difference of \textit{neural trajectories}.

We note that the SGD hypothesis of learning also raises several questions. First, what is the loss function that is optimized? Unlike in conventional machine learning, there are likely a variety of such loss functions, which may vary across brain areas and time \citep{marblestone2016toward}. Second, SGD dynamics depend on the metric that we choose for the space of synaptic weights \citep{surace2020choice}. Also, while the SGD hypothesis is reasonable for the function of the cortex, other components of the brain such as the hippocampus may use different learning algorithms.

\subsection{The Role of Evolution}

In this manuscript we have emphasized the importance of learning. The ability for individuals to learn within their lifetime is indeed an essential component of intelligence. But learning alone is not the only key to human and animal intelligence. Far from being a blank slate, at birth, the brain is pre-wired and structured. This structure provides us straight from birth with innate intuitions, abilities, and mechanisms which make us predisposed to learn much more quickly \citep[Chapters 3 and 4]{dehaene2020we}. These innate structures and mechanisms have arisen through evolution. Machine learning models account for these innate aspects of intelligence using \textit{inductive biases} (or \textit{priors}). Traditionally, these inductive biases are manually crafted. However, given the complexity of the brain, one may wonder whether one will ever manage to reverse-engineer the inductive biases of the brain `by hand'.

Evolution by natural selection can be regarded as another optimization process where, loosely speaking, the `adjustable parameters' are the \textit{genes}, and the `objective' that is maximized is the \textit{fitness} of the individual. The human genome has around $3 \times 10^9$ `parameters' (base pairs). Just like moving from manually crafted computations (in classical AI) to learned computations (in machine learning) proved extremely fruitful both for AI and neuroscience modelling, one may benefit from `evolving' inductive biases, by mimicking the process of evolution in  some way. One branch of machine learning which is relevant to address questions related to the optimization process carried out by evolution is \textit{meta-learning} (Section \ref{sec:contrastive-meta-learning}).
A related path, proposed by \citet{zador2019critique}, is to reverse-engineer the program encoded in the genome which wires up the brain during embryonic development.

We note that the learning rules and loss functions of the brain have also arisen through evolution and are possibly much more complex than in the traditional view of machine learning (as we have formulated it in this manuscript).

\section{Synergy Between Neuroscience and AI}

Is a mathematical theory of the brain all we need to understand the brain? Or do we need to build brain-like machines to claim that we understand it? This question depends of course on what we mean by `understanding' ; it is one of the fundamental questions of philosophy of science. In many fields of science, we have mathematical models of objects that we cannot build (for example, we have physics models of the Sun, but we cannot build one). Although a theory is all we need in principle to explain the measurements of experimentally accessible variables, it seems also clear that, if we can build a brain, or simulate one, our `understanding' of the brain will further improve, and the underlying theory will become more plausible.

Can we simulate a brain in software? In the introductory chapter, we have argued that with current digital hardware this strategy would at best be extremely slow and power hungry, and more likely just unfeasible. Just like it is impossible for statistical physicists to simulate in software the internal dynamics of a fluid composed of $10^{23}$ particles, simulating a brain composed of $10^{11}$ neurons and $10^{15}$ synapses (and many many more proteins) seems unfeasible. In these respects, the development of appropriate neuromorphic systems will eventually be necessary to emulate a brain.

More likely, by making it possible to run and train neural networks with more elaborated neural dynamics that more closely mimic those of real neurons, the development of neuromorphic hardware will help us come up with new hypotheses about the working mechanisms of the brain. As we build more brain-like AI systems, and as the performance of these AI systems improves, we can formulate new mathematical theories of the brain.
Just like the rise of deep learning as a leading approach to AI has eased the flow of information between different fields of AI (computer vision, speech recognition, natural language processing, etc.), we can expect that the development of neuromorphic systems together with mathematical frameworks to train them will ease the flow of information between AI and neuroscience too.

The problem of intelligence is thus both a problem for natural sciences and engineering. It counts to the greatest scientific problems, together with the problem of the origin of life, the problem of the origin of the universe, and many others. One specificity of the problem of intelligence is that, as we make progress towards solving this problem, we can use the knowledge that we acquire to build machines that can help us solve other scientific problems more easily. For example, most recently, a program called AlphaFold 2 promises to help us discover the 3D structure of proteins much more rapidly than prior methods, which is key to understanding most biological mechanisms in living organisms.

%% file: appendix.tex
\chapter{Gradient Estimators}
\label{chapter:appendix}

Recall from Theorem \ref{thm:static-eqprop} that the loss gradient is equal to
\begin{equation}
    \frac{\partial \mathcal{L}}{\partial \theta}(\theta,x,y) = \left. \frac{d}{d\beta} \right|_{\beta=0} \frac{\partial E}{\partial \theta} \left( \theta,x,s_\star^\beta \right).
\end{equation}
The \textit{one-sided gradient estimator} is defined as
\begin{equation}
\widehat{\nabla}_\theta(\beta) = \frac{1}{\beta} \left( \frac{\partial E}{\partial \theta}(\theta, x, s_\star^\beta) - \frac{\partial E}{\partial \theta}(\theta, x, s_\star^0) \right),
\end{equation}
and the \textit{symmetric gradient estimator} is defined as
\begin{equation}
\widehat{\nabla}_\theta^{\rm sym}(\beta) = \frac{1}{2\beta} \left( \frac{\partial E}{\partial \theta}(\theta, x, s_\star^\beta) - \frac{\partial E}{\partial \theta}(\theta, x, s_\star^{-\beta}) \right).
\end{equation}

\begin{lma}
\label{lma:gradient-estimators}
Let $\theta$, $x$ and $y$ be fixed.
Assuming that the function $\beta \mapsto \frac{\partial E}{\partial \theta}(\theta,x,s_\star^\beta)$ is three times differentiable, we have, as $\beta \to 0$:
\begin{align}
\widehat{\nabla}_\theta(\beta) &= \frac{\partial \mathcal{L}}{\partial \theta}(\theta,x,y) + \frac{A}{2} \beta + O(\beta^2), \label{eq:estimate-EP}\\
\widehat{\nabla}_\theta^{\rm sym}(\beta) &= \frac{\partial \mathcal{L}}{\partial \theta}(\theta,x,y) + O(\beta^2),\label{eq:estimate-EP-sym}
\end{align}
where $A = \left. \frac{d^2}{d\beta^2} \right|_{\beta=0} \frac{\partial E}{\partial \theta}(\theta,x,s_\star^\beta)$ is a constant (independent of $\beta$, but dependent on $\theta$, $x$ and $y$).
\end{lma}

Lemma \ref{lma:gradient-estimators} shows that the one-sided estimator $\widehat{\nabla}_\theta(\beta)$ possesses a first-order error term in $\beta$, which the symmetric estimator $\widehat{\nabla}_\theta^{\rm sym}(\beta)$ eliminates.

\begin{proof}
Define $f(\beta) = \frac{\partial E}{\partial \theta}(\theta,x,s_\star^\beta)$ and note that $f'(0) = \frac{\partial \mathcal{L}}{\partial \theta}(\theta,x,y)$ by Theorem \ref{thm:static-eqprop}, and that $f''(0) = \left. \frac{d^2}{d\beta^2} \right|_{\beta=0} \frac{\partial E}{\partial \theta}(\theta, x, s_\star^\beta)$.
As $\beta \to 0$, we have the Taylor expansion $f(\beta) = f(0) + \beta \; f'(0) + \frac{\beta^2}{2} f''(0) +  O(\beta^3)$. With these notations, the one-sided estimator reads $\widehat{\nabla}_\theta(\beta) = \frac{1}{\beta} \left( f(\beta) - f(0) \right) = f'(0) + \frac{\beta}{2} f''(0) + O(\beta^2)$, and the symmetric estimator, which is the mean of $\widehat{\nabla}_\theta(\beta)$ and $\widehat{\nabla}_\theta(-\beta)$, reads $\widehat{\nabla}_\theta^{\rm sym}(\beta) = f'(0) + O(\beta^2)$.
\end{proof}